\documentclass[12pt]{article}
\usepackage{threeparttable}
\usepackage{rotating}
\usepackage{amsfonts}
\usepackage{amsmath}
\usepackage{amsmath,amssymb}
\usepackage{mathtools}
\usepackage{amssymb}
\usepackage{amsmath}
\usepackage{amsthm}
\usepackage{thmtools}
\usepackage{graphicx}
\usepackage{array, color}
\usepackage{bm}
\usepackage{threeparttable}
\usepackage{algorithm}
\usepackage{algorithmic}
\usepackage{float}
\usepackage{mathtools}
\usepackage{multirow}
\usepackage[title]{appendix}
\usepackage[authoryear, round]{natbib}
\usepackage{geometry}
\usepackage[colorlinks= false, urlcolor=blue]{hyperref}
\def\*#1{\bm{#1}}

\newtheorem{theorem}{Theorem}
\newtheorem{lemma}{Lemma}

\newtheorem{model}{Model}

\newtheorem{ND}{ND}
\newtheorem{NGS}{NG}
\newtheorem{condition}{Condition}

\def\Dbb{\mathbb{D}}

\def\Ebb{\mathbb{E}}

\def\cF{\mathcal{F}}

\newcommand{\bI}{\mbox{\bf I}}

\def\cX{\mathcal{X}}

\def\cY{\mathcal{Y}}

\title{Wasserstein Generative Regression}

\author{Shanshan Song\thanks{Shangshan Song and Tong Wang contributed equally to this work.}
\thanks{Department of Statistics, The Chinese University of Hong Kong,  Hong Kong SAR, China}
\ Tong Wang$^*$\thanks{Department of Statistics, The Chinese University of Hong Kong,  Hong Kong SAR, China}
\ Guohao Shen\thanks{Department of Applied Mathematics, The Hong Kong Polytechnic University,  Hong Kong SAR, China}
\ Yuanyuan Lin\thanks{Department of Statistics, The Chinese University of Hong Kong,  Hong Kong SAR, China} $^{**}$
Jian Huang\thanks{Department of Applied Mathematics, The Hong Kong Polytechnic University,  Hong Kong SAR, China} \thanks{Corresponding authors, Yuanyuan Lin, email: ylin@sta.cuhk.edu.hk;  Jian Huang, email: j.huang@polyu.edu.hk}}

\begin{document}
\maketitle

\begin{abstract}
In this paper, we propose a new and unified approach for nonparametric regression and conditional distribution learning. Our approach simultaneously estimates a regression function and a conditional generator using a generative learning framework, where a conditional generator is a function that can generate samples from a conditional distribution. The main idea is to estimate a conditional generator that satisfies the constraint that it produces a good regression function estimator. We use deep neural networks to model the conditional generator. Our approach can handle problems with multivariate outcomes and covariates, and can be used to construct prediction intervals. We provide theoretical guarantees by deriving non-asymptotic error bounds and the distributional consistency of our approach under suitable assumptions. We also perform numerical experiments with simulated and real  data to demonstrate the effectiveness and superiority of our approach over some existing approaches in various scenarios.
\end{abstract}

\noindent
\textbf{keywords}: Conditional distribution, deep neural networks,  generative learning, nonparametric regression, non-asymptotic error bounds.

\section{Introduction}
Regression models and conditional distributions play a key role in a variety of prediction and inference problems in statistics.
There is a vast literature on nonparametric methods for regression analysis and conditional density estimation. Most existing methods use smoothing and basis expansion techniques, including kernel smoothing, local polynomials, and splines
 \citep{silverman86, scott1992, fan1996local, gyorfi2006distribution, wasserman2006, npe2008}.
However, the existing nonparametric regression and conditional
density estimation methods suffer from the  ``curse of dimensionality'', that is, their performance deteriorates dramatically as the dimensionality of  data increases. Indeed, most existing methods can only effectively handle up to a few predictors. Moreover, most existing methods only consider the case when the response is a scalar, but are not applicable to the settings with a high-dimensional response vector.

To circumvent the curse of dimensionality, many researchers have proposed and studied non- and semi-parametric models that impose certain structural constraints that reduce the model dimensionality.
Some notable examples include the single index model
 \citep{ichimura1993,hardle1993optimal},
 the generalized additive model
 \citep{hastie1986, stone1986dimensionality},
  and the projection pursuit model
  \citep{friedman1981projection},
 among others. However, these methods make strong assumptions
about the model structure, which may not hold in reality.
Moreover, these methods aim at estimating the regression function, but do not learn the conditional distribution. Therefore, they can only provide point prediction, not interval prediction with a measure of uncertainty.

In recent years, there have been many important developments in \textit{deep generative learning} \citep{sala2015}, in which deep neural networks are used  to approximate high-dimensional functions, such as generator and discriminator functions. In particular, for learning distributions of high-dimensional data arising in image analysis and natural language processing, the {generative adversarial networks} (GANs) \citep{goodfellow2014generative,arjovsky17}
have proven to be effective and achieved impressive success \citep{reed16,zhu17}.
Instead of estimating the functional form of a density function, GANs start from a known reference distribution and learn a map that pushes the reference distribution to the data distribution. GANs have also been extended to learn conditional distributions \citep{mirza2014cgan,kovachki2021conditional,zjlh2022,liu2021wasserstein}.

One of the main challenges in nonparametric regression is to estimate a function that can accurately capture the relationship between covariate and response variables. GANs can learn complex distributions. However, to the best of our knowledge,
there have not been systematic studies on how GANs can be used for nonparametric regression, despite their successes in distribution learning. Furthermore, conditional GANs, which are a natural extension of GANs for learning conditional distributions,
 do not automatically guarantee a good estimation of a regression function.

We propose a new and unified approach for nonparametric regression and conditional distribution estimation. Our approach estimates the regression function and a conditional generator at the same time using a generative learning method. A conditional generator is a function that transforms a random vector from a known reference distribution to the response variable space, which can be used to sample from a conditional distribution. Thus, when a conditional generator is estimated, it can be used to explore the target conditional distribution. Theoretically, the regression function is the expectation of the conditional generator with respect to the reference distribution. However, empirically such an expectation may not produce a good estimator
of the regression function.

Our main idea is to constrain the conditional generator to produce samples that minimize the quadratic loss of the regression function, which is computed as the expectation of the conditional generator.
Specifically, in the objective function for estimating the conditional generator based on distribution matching using the Wasserstein distance, we incorporate a quadratic loss term to control the error of the estimated regression function. We use deep neural networks to approximate the conditional generator, which can capture the complex structure of the data distribution.
In principle, other approximation methods such as splines can also be used. However, deep neural network approximation has the important advantage of being able to adapting to the latent structure of the data distribution. For simplicity, we call our method Wasserstein generative regression (WGR).


The proposed method has several attractive properties. First,  it is applicable to problems with a high-dimensional response variable, while the existing methods typically only consider the case of a scalar response. Second, the proposed method allows continuous, discrete and mixed types of predictors and responses, while the smoothing and basis expansion methods are mainly applicable to  continuous-type variables. Third, since the proposed method learns a conditional distribution generator, it can be used for constructing prediction intervals. In comparison, the existing nonparametric regression can only give point prediction. Finally, the proposed method is able to adapt to the latent data structure in a data-driven manner and thus can mitigate the curse of dimensionality, under the assumption the data distribution is supported on an approximate low-dimensional set.

The rest of the paper is organized as follows. In Section \ref{Method} we describe the proposed WGR method. We present the implementation details in Section \ref{imp}.
In Section \ref{sec: thm} we establish  non-asymptotic error bounds for the proposed estimator and show that it is consistent. In Section \ref{EXPs} we conduct numerical experiments, including simulation studies and real data analysis, to evaluate the performance of the proposed method. Technical proofs and additional numerical experiments are given in Appendix.


\section{Method}\label{Method}

Consider a pair of random vectors  $(X, Y) \in \cX \times \cY$,  where $X$ is a vector of predictors and $Y$ is a vector of response variables. Suppose $\cX \subseteq \mathbb{R}^d$ and $\cY \subset \mathbb{R}^q$ with $d, q \ge 1$. We allow either or both of $X$ and $Y$ to be high-dimensional.
The predictor $X$ or the response $Y$ can contain both continuous and categorical components. Our goal is to learn  the conditional distribution of $Y$ given $X=x$ and
 estimate the regression function $\Ebb(Y|X=x)$ in a unified framework.

We describe the proposed WGR method in detail below, which
has three main ingredients,
a conditional distribution generator,
a quadratic loss for regression,  and
the Wasserstein metric for distribution matching.

\subsection{Conditional generator}

The theoretical foundation of WGR is the noise outsourcing lemma \citep{kall2002}. 
This lemma states that, if $\cY$ is a standard Borel space \citep{standardBorel2009}, there exist
a Borel-measurable function $g^*:  \cX \times \mathbb{R}^m \to \cY$ and a random variable
$\eta \sim \text{Uniform}[0,1]$ such that $\eta$ is independent of $X$ and
\begin{equation}
\label{noise-out1}
			(X, Y)= (X, g^*(X,\eta))  \  \ \ \ \text{almost surely.}
\end{equation}
We note that the condition that $\cY$ being a standard Borel space is satisfied in all the applications we are interested in. In (\ref{noise-out1}),  for simplicity, $\eta$ is taken to be a uniform random variable. In general, we can  take $\eta$ to be a random vector from  a given reference distribution $P_{\eta}$ that is easy to sample from. For example, we can take $P_{\eta}$ to be the standard multivariate normal $N(\mathbf{0}, \bI_m)$ with $m \ge 1,$ which allows us to control the noise level more easily in practice.

We call the function $g^*$ in (\ref{noise-out1}) a conditional generator, since if $g^*$ satisfies (\ref{noise-out1}), it also satisfies
\begin{equation}
\label{noise-out2}
g^*(\eta, x) \sim P_{Y|X=x},  \eta \sim P_{\eta}, x \in \cX.
\end{equation}
So for a given $x$, to sample from the conditional distribution $P_{Y|X=x}$, we can first generate
$\eta \sim P_{\eta}$, then calculate $g^*(\eta, x)$, which gives a sample from $P_{Y|X=x}.$
 In addition, we can calculate any moments of
$P_{Y|X=x}$ via $g^*(\cdot, x)$. In particular, we have
\[
\Ebb(Y|X=x) = \Ebb_{\eta} [g^*(\eta, x)], x \in \cX.
\]

In summary,  we can determine the usual regression function  (the conditional mean)
and sample from the conditional distribution as follows:
\begin{itemize}
 \setlength\itemsep{-0.05 cm}
\item
Regression function or conditional mean:
$
\mathbb{E}(Y|X=x) = \mathbb{E}_{\eta} g^*(x; \eta), x \in \mathcal{X},
$
\item
Conditional distribution:
$
g^*(x;\eta) \sim P_{Y|X=x}, x \in \cX, \eta \sim P_{\eta}.
$
\end{itemize}
Hence, the conditional generator provides a basis for
a unified framework for nonparametric regression and conditional distribution learning.

\subsection{Objective function}
Let $P_{X, g}$ denote the joint distribution of $(X, g(X,\eta))$, which is the generated distribution based on a conditional generator $g(x, \eta), \eta \sim P_{\eta}, x \in \cX.$
One possible way to measure the quality of a conditional generator $g$ is to compare the generated distribution $P_{X,g}$ with the data distribution $P_{X,Y}$.
A good conditional generator $g$ should ensure that the generated distribution is close to the data distribution in some sense. For example, one can use a distance metric such as the Wasserstein distance or a divergence measure such as the Kullback-Leibler divergence to quantify the discrepancy between the two distributions.

Let $\Dbb$ be a divergence measure
for the difference between $P_{X,g}$ and $P_{X,Y}$. Then, we formulate an objective function that combines this divergence measure with the least squares loss to minimize the distribution mismatch and the prediction error simultaneously. The objective function is
\begin{equation}
\label{obj1}
\lambda_{w}  { \Dbb (P_{X,g} \Vert P_{X,Y})} +  \lambda_{\ell} \mathbb{E}\Vert Y-\mathbb{E}_\eta g(X,\eta)\|^2.
\end{equation}
Here, both $\lambda_{\ell}$ and $\lambda_{w}$ are tuning parameters weighing two losses, which are assumed to be nonnegative and $\lambda_{\ell}+\lambda_{w}=1$.
The objective function (\ref{obj1}) combines two types of losses: the first one evaluates how closely the generated distribution $P_{X,g}$ resembles the data distribution $P_{X,Y};$ the second one is a criterion quantifying how well the regression function fits the data. Intuitively, the objective function (\ref{obj1}) tries to learn the conditional distribution of $Y$ given $X$ with the regularization that the conditional mean is well estimated.

We take $\Dbb$ to be the
\textit{1-Wasserstein} distance.
A computationally convenient form of the 1-Wasserstein distance between $P_{X,g}$ and $P_{X,Y}$ is the Monge-Rubinstein dual \citep{villani2008optimal2}:
\begin{align}
\Dbb_W(P_{X,g}, P_{X,Y}) = \sup_{f\in \mathcal{F}^1_{\text{Lip}}}\big\{
\Ebb_{(X, \eta)} f(X,g(X,\eta))-\Ebb_{(X, Y)}
 f(X,Y)\big\}, \label{wass_divergence}
\end{align}
where
$\mathcal{F}^1_{\text{Lip}}=\{f: \cX \times \cY \to \mathbb{R}, |f(u)-f(v)| \le \|u-v\|_2, \forall u, v \in \cX \times \cY\}$ is a 1-Lipschitz class of functions on  $\cX \times \cY.$
The Lipschitz function $f$ in (\ref{wass_divergence}) is often called a critic or a discriminator.

Then, based on (\ref{obj1}),
the population objective function for the proposed Wasserstein generative regression (WGR) is:
\begin{align}
\label{obj2}
L(g)=  \lambda_{w}\sup_{f\in \mathcal{F}^1_{\text{Lip}}} L_{\text{W}}(g, f)+\lambda_{\ell} L_{\text{LS}} (g),
\end{align}
where
\begin{align*}
 L_{\text{W}}(g, f)&= \Ebb_{(X, \eta)} f(X,g(X,\eta))-\Ebb_{(X, Y)}f(X,Y),\\
 L_{\text{LS}}(g)&= \mathbb{E}_{(X, Y)}\Vert Y- \mathbb{E}_\eta g(X,\eta)\Vert^2.
\end{align*}

Suppose we have a random sample  $\{(X_i,Y_i), i=1,2,\ldots,n\}$  from $P_{X,Y}$, where $n \ge 1$
is the sample size.
Let $\{\eta_{i}, i=1,\ldots,n\}$ and $\{\eta_{ij}, i=1,2,\ldots,n, j=1, \ldots, J\}$ with $J\ge 1$ be  random variables generated independently from $P_{\eta}$. We parameterize the generator function $g$ and the discriminator $f$ by neural network functions $g_{\bm{\theta}}$ and $f_{\bm{\phi}}$ with parameters (weights and biases) $\bm{\theta}$ and $\bm{\phi}$, respectively. That is, we use neural network functions to approximate the generator and critic functions and optimize the objective function given below over the neural networks to obtain an estimator of $g$.
In addition, since $\Ebb_{\eta} g_{\bm{\theta}}(X_i, \eta)$ generally does not have a close form expression, we approximate it by the sample average
 $ J^{-1}\sum_{j=1}^J g_{\bm{\theta}}(X_i,\eta_{ij}).$
Then, the empirical objective function for estimating $(\bm{\theta}, \bm{\phi})$ is
\begin{align}\label{eqn: empirical_loss}
 \widehat{L}(g_{\bm{\theta}},f_{\bm{\phi}})=
  \lambda_{w} \widehat{L}_{\text{W}}(g_{\bm{\theta}},f_{\bm{\phi}})+
 \lambda_{\ell}\widehat{L}_{\text{LS}}(g_{\bm{\theta}}),
\end{align}
where
\begin{align*}
 \widehat{L}_{\text{W}}(g_{\bm{\theta}},f_{\bm{\phi}})&=
 \frac{1}{n}\sum_{i=1}^n\left\{f_{\bm{\phi}}(X_i,g_{\bm{\theta}}(X_i,\eta_i))-
 f_{\bm{\phi}}(X_i,Y_i)\right\}, \\
 \widehat{L}_{\text{LS}}(g_{\bm{\theta}})&=   \frac{1}{n}\sum_{i=1}^{n}\big\| Y_i-
 \frac{1}{J}\sum_{j=1}^J g_{\bm{\theta}}(X_i,\eta_{ij})\big\|^2.
\end{align*}
Let $(\hat{\bm{\theta}}, \hat{\bm{\phi}})$ be a solution to the minimax problem
\begin{align}\label{eqn:generatior_estimator}
(\hat{\bm{\theta}},\hat{\bm{\phi}})=\arg\min_{\bm{\theta}}\max_{\bm{\phi}} \widehat{L}(g_{\bm{\theta}}, f_{\bm{\phi}}).
\end{align}
Then, the estimated conditional generator is $\hat{g}(x, \eta)=g_{\hat{\bm{\theta}}}(x,\eta)$ and the estimated regression function is obtained by taking the expectation of $\hat{g}(x, \eta)$ with respect to
$\eta$, that is,
$\hat{g}(x) = \Ebb_{\eta}\hat{g}(x, \eta).$  Since there is no analytical expression for
the expectation $ \Ebb_{\eta}\hat{g}(x, \eta),$  we approximate it using an empirical average based on a random sample  $\{\eta_1', \ldots, \eta_K'\}, K \ge 1, $ from $P_{\eta}$,
\[
\hat{g}(x) \approx \frac{1}{K}\sum_{k=1}^K \hat{g}(x, \eta_k'),
\]
which gives the estimated regression function.

We note that for a given $x \in \cX$,  $\{\hat{g}(x, \eta_k'), k=1, \ldots, K\}$ are approximately
distributed as $P_{Y|X=x}.$ We can use $\{\hat{g}(x, \eta_k'), k=1, \ldots, K\}$ to explore any aspects of $P_{Y|X=x}$ that we are interested in such as its higher moments and quantiles.

\section{Implementation}\label{imp}
In this section, we present the details for implementing WGR. We first describe the neural networks used in the approximation of $g$ and $f$. We then present the computational algorithm we implemented in detail.

\subsection{ReLU Feedforward Neural Networks }\label{FNN}
We first give a brief description of feedforward neural networks (FNN) with rectified linear unit (ReLU) activation function.
The ReLU function is denoted by $\sigma(x)\coloneqq \max(x,0)$, and it is defined for each component of $x$ if $x$ is a vector.
A neural network  can be expressed as a composite function
$
\zeta(x)=\mathcal{L}_{H} \circ \sigma \circ \mathcal{L}_{H-1} \circ \sigma \circ \cdots \circ \sigma \circ \mathcal{L}_{1} \circ \sigma \circ \mathcal{L}_{0}(x), x \in \mathbb{R}^{p_{0}},
$
where $\mathcal{L}_i(x)=W_ix+b_i$ with a weight matrix $W_i\in \mathbb{R}^{p_{i+1}\times p_{i}}$ and bias vector $b_i\in\mathbb{R}^{p_{i+1}}$ in the $i$-th linear transformation, and $p_i$ is the width of the $i$-th layer,  $i=0,1,\ldots,H$.
The width and depth of the network are described by $W=\max\{p_1,\ldots,p_{H}\}$ and $H$, respectively.
To ease the presentation, we use $\mathcal{NN}(p_0,p_{H+1}, W,H)$ to denote the neural networks with input dimension $p_0$, output dimension $p_{H+1}$, width at most $W$ and depth at most {\color{black}H}.

We now specify the function classes below:
\begin{itemize}
\item For the generator network class $\mathcal{G}$: Let $\mathcal{G}\equiv\mathcal{NN}(d+m,q,W_{\mathcal{G}},H_{\mathcal{G}})$ be a class of ReLU-activated FNNs,  
     $g_{\bm{\theta}}:\mathbb{R}^{d+ m}\to \mathbb{R}^{q},$ with parameter $\bm{\theta}$, width $W_{\mathcal{G}}$, and depth $H_{\mathcal{G}}$.
\item For the discriminator network class  $\mathcal{D}$: Let $\mathcal{D}\equiv\mathcal{NN}(d+q,1,W_{\mathcal{D}},H_{\mathcal{D}})\cap \text{Lip}(\Omega; K_{\mathcal{D}})$ be a class of ReLU-activated FNNs,
    $f_{\bm{\phi}}:\Omega\to \mathbb{R},$ with parameter $\bm{\phi}$, width $W_{\mathcal{D}}$, and depth $H_{\mathcal{D}}$, where  for some $K_{\mathcal{D}} > 0$, $\text{Lip}(\Omega; K_{\mathcal{D}})$ is a class of Lipschitz functions defined below.

For any function $f:\Omega\to\mathbb{R}$, the Lipschitz constant of $f$ is denoted by
\begin{equation*}
    \text{Lip}(f) = \sup_{x,y\in \Omega, x\neq y}\frac{|f(x)-f(y)|}{\|x-y\|}.
\end{equation*}
For a given $0 < K < \infty$, denote $\text{Lip} (\Omega;K)$ as the set of all functions $f: \Omega\mapsto \mathbb{R}$ with $\text{Lip}(f) \leq K$. And let
$\text{Lip} (\Omega;K,C):=\{f \in \text{Lip} (\Omega;K):\|f\|_{\infty} \leq C\}$, where $\|f\|_{\infty}=\sup_{x\in\Omega}\|f(x)\|_{\infty}$ and $C$ is a positive constant.
 Hence, $\mathcal{F}^1_{\text{Lip}}$ defined in (\ref{wass_divergence}) is $\text{Lip} (\mathbb{R}^{d+q};1)$.
\end{itemize}

\begin{algorithm}
	\caption{WGR Algorithm}
    \label{alg: WGR}
	\begin{algorithmic}[1]
	\STATE {\bfseries Require:}
	 (a) Labeled data $\{(X_i,Y_i)\}_{i=1}^n$; (b) Minibatch size $v \leq \min(n,N)$; (c) $J$, the size of noise vector $\eta$ to compute the conditional expectation; $\lambda$, the gradient penalty parameter.
	\FOR {number of training iterations}
		\STATE Sample i.i.d. $\{\eta_{ij},i=1,2,\ldots,n,j=1,\ldots,J\}$ from the standard multivariate normal distribution.
		 \STATE With fixed $\bm{\theta}^{(v)}$, update the discriminator $f_{\bm{\phi}}$ by ascending its stochastic gradient:
		 \begin{small}
	\begin{align}\label{Alg-D}
	  \nabla_{\phi} \frac{\lambda_{w}}{n} \sum_{i=1}^{n} \left\{f_{\bm{\phi}}\left(X_{i}, g_{\bm{\theta}}\left( X_{i},\eta_{i0}\right)\right)
	  -f_{\bm{\phi}}\left(X_{i}, Y_{i}\right)-\lambda\left(\left\|\nabla_{(x, y)} f_{\bm{\phi}}\left(X_{i}, Y_{i}\right)\right\|_{2}-1\right)^{2}\right\}
	\end{align}
	\end{small}
	where the resulting parameter is denoted by $\bm{\phi}^{(v+1)}$.
         \STATE With fixed $\bm{\phi}^{(v+1)}$, update the generator $g_{\bm{\theta}}$ by descending its stochastic gradient:
        \begin{small}
	\begin{align*}
	 \nabla_{\theta} \left[\frac{\lambda_{l}}{n}\sum_{i=1}^{\tilde{n}} \left\{Y_i-\frac{1}{J}\sum_{j=1}^{J}g_{\bm{\theta}}(X_i,\eta_{ij})\right\}^2+ \frac{\lambda_{w}}{n} \sum_{i=1}^{n}f_{\bm{\phi}}\left(X_{i}, g_{\boldsymbol{\theta}}\left( X_{i},\eta_{i0}\right)\right)\right],
	\end{align*}
	\end{small}
	where the resulting parameter is denoted by $\bm{\theta}^{(v+1)}$.
		\ENDFOR
	\end{algorithmic}
	\footnotesize
\end{algorithm}

\subsection{Computation}\label{comp}
We now describe the implementation of WGR. For training the conditional distribution generator $g_{\bm{\theta}}$ and the discriminator $f_{\bm{\phi}}$, we use the leaky rectified linear unit (leaky ReLU) as the activation function in $g_{\bm{\theta}}$ and $f_{\bm{\phi}}$. The training algorithm is 
presented in Algorithm \ref{alg: WGR}. We have implemented it in Pytorch.

To constrain the discriminator $f_{\bm{\phi}}$ to the class of 1-Lipschitz functions, a gradient penalty is used in   (\ref{Alg-D}), which is a slightly modified version of
 the algorithm proposed by \citep{NIPS2017_Gulrajani}. The difference is that we evaluate the gradients at the sample points
in the penalty, instead of using generated intermediate points.
Another approach that we have tried to enforce the Lipschitz constraint is the clipping method \citep{arjovsky17}, which also produces acceptable results, but seems to be less stable than the penalty method described in Algorithm \ref{alg: WGR}.
We use traversal to select the tuning parameters $\lambda_{l}$ and $\lambda_{w}$ that control the trade-off between label and word embeddings. The constraints are that $\lambda_{l}$ and $\lambda_{w}$ must add up to 1 and have one decimal place each.
In Section \ref{EXPs}, we demonstrate the effectiveness of this algorithm in various numerical experiments. However, we do not have a theoretical analysis of its convergence behavior and we leave this as an open problem for future research.

\section{Error analysis and convergence} 
\label{sec: thm}

In this section, we first develop an error decomposition, which decomposes the estimation errors into approximation errors and stochastic errors of the generator and discriminator.
We then derive non-asymptotic error bounds for WGR based on this error decomposition.

\subsection{Error decomposition} \label{sec: error_decomp}

We present a high-level description of the error decomposition for WGR.
For the estimator $\hat{g}$ define in (\ref{eqn:generatior_estimator}), the estimation error consists of two parts:
the $L_2$-based excess risk $\mathbb{E}\Vert \mathbb{E}_\eta\hat{g}(X,\eta)-\mathbb{E}_\eta g^*(X,\eta)\Vert^2$, and the integral probability metric \citep{muller} $d_{\mathcal{F}^1_B}(P_{X,\hat{g}}, P_{X,Y})$ defined as
\begin{equation*}
    d_{\mathcal{F}^1_B}(P_{X,\hat{g}}, P_{X,Y}) = \sup_{f \in \mathcal{F}^1_B} \{\mathbb{E}_{(X,\eta)}f(X,\hat{g}(X,\eta))-\mathbb{E}_{X,Y}f(X,Y)\},
\end{equation*}
where $\mathcal{F}^1_B =\{f: \mathbb{R}^{d+q} \mapsto \mathbb{R}, |f(z_1)-f(z_2)|\leq \|z_1-z_2\|, z_1, z_2\in \mathbb{R}^{d+q}, \|f\|_{\infty}\leq B\}$ is the bounded 1-Lipschitz function class.
Clearly,  if $P_{X,Y}$ has a bounded support,  $d_{\mathcal{F}^1_B}$ is   the 1-Wasserstein distance.
A function class $\mathcal{F}$ is called symmetric if $f\in\mathcal{F}$ implies $-f\in\mathcal{F}$.

We introduce a new error decomposition method in Lemma \ref{lem: error_decomp}, which  decomposes the estimation error into approximation error and stochastic error of the generator and discriminator.


\begin{lemma} \label{lem: error_decomp}
    Assume that the discriminator network class $\mathcal{D}$ is symmetric and the probability measures of $(X, Y)$ and $(X, g(X,\eta))$ are supported on a compact set $\Omega \subseteq \mathbb{R}^{d+q}$ for any $g \in \mathcal{G}$. Then, for the WGR estimator defined in (\ref{eqn:generatior_estimator}),
        \begin{align}
        &\mathbb{E}_{\mathcal{S}}\!\left\{\!\lambda_{\ell} \mathbb{E}\| \mathbb{E}_{\eta}\hat{g}(X,\eta)\!-\! \mathbb{E}_{\eta}g^*(X,\eta)\|^2\!+\!\lambda_{w}d_{\mathcal{F}^1_B}(P_{X,\hat{g}}, P_{X,Y})\!\right\}\nonumber\\
        & \quad \leq  \lambda_{l}\mathcal{E}_1+ 4\lambda_{l}\mathcal{E}_2+ 2\mathcal{E}_3+ 2\lambda_{w}\mathcal{E}_4+ 3\lambda_{w}\mathcal{E}_5+ 3\lambda_{w}\mathcal{E}_6,\label{lemma:error_inequ}
\end{align}
where $\mathcal{S}=\{(X_i,Y_i,\eta_{i})\}_{1\leq i \leq n} \cup \{\eta_{ij}\}_{1\leq i \leq n, 1\leq j \leq J}$ and
\begin{align*}
\mathcal{E}_1 &:= \mathbb{E}_{\mathcal{S}}\Big\{\mathbb{E}\Vert Y-\mathbb{E}_{\eta}g^*(X,\eta)\Vert^2+\mathbb{E}\Vert Y- \mathbb{E}_\eta \hat{g}(X,\eta)\Vert^2  -\frac{2}{n}\sum_{i=1}^{n}\| Y_i-\mathbb{E}_{\eta}\hat{g}(X_i,\eta)\|^2 \Big\},\\
        \mathcal{E}_2 &:= \mathbb{E}_{\mathcal{S}}\Big\{\sup_{g \in \mathcal{G}}\frac{1}{n}\sum_{i=1}^n | \|Y_i-\frac{1}{J}\sum_{j=1}^{J} g(X_i,\eta_{ij}) \|^2 -\|Y_i-\mathbb{E}_\eta g(X_i,\eta)\|^2 | \Big\},\\
        \mathcal{E}_3 &:= \inf_{g \in \mathcal{G}} \Big[\lambda_{l}\mathbb{E}\Vert \mathbb{E}_\eta g(X,\eta) -\mathbb{E}_{\eta}g^*(X,\eta)\Vert^2 +\lambda_{w}\sup_{f \in \mathcal{D}}\{\mathbb{E}f(X,g(X,\eta))-\mathbb{E}f(X,Y)\}\Big],\\
        \mathcal{E}_4 &:= \sup_{h \in \mathcal{F}^1_B} \inf_{f \in \mathcal{D}}\|h-f\|_{\infty},\\
        \mathcal{E}_5 &:= \mathbb{E}_{\mathcal{S}}\Big [\sup_{f \in \mathcal{D}} \Big\{\mathbb{E}f(X,Y)- \frac{1}{n}\sum_{i=1}^n f(X_i,Y_i)\Big\}\Big], \\
         \mathcal{E}_6 &:= \mathbb{E}_{\mathcal{S}}\Big [\sup_{f \in \mathcal{D}, g \in \mathcal{G}} \Big\{\mathbb{E}f(X,g(X,\eta))-\frac{1}{n}\sum_{i=1}^nf(X_i,g(X_i,\eta_i))\Big\}\Big ].
\end{align*}
\end{lemma}

According to their definitions, $\mathcal{E}_1, \mathcal{E}_2, \mathcal{E}_5$ and $\mathcal{E}_6$ are stochastic errors;  $\mathcal{E}_3$ and $\mathcal{E}_4$ are approximation errors.
%
Lemma \ref{lem: error_decomp} provides  a general error decomposition method, which covers the error decomposition inequality in \citet{jiao2021deep}  
 for the traditional nonparametric regression as a special case (corresponding to the case that $\lambda_{\ell}=1, \lambda_{w}=0$ and $J=\infty$). It can also be utilized for the error analysis for the conditional WGAN 
(corresponding to the case that $\lambda_{\ell}=0, \lambda_{w}=1$).
Moreover, when $\lambda_{\ell}=0$ and $\lambda_{w}=1$,
our error decomposition result in (\ref{lemma:error_inequ})  is in line with that in Lemma 9 in \citet{huang2022error} 
for the general GANs.
The main difference is that the upper bound of $\mathcal{E}_6$  in (\ref{lemma:error_inequ})  depends on the sample size $n$, while it is determined by the sample size of the generated noise $\eta$ in \citet{huang2022error}. This is the key  difference in the theoretical analysis  between conditional WGAN and general WGAN.


\subsection{Non-asymptotic error bounds}\label{sec: nonasym_error_bound}

More notations are needed.
For $x\in\mathbb{R}^d$, its $\ell_1,\ell_2,\ell_{\infty}$-norm is defined as
$\|x\|_1=\sum_{k=1}^d |x_k|$, $\|x\|_2=(\sum_{k=1}^d |x_k|^2)^{1/2}$ and $\|x\|_{\infty}=\max_{1\leq k \leq d} |x_k|$, respectively.
Let $\mathbb{N}$ be the set of positive integers and $\mathbb{N}_0\coloneqq\mathbb{N}\cup \{0\}$.
The maximum and minimum of $A$ and $B$ are denoted by $A\vee B$ and $A \wedge B $.
Let $\lfloor A \rfloor$ be the largest integer strictly smaller than $A$ and $\lceil A\rceil$ be the smallest integer strictly larger than $A$.
For any $\beta>0$ and a set $\Gamma \subseteq \mathbb{R}^{m+d}$, the H\"older class of functions $\mathcal{H}^{\beta}(\Gamma,B_1)$ with a constant $0<B_1< \infty$ is defined as
\begin{small}
\begin{align*}
    \mathcal{H}^{\beta}(\Gamma,B_1) = \left\{f:\Gamma \rightarrow \mathbb{R}, \max_{\|\alpha\|_1\leq \lfloor \beta\rfloor}\|D^{\alpha}f\|_{\infty}\leq B_1,\max_{\|\alpha\|_1=\lfloor \beta\rfloor}\sup_{x, y \in \Gamma,
    x\neq y}\frac{|\partial^{\alpha}f(x)-\partial^{\alpha}f(y)|}{\|x-y\|^{r}}\leq B_1\right\},
\end{align*}
\end{small}
where $\partial^{\alpha}=\partial^{\alpha_1}\cdots \partial^{\alpha_{m+d}}$ with $\alpha=(\alpha_1,\ldots,\alpha_{m+d})^{\top} \in \mathbb{N}_0^{m+d}$.

The following assumptions are needed.
\begin{condition}\label{C1}
The probability measures of $(X, Y)$ and $(X, g(X,\eta))$ are
supported on a compact set $\Omega_X \times \Omega_Y \subseteq [-B_0,B_0]^{d+q}\subseteq \mathbb{R}^{d+q}$ for any $g \in \mathcal{G}$, where $0<B_0< \infty$ is a constant.
\end{condition}
\begin{condition}\label{C2}
The probability measure of $\eta$ is
supported on $\Omega_{\eta} \subseteq [-B_0,B_0]^m$.
\end{condition}
\begin{condition}\label{C3}
For $g^*=(g^*_1,\ldots,g^*_q)^{\top}$, $g^*_k \in \mathcal{H}^{\beta}(\Omega_{\eta}\times \Omega_X,B_1), k=1,\ldots, q$, where $\beta>0$ and $0<B_1<\infty$.
\end{condition}
\begin{condition}\label{C4}
For any $x\in \Omega_X$, there exists a vector $\eta_x \in \Omega_{\eta}$ such that for any $g,\tilde{g} \in \mathcal{G}$,
\begin{equation*}
    \|\mathbb{E}_{\eta}g(x,\eta)-\mathbb{E}_{\eta}\tilde{g}(x,\eta)\|_1 \leq \|g(x,\eta_x)-\tilde{g}(x,\eta_x)\|_1.
\end{equation*}
\end{condition}

Let $W,\bar{W},\bar{H} \in \mathbb{N}$, which may depend on $n$. 
We also make the following assumptions on the network classes $\mathcal{D}$ and $\mathcal{G}$.

\begin{ND}\label{ND1}
The discriminator ReLU network class $\mathcal{D}=\mathcal{NN}(d+q,1,W_{\mathcal{D}}, H_{\mathcal{D}}) \cap \text{Lip} ([-B_0,$\\$B_0]^{d+q};K_{\mathcal{D}})$ has width $W_{\mathcal{D}}=W^{d+q}\{9(W+1)+5(d+q)-1\}$, depth $H_{\mathcal{D}}=3+14(d+q)(d+q-1)$
and Lipschitz constant $K_{\mathcal{D}}\leq 54B2^{d+q}(d+q)^{1/2}W^2$. 
\end{ND}

\begin{NGS}\label{NG1}
The generator ReLU network class $\mathcal{G}=\mathcal{NN}(m+d,q, W_{\mathcal{G}}, H_{\mathcal{G}})$ has width $W_{\mathcal{G}}=38q(\lfloor\beta\rfloor+1)^23^{(m+d)}(m+d)^{\lfloor\beta\rfloor+1}\bar{W}\lceil\log_2(8\bar{W})\rceil$ and depth $H_{\mathcal{G}}=21(\lfloor\beta\rfloor+1)^2\bar{H}$\\$\lceil\log_2(8\bar{H})\rceil+2(m+d)$. 
\end{NGS}
Conditions \ref{C1}-\ref{C3} require that $P_{X,Y}$, $P_{X,g}$ and $P_{\eta}$ have a bounded support. Condition \ref{C3} is a smoothness condition for $g^*$ in (\ref{noise-out1}). Condition \ref{C4} is a technical condition.
The upper bound of the Lipschitz constant in {\bf ND} \ref{ND1} is needed to achieve small approximation error.
More details can be found in
Lemma \ref{lem: approx_error_D} in the Appendix.
To lighten the notations, we define the following two quantities:
\begin{align*}
    a &:= \frac{\beta}{2\beta+\{3(m+d)\}\vee \{2\beta(d+q+1)\}},\\
    b &:= \frac{3(m+d)}{2[2\beta+\{3(m+d)\}\vee \{2\beta(d+q+1)\}]}.
\end{align*}

\begin{theorem}\label{thm: non-asymptotic_LS}
    Suppose that Conditions \ref{C1} - \ref{C4} hold and the network parameters of $\mathcal{D}$ and $\mathcal{G}$  satisfy ${\bf ND}$ \ref{ND1} and ${\bf NG}$ \ref{NG1} with $W=\lceil n^a\rceil$, $\bar{W}=\lceil n^b/\log^2 n \rceil$ and $\bar{H}=\lceil \log n\rceil$. Then, for $J\gtrsim n$ and  given weights $\lambda_{l}$ and $\lambda_{w}$ satisfying $0<\lambda_{l}, \lambda_{w}< 1$ and $\lambda_{l}+\lambda_{w}=1$, 
   we have
        \begin{align*}
        \mathbb{E}_{\mathcal{S}}\left\{\mathbb{E}_X\| \mathbb{E}_{\eta}\hat{g}(X,\eta)- \mathbb{E}_{\eta}g^*(X,\eta)\|^2\right\} \leq C_1 n^{-\frac{\beta}{2\beta+\{3(m+d)\}\vee \{2\beta(d+q+1)\}}}(\log n)^{\frac{2\beta}{m+d}\vee1},
    \end{align*}
    where $C_1$ is a positive constant independent of $n$ and $J$. 
\end{theorem}

Theorem \ref{thm: non-asymptotic_LS} establishes a non-asymptotic upper bound for the excess risk of WGR using deep neural networks.
The convergence rates in Theorem \ref{thm: non-asymptotic_LS} is slightly slower than $O(n^{-\frac{2\beta}{2\beta+d}}\log^{c}n)$, the rate  in 
 deep least squares nonparametric regression as in \cite{jiao2021deep}. 
This is because
in
 nonparametric regression,
there is no distributional matching constraint, thus the noise vector $\eta$ is not involved and a faster convergence rate can be achieved.
In our proposed  framework, we are not only interested in estimating  the  mean regression function, but also the conditional generator, which involves the noise vector from a reference distribution.
This increases the dimensionality of the problem and results in a slower convergence rate.

We next establish a non-asymptotic error bound for the integral probability metric $d_{\mathcal{F}^1_B}(P_{X,\hat{g}}, P_{X,Y})$.

\begin{theorem}\label{thm: non-asymptotic_WGAN}
     Suppose that those conditions of Theorem \ref{thm: non-asymptotic_LS} hold. Then, for
      $J\gtrsim n$ and { given weights $\lambda_{l}$ and $\lambda_{w}$ satisfying $0\leq\lambda_{l}< 1,0<\lambda_{w}\leq1$ and $\lambda_{l}+\lambda_{w}=1$}, 
  we have
    \begin{align*}
       \mathbb{E}_{\mathcal{S}}\left\{d_{\mathcal{F}^1_B}(P_{X,\hat{g}}, P_{X,Y})\right\} \leq  C_2 n^{-\frac{\beta}{2\beta+\{3(m+d)\}\vee \{2\beta(d+q+1)\}}}(\log n)^{\frac{2\beta}{m+d}\vee1},
    \end{align*}
    where $C_2$ is a positive constant independent of $n$ and $J$. 
\end{theorem}

The non-asymptotic error bounds in Theorem \ref{thm: non-asymptotic_LS} and Theorem \ref{thm: non-asymptotic_WGAN} are established for  fixed positive weights $\lambda_{l}$ and $\lambda_{w}$.
 In this case, we obtain the same convergence rate of the excess risk
  $\mathbb{E}\Vert \mathbb{E}_\eta\hat{g}(X,\eta)-\mathbb{E}_\eta g^*(X,\eta)\Vert^2$ and the integral probability metric $d_{\mathcal{F}^1_B}(P_{X,\hat{g}}, P_{X,Y})$.
Note that the conditional GAN is involved in our proposed procedure, thus the joint stochastic error of the discriminator and generator is affected by the dimension of the noise vector $\eta,$
 leading to a slower convergence rate compared with the one in Theorem 5 in \citet{huang2022error} for GAN estimators.

 Next, we present a non-asymptotic error bound for varying weights $\lambda_{l}$ and $\lambda_{w}$, which can diverge with the sample size $n$.


\begin{theorem}\label{thm: varying_weights}
    Suppose that Conditions \ref{C1} - \ref{C4} hold and the network parameters of $\mathcal{D}$ and $\mathcal{G}$  satisfy ${\bf ND}$ \ref{ND1} and ${\bf NG}$ \ref{NG1} with $W=\lceil n^{a} \rceil$, $\bar{W}=\lceil n^{b}/\log^2 n\rceil$ and $\bar{H}=\lceil \log n\rceil$.
    Then, for $\lambda_{l}>0,\lambda_{w}>0$ satisfying $\lambda_{l}+\lambda_{w}=1$ and $\lambda_{w}=O(n^{-1/(d+q+2)})$,
    when $2\beta(d+q+1)\geq 3(m+d)+\beta$ and $J\gtrsim n^{\{3(m+d)+6\beta\}/\{4\beta(d+q+2)\}}$,  we have
    \begin{align*}
        \mathbb{E}_{\mathcal{S}}\left\{\mathbb{E}\| \mathbb{E}_{\eta}\hat{g}(X,\eta)- \mathbb{E}_{\eta}g^*(X,\eta)\|^2\right\}
        \leq C_3 n^{-\frac{3}{2(d+q+2)}}(\log n)^{\frac{2\beta}{m+d}\vee2},
    \end{align*}
    where $C_3$
    is a positive constant independent of $n$ and $J$. Moreover, as $n\rightarrow \infty$,
    \begin{align*}
        \mathbb{E}_{\mathcal{S}}\left\{d_{\mathcal{F}^1_B}(P_{X,\hat{g}}, P_{X,Y})\right\}\rightarrow 0.
    \end{align*}
\end{theorem}

When $\lambda_{w}=O(n^{-1/(d+q+2)})$, Theorem \ref{thm: varying_weights} gives an improved non-asymptotic error bound of the proposed estimator, and it also implies that our estimated distribution $P_{X,\hat{g}}$ converges weakly to $P_{X,Y}$ as $n\to \infty$.

\section{Numerical studies}
\label{EXPs}

In this section, a number of  experiments including simulation studies and real data examples  are conducted to assess the performance of the proposed method.
We implement WGR in Pytorch, and use the stochastic gradient descent algorithm RMSprop in the training process.

For comparison, we also compute the nonparametric least squares regression using neural networks (NLS)  and conditional Wasserstein GAN (cWGAN) \citep{arjovsky17, liu2021wasserstein}.
Aside from the results presented in this section, additional numerical results are provided in
the supplementary materials,
including the experiments with data generated from other models, experiments investigating the effects of the noise dimension $m$ and the size $J$, and neural networks with different architectures.


\subsection{Simulation Studies}\label{Sim}
We conduct simulation studies to evaluate the performance of WGR with univariate or multi-dimensional response $Y$.
We use five different models to generate data for our analysis. Each model has its own parameters and assumptions, which are summarized in Table \ref{sim-infor}. The table also shows the sample size and the architectures of the neural networks used in the analysis. We compare the performance of WGR  with other existing methods under these models.

\begin{model}\label{M1}
A nonlinear regression model with an additive error term: $$Y=X_{1}^{2}+\exp \left(X_{2}+X_{3} / 3\right)+\sin \left(X_{4}+X_{5}\right)+\varepsilon,$$ where $ \varepsilon \sim N(0,1).$
\end{model}
\begin{model}\label{M2}
 A nonlinear regression model with additive 
 heteroscedastic error: $$Y=X_{1}^{2}+\exp (X_{2} +X_{3}/ 3 )+X_{4}-X_{5}+ (0.5+X_{2}^{2} / 2+X_{5}^{2} / 2 ) \varepsilon,$$ where $\varepsilon \sim N(0,1)$.
\end{model}
\begin{model}\label{M-2d-1}
A  Gaussian mixture model: $$Y_1=X+\varepsilon_1,\ Y_2 = X+\varepsilon_2,$$ where
$\varepsilon_i \sim I_{\{U_i <1/3\}}N(-2,0.25^2) + I_{\{1/3 < U_i < 2/3\}}N(0,0.25^2)+ I_{\{U_i>2/3\}}N(2,0.25^2)$, $i=1,2,$
and $U_1\sim \text{Uniform}(0,1)$, $U_2\sim \text{Uniform}(0,1)$.
\end{model}
\begin{model}\label{M-2d-2}
Involute model: $$Y_1=2X+U\sin(2U)+\varepsilon_1, \ Y_2=2X+U\cos(2U)+\varepsilon_2,$$
where $U\sim \text{Uniform}(0,2\pi)$, $\varepsilon_1\sim N(0,0.4^2)$, $\varepsilon_2 \sim N(0,0.4^2)$.
\end{model}
\begin{model}\label{M-2d-3}
Octagon Gaussian mixture: $$Y_1 = X+\varepsilon_1, \ Y_2 = X+\varepsilon_2,$$
where
$\varepsilon_1\!\sim \!I_{\{U_1\in (i-1,i)\}}N(\bm{\mu}_i, \bm{\Sigma}_i),$  $\varepsilon_2\!\sim \!I_{\{U_2\in (i-1,i)\}}N(\bm{\mu}_i, \bm{\Sigma}_i),$
$ \bm{\mu}_i = (3\cos \frac{\pi_i}{4},3 \sin \frac{\pi_i}{4} )$, \\
$ \bm{\Sigma}_i =\left(\!\begin{array}{cc}
\cos ^{2} \frac{\pi i}{4}+0.16^{2} \sin ^{2} \frac{\pi i}{4} \!& \!\left(1-0.16^{2}\right) \sin \frac{\pi i}{4} \cos \frac{\pi i}{4} \\
\left(1-0.16^{2}\right) \sin \frac{\pi i}{4} \cos \frac{\pi i}{4}\! & \!\sin ^{2} \frac{\pi i}{4}+0.16^{2} \cos ^{2} \frac{\pi i}{4}
\end{array}\!\right)$, $U_1\sim \text{Uniform}(0,8)$,\\ $U_2\sim \text{Uniform}(0,8)$, $ i=1,\ldots,8$.
\end{model}

\begin{table}
\centering
    \caption{\label{sim-infor}Description of  simulation settings.}
    \begin{threeparttable}
    \begin{tabular}{c|cc}
    \hline
    \multicolumn{3}{c}{Data structure} \\
    \hline
        Model & M1, M2 & M3, M4, M5 \\
        Response $(Y)$ & $Y\in\mathbb{R}$ & $Y\in\mathbb{Y}^2$\\
        Covariate $(X)$ & $N(0,I_{5})$; $\  N(0,I_{100})$ & $N(0,1)$ \\
        Noise $(\eta)$ & $N(0,I_{3})$ & $N(0,I_{10})$\\
    \hline
   \multicolumn{3}{c}{ Sample size} \\
    \hline
        $J$ & 200 & 50\\
        Training & 5000 & 40000\\
        Validation & 1000 & 2000\\
        Testing& 1000 & 10000\\
    \hline
    \multicolumn{3}{c}{Network architecture} \\
    \hline
    Generator network & (32, 16) & (512, 512, 512)\\
    Discriminator network  $f_{\bm{\phi}}$ & (32, 16) & (512, 512, 512) \\
    \hline
    \end{tabular}
    \footnotesize
    Notes:  In Model 1 and Model 2, $X\sim N(0,I_{5})$ is the intrinsic dimensional case, and $X\sim N(0,I_{100})$ is the high dimensional case under sparsity assumption. $J$ is the size of the random noise sample. The networks used in the simulation are fully-connected feedforward neural networks with widths specified  above.
\end{threeparttable}
\end{table}

In the validation and testing stages, for each realization of $X$, we generate $J=500$  i.i.d. noise samples $\{\eta_j,j=1,\ldots,J\}$ from the standard normal distribution and compute $\{\hat{g}_{\bm{\theta}}(X,\eta_j),j=1,\ldots,J\}$.
To evaluate the predictive power of the WGR estimator $\hat{g}_{\theta}$, we use the $L_1$ and $L_2$ errors  defined as
\begin{align}
L_1=\frac{1}{n}\sum_{i=1}^n \|Y_i-\frac{1}{J}\sum_{j=1}^J\hat{g}_{\bm{\theta}}(X_i,\eta_{ij}) \|, \
L_2=\frac{1}{n}\sum_{i=1}^n \|Y_i-\frac{1}{J}\sum_{j=1}^J\hat{g}_{\bm{\theta}}(X_i,\eta_{ij}) \|^2.\label{L1L2}
\end{align}
In addition,
we can used the estimated conditional generator to obtain the estimated $\tau$-th conditional quantile for a given $X=X_{i}$, denoted by $\hat{F}_{Y| X}^{-1}(\tau| X=X_i)$, via Monte Carlo. Also, we can calculate the estimated conditional mean by $\hat{E}(Y|X=X_i)=(1/J)\sum_{j=1}^J g(X_i,\eta_{ij})$  and the estimated conditional standard deviation by $\hat{SD}(Y|X=X_i)=[(1/J)\sum_{j=1}^J\{g(X_i,\eta_{ij})-\hat{E}(Y|X=X_i)\}^2 ]^{1/2}$.
Hence, we can compute the mean squared error (MSE) of the estimated conditional mean, standard deviation and the estimated $\tau$-th quantile, defined by
\begin{align*}
&\text{MSE(mean)}=\frac{1}{K}\sum_{i=1}^K \{\hat{E}(Y| X=X_i)-E(Y|X=X_i) \}^2,\\
&\text{MSE(sd)}=\frac{1}{K}\sum_{i=1}^K \{\hat{SD}(Y|X=X_i)-SD(Y|X=X_i) \}^2,\\
&\text{MSE}(\tau)=\frac{1}{K} \sum_{i=1}^{K}\{\hat{F}_{Y|X}^{-1}(\tau|X=X_{i})-F_{Y|X}^{-1}(\tau|X=X_{i})\}^{2},
\end{align*}
where $K$ is the size of the validation or testing set. We consider $\tau= 0.05, 0.25, 0.50,$ $0.75, 0.95$.

\begin{figure}[H]
\begin{minipage}[t]{\linewidth}
\hspace{3.5cm} Truth \hspace{2.1cm} cWGAN  \hspace{1.8cm} WGR
\end{minipage}
\\
\begin{minipage}[t]{\linewidth}
\centering
	\rotatebox{90}{ \hspace{0.1cm}  Model \ref{M-2d-1}}
	\includegraphics[width=0.15\textheight]{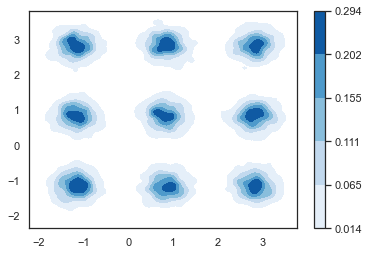}
	\includegraphics[width=0.15\textheight]{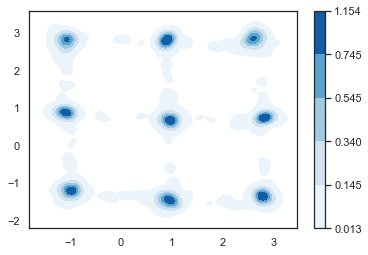}
	\includegraphics[width=0.15\textheight]{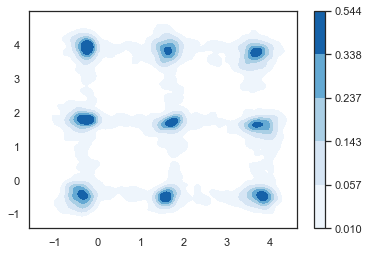}	
\end{minipage}
\\
\begin{minipage}[t]{\linewidth}
\centering
	\rotatebox{90}{ \hspace{0.1cm} Model \ref{M-2d-2}}
	\includegraphics[width=0.15\textheight]{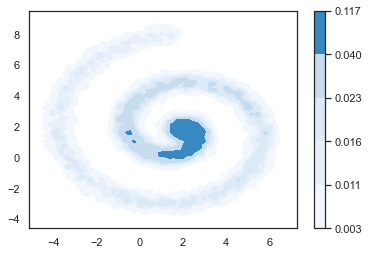}
	\includegraphics[width=0.15\textheight]{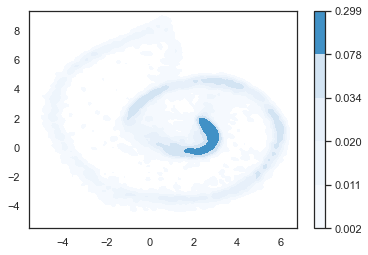}
	\includegraphics[width=0.15\textheight]{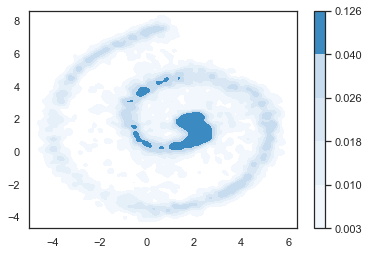}
\end{minipage}
\\
\begin{minipage}[t]{\linewidth}
\centering
	\rotatebox{90}{ \hspace{0.12cm} Model \ref{M-2d-3}}
	\includegraphics[width=0.15\textheight]{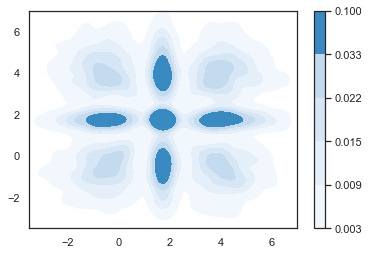}
	\includegraphics[width=0.15\textheight]{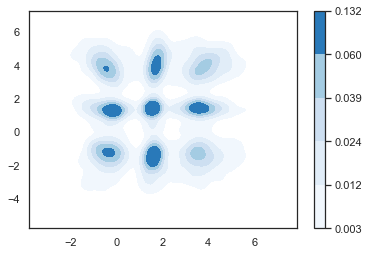}
	\includegraphics[width=0.15\textheight]{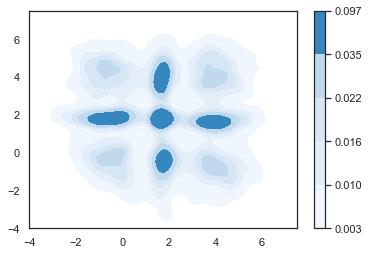}
\end{minipage}
\footnotesize
\caption{\label{Fig-2D}Comparison of conditional density estimation.  The abbreviations represent the same as in Table \ref{tab-sim-r}.  The conditional density functions are estimated using 5000 samples, which are generated by the conditional samplers from the methods given the randomly selected value of $X$.
}
\end{figure}

\begin{table}
\centering
\caption{\label{tab-sim-r} Comparison of  WGR with NLS and cWGAN for Models 1 and 2}
\begin{threeparttable}
\begin{tabular}{c| cc|cc cc}
\hline
Model & $d $ & Method & $L_1$ & $L_2$ & Mean & Sd\\
\hline
\multirow{6}{*}{ M\ref{M1}} &\multirow{3}{*}{$5$} & NLS & 0.83(0.02) &  1.08(0.05) &  - & -  \\
&& cWGAN &  1.00(0.02) &  1.97(0.25) & 0.98(0.25) & 0.09(0.03) \\
&& WGR & {\bf 0.82}(0.02) & {\bf 1.07}(0.05) & {\bf 0.06}(0.01) & {\bf 0.04}(0.01)\\
\cline{2-7}
&\multirow{3}{*}{$100$} & NLS &   1.17(0.03)  &  2.52(0.15) &  - & - \\
&& cWGAN & 1.17(0.04) & 2.65(0.41) & 1.67(0.39) & 0.81(0.02) \\
&& WGR &  {\bf 1.15}(0.04) &  {\bf 2.40}(0.24) & {\bf 1.64}(0.25) & {\bf 0.16}(0.04)\\
\hline
\multirow{6}{*}{ M\ref{M2}}&\multirow{3}{*}{$5$} & NLS & 1.24(0.04) & {\bf 2.45}(0.38) &  - & -  \\
&& cWGAN & 1.32(0.05) & 4.11(0.44) & 0.85(0.24) & 0.37(0.10)  \\
&& WGR & {\bf 1.23}(0.04)  &3.41(0.38) &  {\bf 0.19}(0.10) & {\bf 0.22}(0.04) \\
\cline{2-7}
&\multirow{3}{*}{$100$} & NLS &  1.63(0.05) & 5.37(0.47) &  - & -\\
&& cWGAN & 1.64(0.06) & 5.34(0.46) & 2.20(0.22) & 1.11(0.08)\\
&& WGR & {\bf 1.61}(0.06) & {\bf 5.27}(0.45) & {\bf 2.06}(0.24) & {\bf 0.30}(0.03)\\
\hline
\end{tabular}
\footnotesize
Notes:  $d$  is the dimension of covariate $X$. The corresponding standard errors are given in parentheses. The smallest $L_1$ and $L_2$ error, and the smallest MSEs are in boldface. NLS represents the nonparametric least squares regression, cWGAN represents the conditional WGAN, WGR is our proposed Wasserstein generative regression method.
\end{threeparttable}
\end{table}

\begin{table}
\centering
\caption{\label{Tab-Q}MSE of  the estimated conditional quantile at different quantile levels} 
\begin{threeparttable}
\begin{tabular}{cc| cc |cc}
\hline
& & \multicolumn{2}{c|}{$X\sim N(0,I_{5})$} & \multicolumn{2}{c}{$X\sim N(0,I_{100})$}\\
Model & $\tau$ & cWGAN & WGR & cWGAN & WGR \\
\hline
\multirow{5}{*}{M\ref{M1}} &  0.05 &  1.22(0.23) & {\bf 0.29}(0.07) &3.18(0.71) &  {\bf 1.84}(0.23)   \\
& 0.25 & 1.04(0.25) & {\bf 0.10}(0.01) & 1.83(0.22) & {\bf 1.69}(0.20)\\
& 0.50 & 0.99(0.26) &  {\bf 0.09}(0.02) & 1.75(0.16) & {\bf 1.66}(0.16)  \\
& 0.75 & 1.03(0.24) &  {\bf 0.10}(0.03) & 1.89(0.21) & {\bf 1.88}(0.13)  \\
& 0.95 & 1.34(0.21) & {\bf 0.23}(0.06) & 3.61(0.39) & {\bf 2.41}(0.14) \\
\hline
\multirow{5}{*}{M\ref{M2}} &  0.05 &  1.86(0.21) &{\bf 0.77}(0.09) & 4.99(0.51) &{\bf 3.42}(1.07) \\
& 0.25 & 0.94(0.26) & {\bf 0.31}(0.06)  & 2.63(0.24) & {\bf 2.26}(0.28)\\
& 0.50 & 0.85(0.26)  & {\bf 0.19}(0.04)  & 2.21(0.22)  & {\bf 2.19}(0.22) \\
& 0.75 & 1.00(0.21) & {\bf 0.27}(0.05) & 2.79(0.36)& {\bf 2.57}(0.27) \\
& 0.95 & 2.59(0.52) & {\bf 0.81}(0.15) & 5.41(0.65) & {\bf 3.49}(0.47)  \\
\hline
\end{tabular}
\footnotesize
Notes: The notations are  the same as in Table \ref{tab-sim-r}. The corresponding  standard errors are given in parentheses.
The smallest MSEs are in boldface.
\end{threeparttable}
\end{table}

We repeat the simulations 10 times. For each evaluation criterion, the simulation standard errors are computed and provided in parentheses. Table \ref{tab-sim-r} summaries the average $L_1$ error, $L_2$ error, MSE(mean) and MSE(sd). Table \ref{Tab-Q} reports the average MSE($\tau$) for different $\tau$. In Figure \ref{Fig-2D}, we visualize the quality of the conditional samples and the conditional density estimation given a random  realization of $X$.

It can be seen that, for Models 1 and 2,  the three methods are comparable in terms of  $L_1$ and $L_2$ errors. But for the conditional mean, conditional standard deviation, and  conditional quantile estimation, WGR has smaller MSEs values compared with cWGAN, indicating that WGR works better in distributional matching.

For Models 3-5, Figure \ref{Fig-2D} shows the kernel-smoothing conditional density estimates for a randomly selected value of $X$ based on 5,000 samples generated using the estimated conditional generator. It can be seen that WGR can better estimate the underlying conditional distributions for these models.

\subsection{Real data examples}
We demonstrate the effectiveness of WGR on four datasets: CT slides \citep{Graf2011PositionPI},  UJIndoorLoc \citep{UJIIndoorLoc},  MNIST \citep{mnist2010}, and STL10 \citep{STL10}. The results from the STL10 dataset are given in the Appendix.
Table \ref{DataSet} gives a summary of the dimensions and training sizes of these datasets and
the noise vectors used in the analysis.

\begin{table}
\centering
    \caption{ \label{DataSet}Summary of datasets and noise dimension and size}
    \begin{threeparttable}
    \begin{tabular}{l|rrrr}
    \hline
    & CT slides & UJIndoorLoc & MNIST & STL10\\
    \hline
    Dimension of $X$ & 383 & 520 & 588 & 36477\\
    Dimension of $Y$ & 1 & 6 & 144 & 12675\\
    \hline
    Training size & 40000 & 14948 & 20000 & 10000\\
    Validation size& 3500 & 1100 & 1000 & 1000\\
    Testing size & 10000 & 5000 &10000 & 2000\\
    \hline
    Dimension of  $\eta$ &50&  50 & 100 & 12675\\
    Size of $\eta$ $(J)$ & 200 & 200 & 1 & 1\\
    \hline
    \end{tabular}
    \footnotesize
    Notes: The noise vector $\eta$ is sampled from the multivariate standard normal distribution. $J$ is the size of the noise vector generated in each iteration.
\end{threeparttable}
\end{table}

\subsubsection{
The CT slices dataset}

We evaluate the methods on the CT slices dataset \citep{Graf2011PositionPI} and compare their prediction accuracy. This dataset can be found at the UCI machine learning repository (\url{https://archive.ics.uci.edu/ml/datasets/Relative+location+of+CT+slices+on+axial+axis}).
The dataset contains 53,500 CT images from 74 patients (43 male, 31 female) with different anatomical landmarks annotated on the axial axis of the human body.
Each CT image is represented by two histograms in polar space: one for the bone structures and another for the air inclusions inside the body. The covariate vector consists of 383 variables: 239 for the bone histogram and 145 for the air histogram. The response variable is the relative location of the image on the axial axis, which ranges from $0$ to $180$, where $0$ indicates the top of the head and $180$ indicates the soles of the feet.

The sample size of this dataset is 53,500. We use 40,000 observations for training, 3,500 observations for validation, and 10,000 observations for testing. Both the generator network and the critic network have two hidden layers with widths 128 and 64, respectively. The LeakyReLU activation function is used in both networks. The noise vector $\eta$ is generated from $N(\bm{0}, \bm{I}_{50}).$
The number of the noise vectors $J$ is set to be $200.$

For evaluation, besides the $L_1$ and $L_2$ errors in (\ref{L1L2}), we also compute the average length of the estimated 95$\%$ prediction interval (PI) and the corresponding coverage probability (CP), defined as
\begin{align*}
&\text{PI} = \! \frac{1}{K}\sum_{i=1}^K \{ \hat{F}_{Y|X}^{-1}(0.975|X_i)-  \hat{F}_{Y|X}^{-1}(0.025|X_i) \},\\
&\text{CP}=\! \frac{1}{K}\sum_{i=1}^K  I \{Y_i \in \![ \hat{F}_{Y|X}^{-1}(0.025|X_i),\hat{F}_{Y|X}^{-1}(0.975|X_i) ]\},
\end{align*}
where $\hat{F}_{Y|X}^{-1}(\tau|X_i)$ is the estimated $\tau$-th conditional quantile for a given $X=X_i$, and $K$ is the sample size of the validation or testing set.
The numerical results are summarized in Table \ref{Tab-CT}.

Figure \ref{Fig-CT} shows the prediction intervals for 200 test samples, which are sorted in ascending order according to the value of $Y$. We randomlyselect 200 samples from the test dataset and estimate the conditional prediction interval based on 10,000 observations. The prediction intervals are sorted  in acsending order according to the value of $Y$ and are shown in Figure \ref{Fig-CT}.
In addition, we display the estimated conditional density functions  for 10 test samples in Figure \ref{Fig-CT}. The conditional density function is estimated using kernel smoothing based on 10,000
values calculated from the conditional generator.

\begin{table}
\caption{\label{Tab-CT} Summary statistics for CT test data}
\centering
\begin{tabular}{c|cccc}
\hline
Method & L1 & L2 & PI & CP\\
\hline
NLS & 0.40 & 0.51 & - & - \\
cWGAN & 0.95 & 2.30 & 1.54 & 0.48\\
WGR & 0.36 & 0.48 & 2.80 & 0.96\\ \hline
\end{tabular}
\end{table}

\begin{figure}[H]
\begin{minipage}[t]{0.88\linewidth}
\hspace{3.9cm} Prediction interval \hspace{0.2cm} Estimated conditional density
\end{minipage}
\begin{minipage}[t]{\linewidth}
	\centering
	\rotatebox{90}{ \hspace{0.8cm} {\footnotesize cWGAN}}
	\includegraphics[width=0.20\textheight]{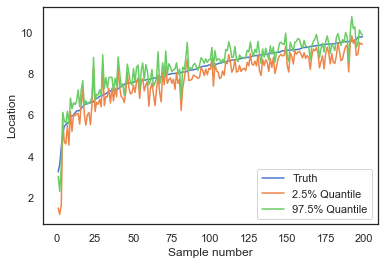} \includegraphics[width=0.20\textheight]{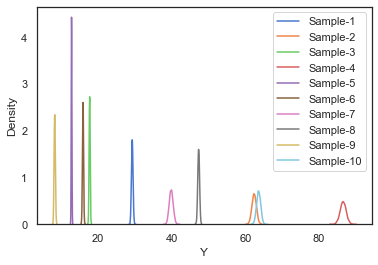}
\end{minipage}
\begin{minipage}[t]{\linewidth}
	\centering
	\rotatebox{90}{ \hspace{1.0cm} {\footnotesize WGR}}
	\includegraphics[width=0.20\textheight]{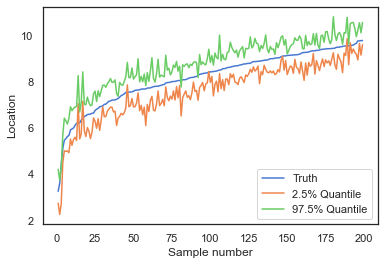} \includegraphics[width=0.20\textheight]{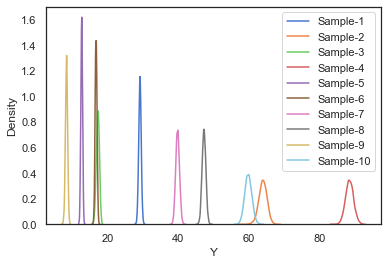}
\end{minipage}
\footnotesize
\caption{\label{Fig-CT}The prediction intervals for the testing dataset and the estimated conditional density functions of 10 randomly selected samples. 
In the left two panels, the blue line represents the truth, the orange and the green lines represent the 2.5$\%$ and the 97.5$\%$ quantiles, respectively. In the right two panels, each color represents an observation from the test dataset.
}
\end{figure}

As shown in Table \ref{Tab-CT}, WGR and NLS have similar performance in terms of $L_1$ and $L_2$ errors, and both methods are superior to cWGAN. Furthermore, the CP of WGR is close to the nominal level of 95$\%$ and much higher than cWGAN. In addition, Figure \ref{Fig-CT} illustrates that the conditional distributions estimated by cWGAN are more peaked than those of our proposed method, and this accounts for why the prediction interval obtained by cWGAN covers fewer points than WGR.

\subsubsection{
The UJIndoorLoc dataset}


We present an analysis of the UJIndoor dataset \citep{UJIIndoorLoc}, a multi-building multi-floor indoor localization database that relies on WLAN/WiFi fingerprinting. The dataset can be downloaded from the UCI machine learning repository (\url{https://archive.ics.uci.edu/ml/datasets/UJIIndoorLoc}).
This  dataset contains 21,048 observations, which are divided into three parts: 14948 for training, 1100 for validation, and 5000 for testing.
Each observation has 529 attributes. The attributes include the WiFi fingerprint, which is composed of 520 intensity values of different detected Wireless Access Points (WAPs), ranging from -104dBm (very weak signal) to 0dBm (strong signal), and 100 for non-detected WAPs. The attributes also include the location information, which consists of six variables: {\it longitude, latitude, floor, building ID, space ID, and relative position}. The first two variables are continuous, while the others are categorical with at least two levels. We apply standardization to the data before training.
Our goal is to predict the location information from the WiFi fingerprint using different machine learning methods and compare their performance.

The neural networks used are two-layer fully connected feedforward networks with 256 and 128 nodes, respectively. The LeakyReLU activation function is used in both the conditional generator and critic networks. The noise vector $\eta$  is generated from $N(\bm{0},\bm{I}_{50})$. And we use $J=200$.

Table \ref{Tab-UJI} presents the analysis  results and Figure \ref{Fig-UJI-pred} shows the prediction intervals for \textit{building ID, space ID}, and \textit{relative position} in  the response vector, based on 200 samples randomly selected from the test dataset.
Compared with cWGAN, the prediction intervals of WGR have a higher coverage probability with a comparable length.

\begin{table}
\center
\caption{\label{Tab-UJI} Analysis results of the UJIndoorLoc testing dataset}
\begin{threeparttable}
\begin{tabular}{c c| ccc ccc }
\hline
Method & & LNG &  LAT & Floor & B-ID & S-ID & RP \\
\hline
\multirow{4}{*}{NLS} & $L_1$ & 0.07& 0.09 & 0.12 & 0.05 & 0.16 & 0.34 \\
& $L_2$ & 0.06 & 0.10 & 0.05 & 0.12 & 0.21 & 0.58 \\
& PI & - & - & - & - & - & -\\
& CP & - & -& - & - & - & -\\
\hline
\multirow{4}{*}{cWGAN} & $L_1$ & 0.14 & 0.17 & 0.23 & 0.12 & 0.29 & 0.23 \\
& $L_2$ & 0.04 & 0.07 & 0.11 & 0.04 & 0.36 & 0.11\\
& PI & 0.24 & 0.26 & 0.31 & 0.22 & 0.61 & 1.06\\
& CP & 0.55 & 0.50 & 0.44 & 0.53 & 0.68 & 0.49\\
\hline
\multirow{4}{*}{WGR} & $L_1$ & 0.08 & 0.10 & 0.18 & 0.05 & 0.20 & 0.41 \\
& $L_2$ & 0.01 & 0.03 & 0.06 & 0.01 & 0.16 & 0.51 \\
& PI & 0.24 & 0.27 & 0.37 & 0.24 & 0.74 & 2.29\\
& CP & 0.75 & 0.71 & 0.60 & 0.89 & 0.84 & 0.58\\
\hline
\end{tabular}
\footnotesize
Notes: WGR is the proposed method. LNG is Longitude, LAT is latitude, B-ID is building ID, S-ID is space ID, RP is relative position.
\end{threeparttable}
\end{table}

\begin{figure}[H]
\centering
\begin{minipage}[t]{0.90\linewidth}
\hspace{1.8cm} \textit{building ID} \hspace{2.2cm} \textit{space ID} \hspace{2.0cm}
\textit{relative position}
\end{minipage}
\\
\begin{minipage}[t]{\linewidth}
	\centering
	\rotatebox{90}{ \hspace{0.8cm}{\footnotesize cWGAN} }
	\includegraphics[width=0.20\textheight]{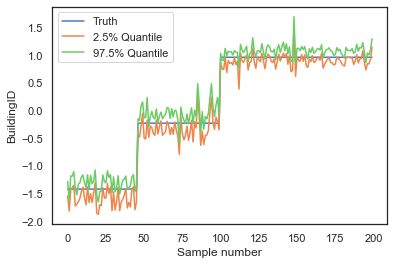}
	\includegraphics[width=0.20\textheight]{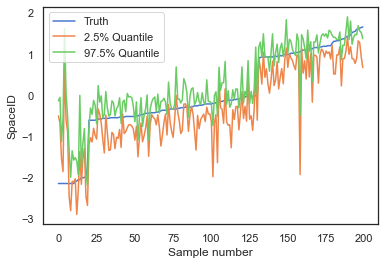}
	\includegraphics[width=0.20\textheight]{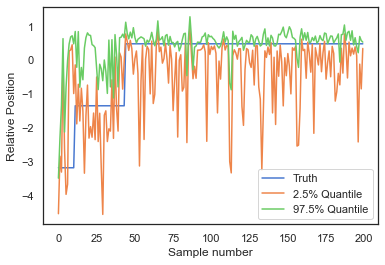}
\end{minipage}
\\
\begin{minipage}[t]{0.90\linewidth}
	\centering
	\rotatebox{90}{ \hspace{1.0cm}{\footnotesize WGR}}
	\includegraphics[width=0.2\textheight]{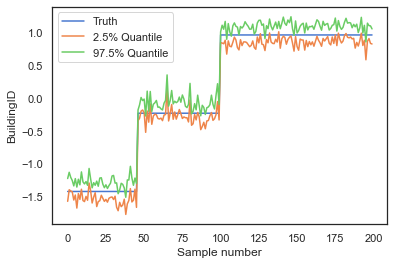}
	\includegraphics[width=0.2\textheight]{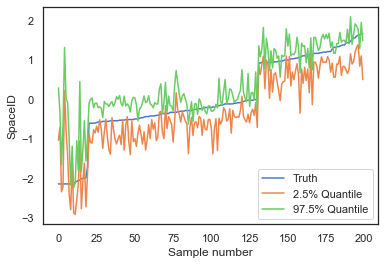}
	\includegraphics[width=0.2\textheight]{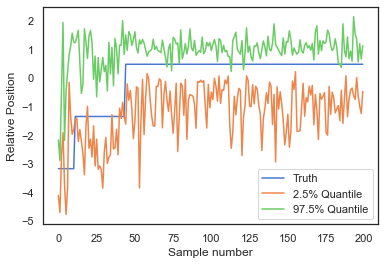}
\end{minipage}
\footnotesize
\caption{\label{Fig-UJI-pred}Prediction intervals for \textit{building ID, space ID}, and
\textit{relative position} in the UJIndoorLoc test dataset.
The blue line represents the truth, the orange line represents the 2.5$\%$ quantile, and the green line represents the 97.5$\%$ quantile.}
\end{figure}

We remark that the prediction intervals in these two data examples do not account for the uncertainty in estimating the conditional generator. To achieve theoretically valid coverage probability, one could use the conformal prediction framework \citep{papadopoulos2002, vovk2005} to adjust the prediction intervals accordingly. However, this is problem is beyond the scope of the current paper  and we leave it for future work.

\subsubsection{
MNIST handwritten digits dataset}\label{MNIST}
We now demonstrate the performance of WGR on a high-dimensional problem, where both $X$ and  $Y$ are high-dimensional. We use the MNIST dataset \citep{mnist2010} of handwritten digits that can be downloaded from \url{http://yann.lecun.com/exdb/mnist/}. The MNIST dataset consists of $28\times 28$ matrices of gray-scale images with values ranging from 0 to 1. Each image has a corresponding label in $\{0,1,\ldots,9\}$.
We apply WGR to the task of reconstructing the central part of an image that is masked. We assume that the masked part is the response $Y\in\mathbb{R}^{14\times 14}$ and the remaining part is the covariate $X$, which has a dimension of $28\times28-14\times14=588$.

To evaluate the quality of the reconstructed images, we randomly sampled two images per digit from the test set and compared the results of three different methods in Figure \ref{Fig-MNIST}.
The figure shows that WGR produces sharper and more faithful images than the other methods, as it preserves more details and reduces artifacts.


\begin{figure}[H]
\begin{minipage}[t]{\linewidth}
\hspace{4.0cm}$X$ \hspace{1.0cm} Truth  \hspace{0.8cm} WGR \hspace{0.6cm} NLS \hspace{0.5cm} cWGAN
\end{minipage}
\\
\centering
\begin{minipage}[t]{0.60\linewidth}
	\includegraphics[width=0.08\textheight]{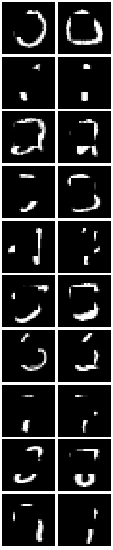}
	\includegraphics[width=0.08\textheight]{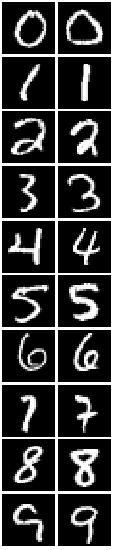}
	\includegraphics[width=0.08\textheight]{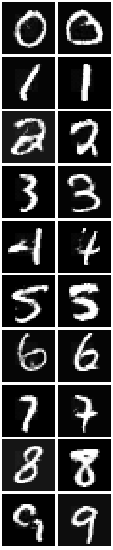}
	\includegraphics[width=0.08\textheight]{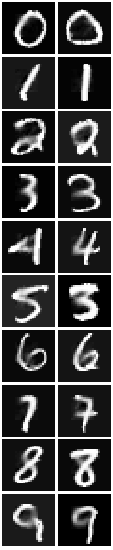}
	\includegraphics[width=0.08\textheight]{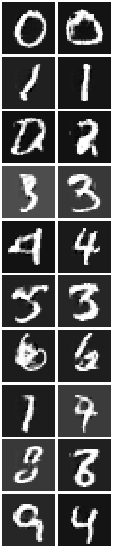}
\end{minipage}
\caption{\label{Fig-MNIST}Reconstructed images in MNIST test data.}
\end{figure}

\subsubsection{MNIST dataset: effects of sample sizes and network architectures.}

We conduct experiments to investigate how the training sample sizes and the network architectures affect the quality of the generated images using the MNIST dataset. We apply WGR with training sample sizes $n=2,000$ and $n=20,000$. The validation sample size is 1,000 and the test sample size is 10,000. For the network architectures, we use a network with 2 CNN layers and a network with 3 fully-connected layers, respectively. Figure \ref{Fig-App-MNIST} displays the reconstructed images with the two different network architectures. WGR with CNN layers is more stable than NLS and cWGAN when the training sample size varies. Moreover, it can be observed that WGR tends to generate images with higher quality.

\begin{figure}[H]
\begin{minipage}[t]{\linewidth}
\hspace{6.0cm} WGR  \hspace{2.8cm} NLS \hspace{2.2cm} cWGAN
\end{minipage}
\begin{minipage}[t]{\linewidth}
\hspace{1.6cm} $X$ \hspace{0.7cm} Truth \hspace{0.9cm} $2000$ \hspace{0.5cm} $20000$\hspace{1.0cm} $2000$ \hspace{0.6cm} $20000$\hspace{1.0cm}
$2000$ \hspace{0.5cm} $20000$
\end{minipage}
\\
\begin{minipage}[t]{\linewidth}
	\centering
	\rotatebox{90}{ \hspace{3.5cm} CNN}
	\includegraphics[width=0.07\textheight]{X-410.png}
	\includegraphics[width=0.07\textheight]{True-410.png}
		\hspace{0.2cm}
	\includegraphics[width=0.07\textheight]{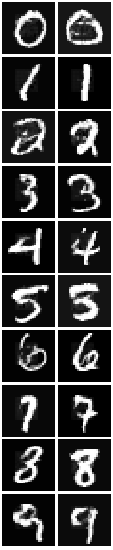}	
	\includegraphics[width=0.07\textheight]{LSWGAN-20000L-DC-410.png}		
		\hspace{0.2cm}
	\includegraphics[width=0.07\textheight]{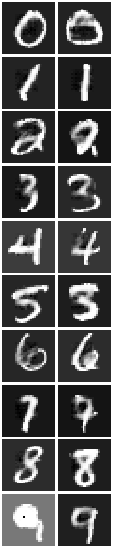}
	\includegraphics[width=0.07\textheight]{LS-20000L-DC-410.png}
		\hspace{0.2cm}
	\includegraphics[width=0.07\textheight]{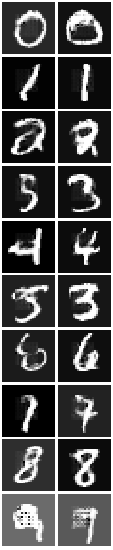}
	\includegraphics[width=0.07\textheight]{WGAN-20000L-DC-410.png}
\end{minipage}

\medskip
\begin{minipage}[t]{\linewidth}
	\centering
	\rotatebox{90}{ \hspace{2.3cm} Fully-Connected}
	\includegraphics[width=0.07\textheight]{X-410.png}
	\includegraphics[width=0.07\textheight]{True-410.png}
		\hspace{0.2cm}
	\includegraphics[width=0.07\textheight]{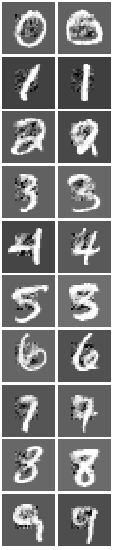}	
	\includegraphics[width=0.07\textheight]{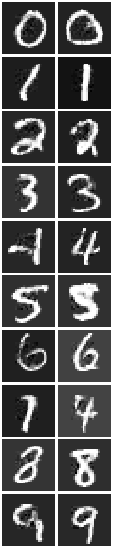}
		\hspace{0.2cm}
	\includegraphics[width=0.07\textheight]{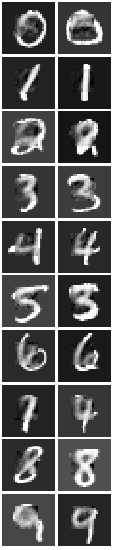}
	\includegraphics[width=0.07\textheight]{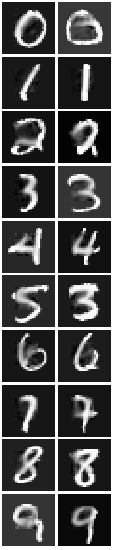}
		\hspace{0.2cm}
	\includegraphics[width=0.07\textheight]{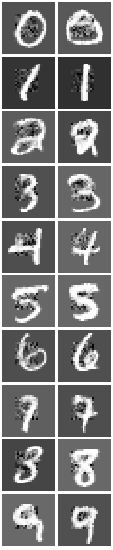}
	\includegraphics[width=0.07\textheight]{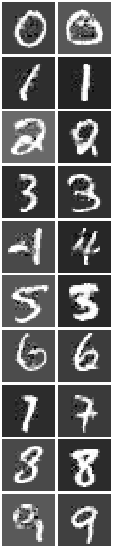}
\end{minipage}
\caption{\label{Fig-App-MNIST}Reconstructed images in MNIST test data.}
\end{figure}

\section{Conclusions}

In this paper, we have proposed a generative regression approach,  Wasserstein generative regression (WGR), for simultaneously estimating a regression function and a conditional generator.
We have provided theoretical support for WGR by establishing its non-asymptotic error bounds
and convergence properties.
Our numerical experiments demonstrate that it works well in various situations from the standard generalized nonparametric regression problems to more complex image reconstruction tasks.

WGR can be viewed as a way of estimating a conditional generator with a data-dependent regularization on the first conditional moment. However, our framework is not limited to this problem and can be adapted to other estimation tasks by choosing different loss functions. For instance, we can estimate the conditional median function or the conditional quantile function by using other losses that are more suitable for these objectives. We can also impose regularization on higher conditional moments or other properties of the conditional distribution, depending on the research question.

Although we have established non-asymptotic error bounds and convergence properties of WGR, our analysis is only a first attempt to deal with a challenging technical problem that involves empirical processes on complex functional spaces and approximation properties of deep neural networks. Further work is needed to better understand the properties of generative regression methods, including the proposed WGR. For instance, it would be interesting to know if the error bounds we derived are optimal or if they can be improved.
WGR is a nonparametric method. For statistical inference and model interpretation, it is desirable to incorporate a semiparametric structure \citep{bkrw1998}
or a  variable selection and dimension reduction component in WGR \citep{cgw2023, hbm2012}.

Generative regression
leverages the power of deep neural networks to model complex and high-dimensional conditional distributions. Unlike traditional regression methods that only output point estimates, generative regression can capture the uncertainty and variability of the data by generating samples from the learned distribution.
This allows for more interpretable results in various statistical applications.
Therefore, we expect  generative learning to be a useful addition to the existing methods for prediction and inference in statistics.

\bibliographystyle{apalike}
\bibliography{Reference}

\newpage
\setcounter{equation}{0}  
\renewcommand{\theequation}{A.\arabic{equation}}
\setcounter{table}{0}
\renewcommand{\thetable}{A.\arabic{table}}
\setcounter{figure}{0}
\renewcommand{\thefigure}{A.\arabic{figure}}
\setcounter{equation}{0}  

\begin{appendices}
\noindent
\textbf{\LARGE Appendices}

In this appendix, we provide detailed proofs of the main theorems and additional numerical experiments, including simulation studies and real data examples.

\section{Proof of Lemma \ref{lem: error_decomp}}\label{appdx:S1}
We first recall the error decomposition.
For the proposed estimator $\hat{g}$, the estimation errors can be decomposed as:
\begin{align*}
& \mathbb{E}_{\mathcal{S}}\left\{\lambda_{\ell} \mathbb{E}\left\|\mathbb{E}_\eta \hat{g}(X, \eta)-\mathbb{E}_\eta g^*(X, \eta)\right\|^2+\lambda_w d_{\mathcal{F}_B^1}\left(P_{X, \hat{g}}, P_{X, Y}\right)\right\} \\
& \leq \lambda_l \mathcal{E}_1+4 \lambda_l \mathcal{E}_2+2 \mathcal{E}_3+2 \lambda_w \mathcal{E}_4+3 \lambda_w \mathcal{E}_5+3 \lambda_w \mathcal{E}_6,
\end{align*}
where $\mathcal{S}=\left\{\left(X_i, Y_i, \eta_i\right)\right\}_{1 \leq i \leq n} \cup\left\{\eta_{i j}\right\}_{1 \leq i \leq n, 1 \leq j \leq J}$ and
\begin{small}
\begin{align}
& \mathcal{E}_1:=\mathbb{E}_{\mathcal{S}}\left\{\mathbb{E}\left\|Y-\mathbb{E}_\eta g^*(X, \eta)\right\|^2+\mathbb{E}\left\|Y-\mathbb{E}_\eta \hat{g}(X, \eta)\right\|^2-\frac{2}{n} \sum_{i=1}^n\left\|Y_i-\mathbb{E}_\eta \hat{g}\left(X_i, \eta\right)\right\|^2\right\},\label{E1} \\
& \mathcal{E}_2:=\mathbb{E}_{\mathcal{S}}\left\{\sup _{g \in \mathcal{G}} \frac{1}{n} \sum_{i=1}^n\left|\left\|Y_i-\frac{1}{J} \sum_{j=1}^J g\left(X_i, \eta_{i j}\right)\right\|^2-\left\|Y_i-\mathbb{E}_\eta g\left(X_i, \eta\right)\right\|^2\right|\right\},\label{E2} \\
& \mathcal{E}_3:=\inf _{g \in \mathcal{G}}\left[\lambda_l \mathbb{E}\left\|\mathbb{E}_\eta g(X, \eta)-\mathbb{E}_\eta g^*(X, \eta)\right\|^2+\lambda_w \sup _{f \in \mathcal{D}}\{\mathbb{E} f(X, g(X, \eta))-\mathbb{E} f(X, Y)\}\right], \label{E3}
\end{align}
\begin{align}
& \mathcal{E}_4:=\sup _{h \in \mathcal{F}_B^1} \inf _{f \in \mathcal{D}}\|h-f\|_{\infty}, \label{E4}\\
& \mathcal{E}_5:=\mathbb{E}_{\mathcal{S}}\left[\sup _{f \in \mathcal{D}}\left\{\mathbb{E} f(X, Y)-\frac{1}{n} \sum_{i=1}^n f\left(X_i, Y_i\right)\right\}\right],\label{E5} \\
& \mathcal{E}_6:=\mathbb{E}_{\mathcal{S}}\left[\sup _{f \in \mathcal{D}, g \in \mathcal{G}}\left\{\mathbb{E} f(X, g(X, \eta))-\frac{1}{n} \sum_{i=1}^n f\left(X_i, g\left(X_i, \eta_i\right)\right)\right\}\right] .\label{E6}
\end{align}
\end{small}

We will prove  Lemma \ref{lem: error_decomp} in several steps.

{\it Step 1.} To rewrite the estimation error.
For the generalized nonparametric regression model
$Y=g^*(X;\eta), $
since the proposed objective function is
\begin{align*}
L(g, f)=\lambda_{\ell} \mathbb{E}\left\|Y-\mathbb{E}_\eta g(X, \eta)\right\|^2+\lambda_w\{\mathbb{E} f(X, g(X, \eta))-\mathbb{E} f(X, Y)\},
\end{align*}
we have $\sup_{f \in \mathcal{F}^1_B} L(g^*,f)=\lambda_{\ell}\mathbb{E}\|Y- \mathbb{E}_{\eta}g^*(X,\eta)\|^2.$
Then, we write  the estimation error
\begin{align}
        &\mathbb{E}_{\mathcal{S}}\left\{\lambda_{\ell} \mathbb{E}\| \mathbb{E}_{\eta}\hat{g}(X,\eta)- \mathbb{E}_{\eta}g^*(X,\eta)\|^2+\lambda_{w}d_{\mathcal{F}^1_B}(P_{X,\hat{g}}, P_{X,Y})\right\}\nonumber\\
        & \quad = \mathbb{E}_{\mathcal{S}}\left\{\sup_{f \in \mathcal{F}^1_B} L(\hat{g},f)- \sup_{f \in \mathcal{F}^1_B} L(g^*,f)\right\}\nonumber\\
        & \quad = \mathbb{E}_{\mathcal{S}}\left\{\sup_{f \in \mathcal{F}^1_B} L(\hat{g},f) - \sup_{f \in \mathcal{D}} L(\hat{g},f)\right\}
                     + \mathbb{E}_{\mathcal{S}}\left\{\sup_{f \in \mathcal{D}} L(\hat{g},f)- \sup_{f \in \mathcal{F}^1_B} L(g^*,f)\right\}.\label{eqn: exc_risk_initial}
\end{align}

{\it Step 2.} To bound the first term in  (\ref{eqn: exc_risk_initial}).
By Lemma 24 in \citet{huang2022error}, 
\begin{align*}
&\mathbb{E}_{\mathcal{S}}\left\{\sup_{f \in \mathcal{F}^1_B} L(\hat{g},f) - \sup_{f \in \mathcal{D}} L(\hat{g},f)\right\}\\
= &\mathbb{E}_{\mathcal{S}}\left\{d_{\mathcal{F}^1_B}(P_{X,\hat{g}}, P_{X,Y})- d_{\mathcal{D}}(P_{X,\hat{g}}, P_{X,Y})\right\} \\
\leq &2\sup_{h\in\mathcal{F}^1_{B}}\inf_{f\in\mathcal{D}}\|h-f\|_{\infty}= 2\mathcal{E}_4,
\end{align*}
where $\mathcal{E}_4$ is defined in (\ref{E4}).

{\it Step 3.} To bound the second term in (\ref{eqn: exc_risk_initial}).
By the definition of $L(g,f)$ and $g^*$,
\begin{align}\label{eq:term-2}
    \nonumber   &\mathbb{E}_{\mathcal{S}}\left\{\sup_{f \in \mathcal{D}} L(\hat{g},f)- \sup_{f \in \mathcal{F}^1_B} L(g^*,f)\right\} \\
       =& \lambda_{\ell}\mathbb{E}_{\mathcal{S}}\left\{\mathbb{E}\Vert Y- \mathbb{E}_\eta \hat{g}(X,\eta)\Vert^2-\mathbb{E}\Vert Y-\mathbb{E}_{\eta}g^*(X,\eta)\Vert^2\right\}
       + \lambda_w \mathbb{E}_{\mathcal{S}}\left[d_{\mathcal{D}}(P_{X,\hat{g}}, P_{X,Y})\right].
\end{align}
We decompose the second term in (\ref{eq:term-2}) as
\begin{align}\label{eq:IMP(Pxg,Pxy)}
 \nonumber   \mathbb{E}_{\mathcal{S}}[ d_{\mathcal{D}}(P_{X,\hat{g}}, P_{X,Y})] &\leq \mathbb{E}_{\mathcal{S}}\left[\sup_{f \in \mathcal{D}} \left\{\mathbb{E}f(X,Y)- \frac{1}{n}\sum_{i=1}^n f(X_i,Y_i)\right\}\right]\\
 \nonumber       &\quad +\mathbb{E}_{\mathcal{S}}\left[\sup_{f \in \mathcal{D}, g \in \mathcal{G}} \left\{\mathbb{E}f(X,g(X,\eta))-\frac{1}{n}\sum_{i=1}^nf(X_i,g(X_i,\eta_i))\right\}\right]\\
  \nonumber   &\quad +\mathbb{E}_{\mathcal{S}}\left[\sup_{f \in \mathcal{D}}\left\{\frac{1}{n}\sum_{i=1}^nf(X_i,\hat{g}(X_i,\eta_i))-\frac{1}{n}\sum_{i=1}^n f(X_i,Y_i)\right\}\right]\\
 &\coloneqq \mathcal{E}_5 + \mathcal{E}_6 + \Delta,
\end{align}
where $\mathcal{E}_5$ and $\mathcal{E}_6$ are defined in (\ref{E5}) and (\ref{E6}).
Note that $\Delta$ in (\ref{eq:IMP(Pxg,Pxy)}) is non-negative as $\mathcal{D}$ is symmetric.
To deal with $\Delta$, we introduce a new estimator $\bar{g}$ defined as
{\small
\begin{equation}\label{eqn: def_barg}
    \bar{g} = \arg\inf_{g \in \mathcal{G}} \left[\lambda_{l}\mathbb{E}\Vert \mathbb{E}_\eta g(X,\eta) -\mathbb{E}_{\eta}g^*(X,\eta)\Vert^2+\lambda_{w}\sup_{f \in \mathcal{D}} \left\{\mathbb{E}f(X,g(X,\eta))-\mathbb{E}f(X,Y)\right\}\right].
\end{equation}}
It then follows from the defintion of $\hat{g}$ that
\begin{align}
\Delta\leq \mathbb{E}_{\mathcal{S}}\left\{\sup_{f \in \mathcal{D}}\hat{L}(\hat{g},f)\right\} \leq \mathbb{E}_{\mathcal{S}}\left\{\sup_{f \in \mathcal{D}}\hat{L}(\bar{g},f)\right\}.
\end{align}
Then, similar to (\ref{eq:IMP(Pxg,Pxy)}), we have
\begin{align}\label{eq:L-bar-g}
 \nonumber    \mathbb{E}_{\mathcal{S}}\left\{\sup_{f \in \mathcal{D}}\hat{L}(\bar{g},f)\right\}\leq &\lambda_{w}\mathcal{E}_5+\lambda_{w}\mathcal{E}_6+\lambda_{w}\sup_{f \in \mathcal{D}} \left\{\mathbb{E}f(X,\bar{g}(X,\eta))-\mathbb{E}f(X,Y)\right\}\\
  &\quad + \lambda_{l}\mathbb{E}_{\mathcal{S}}\left\{\frac{1}{n}\sum_{i=1}^{n}\left\| Y_i-\frac{1}{J}\sum_{j=1}^{J} \bar{g}(X_i,\eta_{ij})\right\|^2\right\}.
\end{align}
To bound the last term in (\ref{eq:L-bar-g}), we subtract $\mathbb{E}\Vert Y-\mathbb{E}_\eta g^*(X,\eta)\Vert^2$ on both sides of (\ref{eq:L-bar-g}). Then,  
\begin{align*}
    &\mathbb{E}_{\mathcal{S}}\left\{\frac{1}{n}\sum_{i=1}^{n}\left\| Y_i-\frac{1}{J}\sum_{j=1}^{J} \bar{g}(X_i,\eta_{ij})\right\|^2\right\}-\mathbb{E}\Vert Y-\mathbb{E}_\eta g^*(X,\eta)\Vert^2\\
     \leq &\mathbb{E}_{\mathcal{S}}\left\{\frac{1}{n}\sum_{i=1}^{n}\left\| Y_i-\frac{1}{J}\sum_{j=1}^{J} \bar{g}(X_i,\eta_{ij})\right\|^2\right\}-\mathbb{E}_{\mathcal{S}}\left\{\frac{1}{n}\sum_{i=1}^{n}\Vert Y_i-\mathbb{E}_\eta \bar{g}(X_i,\eta)\Vert^2\right\}\\
    &+  \mathbb{E}\Vert\mathbb{E}_\eta \bar{g}(X,\eta)-\mathbb{E}_\eta g^*(X,\eta)\Vert^2 \\
     \leq &\mathcal{E}_2+ \mathbb{E}\Vert\mathbb{E}_\eta \bar{g}(X,\eta)-\mathbb{E}_\eta g^*(X,\eta)\Vert^2,
\end{align*}
where $\mathcal{E}_2$ is defined in (\ref{E2}). Then, by the definition of $\bar{g}$,
\begin{align*}
&\mathbb{E}\Vert\mathbb{E}_\eta \bar{g}(X,\eta)-\mathbb{E}_\eta g^*(X,\eta)\Vert^2 \\
\leq &\inf _{g \in \mathcal{G}}\left[\lambda_l \mathbb{E}\left\|\mathbb{E}_\eta g(X, \eta)-\mathbb{E}_\eta g^*(X, \eta)\right\|^2+\lambda_w \sup _{f \in \mathcal{D}}\{\mathbb{E} f(X, g(X, \eta))-\mathbb{E} f(X, Y)\}\right]\coloneqq \mathcal{E}_3,
\end{align*}
where $\mathcal{E}_3$ is  defined in (\ref{E3}). Thus,
\begin{align}\label{eqn: L_n_J_hatg_inequal}
    \mathbb{E}_{\mathcal{S}}\left\{\sup_{f \in \mathcal{D}}\hat{L}(\hat{g},f)\right\}- \lambda_{l}\mathbb{E}\Vert Y-\mathbb{E}_\eta g^*(X,\eta)\Vert^2
    \leq \lambda_{w}\mathcal{E}_5+\lambda_{w}\mathcal{E}_6+ \lambda_{\ell}\mathcal{E}_2+ \mathcal{E}_3.
\end{align}
Further, by the defintion of $\hat{L}(g,f)$,
\begin{align}\label{eq:-L_n}
	\nonumber&-\mathbb{E}_{\mathcal{S}}\left\{\sup_{f \in \mathcal{D}}\hat{L}(\hat{g},f)\right\}+\lambda_{l}\mathbb{E}\Vert Y-\mathbb{E}_\eta g^*(X,\eta)\Vert^2\\
	\nonumber=&-\lambda_l\mathbb{E}_{\mathcal{S}}\left[\sup_{f\in\mathcal{D}}\left\{\frac{1}{n}\sum_{i=1}^n\|Y_i-\frac{1}{J}\sum_{j=1}^J \hat{g}(X_i,\eta_{ij})\|^2-\frac{1}{n} \sum_{i=1}^n\left\|Y_i-\mathbb{E}_\eta \hat{g}\left(X_i, \eta\right)\right\|^2\right\}\right.\\
	\nonumber&-\left.\lambda_w\{\frac{1}{n} \sum_{i=1}^n f\left(X_i, \hat{g}\left(X_i, \eta_i\right)\right)-\frac{1}{n} \sum_{i=1}^n f\left(X_i, Y_i\right)\right\}\\
	\nonumber&+\left.\{\lambda_{l}\mathbb{E}\Vert Y-\mathbb{E}_\eta g^*(X,\eta)\Vert^2-\frac{1}{n} \sum_{i=1}^n\left\|Y_i-\mathbb{E}_\eta \hat{g}\left(X_i, \eta\right)\right\|^2\}\right]\\
	\leq&\lambda_l\mathcal{E}_2 - \lambda_w\Delta + \lambda_l\mathbb{E}_{\mathcal{S}}\{\mathbb{E}\Vert Y-\mathbb{E}_\eta g^*(X,\eta)\Vert^2-\frac{1}{n} \sum_{i=1}^n\left\|Y_i-\mathbb{E}_\eta \hat{g}\left(X_i, \eta\right)\right\|^2\}.
\end{align}
Combining (\ref{eqn: L_n_J_hatg_inequal}) and (\ref{eq:-L_n}), we have
\begin{align*}\
\lambda_w\Delta\leq& \lambda_{w}\mathcal{E}_5+\lambda_{w}\mathcal{E}_6+ 2\lambda_{\ell}\mathcal{E}_2+ \mathcal{E}_3+ \lambda_l\mathbb{E}_{\mathcal{S}}\{\mathbb{E}\Vert Y-\mathbb{E}_\eta g^*(X,\eta)\Vert^2-\frac{1}{n} \sum_{i=1}^n\left\|Y_i-\mathbb{E}_\eta \hat{g}\left(X_i, \eta\right)\right\|^2\}.
\end{align*}
As a result,
\begin{align*}
&\mathbb{E}_{\mathcal{S}}\left\{\sup _{f \in \mathcal{D}} L(\hat{g}, f)-\sup _{f \in \mathcal{F}_B^1} L\left(g^*, f\right)\right\}\\
\leq& \lambda_{\ell} \mathbb{E}_S\left\{\mathbb{E}\left\|Y-\mathbb{E}_\eta \hat{g}(X, \eta)\right\|^2-\mathbb{E}\left\|Y-\mathbb{E}_\eta g^*(X, \eta)\right\|^2\right\}+ 2\lambda_w\mathcal{E}_5+2\lambda_w\mathcal{E}_6+2\lambda_l\mathcal{E}_2+\mathcal{E}_3\\
&+\lambda_l\mathbb{E}_{\mathcal{S}}\{\mathbb{E}\Vert Y-\mathbb{E}_\eta g^*(X,\eta)\Vert^2-\frac{1}{n} \sum_{i=1}^n\left\|Y_i-\mathbb{E}_\eta \hat{g}\left(X_i, \eta\right)\right\|^2\}\\
\leq&2\lambda_w\mathcal{E}_5+2\lambda_w\mathcal{E}_6+2\lambda_l\mathcal{E}_2+\mathcal{E}_3+\mathbb{E}_S\left\{\mathbb{E}\left\|Y-\mathbb{E}_\eta \hat{g}(X, \eta)\right\|^2-\frac{1}{n} \sum_{i=1}^n\left\|Y_i-\mathbb{E}_\eta \hat{g}\left(X_i, \eta\right)\right\|^2\right\}\\
\coloneqq&2\lambda_w\mathcal{E}_5+2\lambda_w\mathcal{E}_6+2\lambda_l\mathcal{E}_2+\mathcal{E}_3+\lambda_{l}\mathcal{E}_1,
\end{align*}
where $\mathcal{E}_1$ is given in (\ref{E1}).

Consequently,
\begin{align*}
& \mathbb{E}_{\mathcal{S}}\left\{\lambda_{\ell} \mathbb{E}\left\|\mathbb{E}_\eta \hat{g}(X, \eta)-\mathbb{E}_\eta g^*(X, \eta)\right\|^2+\lambda_w d_{\mathcal{F}_B^1}\left(P_{X, \hat{g}}, P_{X, Y}\right)\right\} \\
& \leq \lambda_l \mathcal{E}_1+4 \lambda_l \mathcal{E}_2+2 \mathcal{E}_3+2 \lambda_w \mathcal{E}_4+3 \lambda_w \mathcal{E}_5+3 \lambda_w \mathcal{E}_6.
\end{align*}
The proof of Lemma \ref{lem: error_decomp} is complete.

\section{Proofs of the main theorems}\label{appdx:S2}

\subsection*{Proof of Theorem \ref{thm: non-asymptotic_LS}.}
{\color{black}
For any fixed $\lambda_l$ and $\lambda_{w}$ satisfying $0<\lambda_{l},\lambda_{w}<1$, it follows from Lemma \ref{lem: error_decomp} that
\begin{align}
    \nonumber &\mathbb{E}_{\mathcal{S}}\left\{\mathbb{E}\Vert \mathbb{E}_\eta \hat{g}(X,\eta)-\mathbb{E}_\eta g^*(X,\eta)\Vert^2\right\} \nonumber\\
    \nonumber \leq & \frac{1}{\lambda_{l}}\mathbb{E}_{\mathcal{S}}\{\lambda_{l}\mathbb{E}\Vert \mathbb{E}_\eta \hat{g}(X,\eta)-\mathbb{E}_\eta g^*(X,\eta)\Vert^2 +\lambda_{w}d_{\mathcal{F}^1_B}(P_{X,\hat{g}}, P_{X,Y})\} \label{eqn: excess_risk_LS}\\
    \leq & \frac{1}{\lambda_l}(\lambda_{l}\mathcal{E}_1+ 4\lambda_{l}\mathcal{E}_2+ 2\mathcal{E}_3+ 2\lambda_{w}\mathcal{E}_4+ 3\lambda_{w}\mathcal{E}_5+ 3\lambda_{w}\mathcal{E}_6).
\end{align}
The upper bounds of the error terms in (\ref{eqn: excess_risk_LS}) are given in Lemmas \ref{lem: stochastic_error_G_LS2} to \ref{lem: stochastic_D_G} in Section \ref{appdx:S3}.
Thus, under Conditions \ref{C1}-\ref{C4}, we have
\begin{align}
    &\mathbb{E}_{\mathcal{S}}\{\lambda_{l}\mathbb{E}\Vert \mathbb{E}_\eta \hat{g}(X,\eta)-\mathbb{E}_\eta g^*(X,\eta)\Vert^2 +\lambda_{w}d_{\mathcal{F}^1_B}(P_{X,\hat{g}}, P_{X,Y})\}\nonumber\\
    \leq &\lambda_{l}C_1(B_0\vee1)^2q^2 \frac{H_{\mathcal{G}}S_{\mathcal{G}}\log S_{\mathcal{G}}\log n}{n}\nonumber\\
    & + 16\lambda_{l}C_2qB_0^2\sqrt{\frac{H_{\mathcal{G}}S_{\mathcal{G}}\log S_{\mathcal{G}} \log J}{J}}\nonumber\\
    & + 2({ K_{\mathcal{D}}\lambda_{w}}+2B_0q^2\lambda_{l}) 19B_1(\lfloor\beta\rfloor+1)^2(m+d)^{\lfloor\beta\rfloor+1/2+(\beta \vee 1)/2}(\bar{W}\bar{H})^{-2\beta/(m+d)}\nonumber\\
    & + 2\lambda_{w}2^{d+q}(d+q)BW^{-1} \nonumber\\
    & + 3\lambda_{w}C_3 \{B+2^{d+q}(d+q)B\} \sqrt{\frac{H_{\mathcal{D}}S_{\mathcal{D}}\log(S_{\mathcal{D}})\log n}{n}}\nonumber\\
    &  + 3\lambda_{w}C_4B \{B+2^{d+q}(d+q)B\}\sqrt{\frac{(H_{\mathcal{G}}+H_{\mathcal{D}}+1)S_{\mathcal{G},\mathcal{D}}\log(S_{\mathcal{G},\mathcal{D}})\log n}{n}}, \label{eqn: error_weighted_bound}
\end{align}
where $C_1,C_2,C_3,C_4$ are positive constants independent of $n,J,\beta,B_0,B_1,B$,
$S_{\mathcal{G},\mathcal{D}}$ is the size of the network in $\mathcal{NN}(m+d,1,2W_{\mathcal{G}}+2W_{\mathcal{D}}+2d,H_{\mathcal{G}}+H_{\mathcal{D}}+1)$, and
 the Lipschitz constant for the discriminator network class  $K_{\mathcal{D}}\leq 54B2^{d+q}(d+q)^{1/2}W^2${\color{red}}.

To simplify (\ref{eqn: error_weighted_bound}), we let the network parameters $W$, $\bar{W}$ and $\hat{H}$ in the discriminator network class {\bf ND} \ref{ND1} and generator network class {\bf NG} \ref{NG1} be
\begin{align*}
    W=\lceil n^a\rceil, \bar{W}=\lceil n^b/\log^2 n\rceil, \bar{H}=\lceil\log n\rceil,
\end{align*}
where
\begin{align*}
    a &= \frac{\beta}{2\beta+\{3(m+d)\}\vee \{2\beta(d+q+1)\}},\\
    b &= \frac{3(m+d)}{2[2\beta+\{3(m+d)\}\vee \{2\beta(d+q+1)\}]}.
\end{align*}
Finally,
\begin{align*}
    &\mathbb{E}_{\mathcal{S}}\left\{\lambda_{l}\mathbb{E}\Vert \mathbb{E}_\eta \hat{g}(X,\eta)-\mathbb{E}_\eta g^*(X,\eta)\Vert^2 +\lambda_{w}d_{\mathcal{F}^1_B}(P_{X,\hat{g}}, P_{X,Y})\right\}\\
    &\quad \leq C n^{-\frac{\beta}{2\beta+\{3(m+d)\}\vee \{2\beta(d+q+1)\}}}(\log n)^{\frac{2\beta}{m+d}\vee1},
\end{align*}
where $C$ is independent of $(n,J)$. The proof of Theorem \ref{thm: non-asymptotic_LS} is complete. }

\subsection*{Proof of Theorem \ref{thm: non-asymptotic_WGAN}.}
{\color{black}
For any fixed $\lambda_{l}$ and $\lambda_{w}$ satisfying $0<\lambda_{l},\lambda_{w}<1$,
\begin{align*}
    &\mathbb{E}_{\mathcal{S}}\left\{d_{\mathcal{F}^1_B}(P_{X,\hat{g}}, P_{X,Y})\right\}\leq \frac{1}{\lambda_{w}}\mathbb{E}_{\mathcal{S}}\{\lambda_{l}\mathbb{E}\Vert \mathbb{E}_\eta \hat{g}(X,\eta)-\mathbb{E}_\eta g^*(X,\eta)\Vert^2 +\lambda_{w}d_{\mathcal{F}^1_B}(P_{X,\hat{g}}, P_{X,Y})\}.
\end{align*}
Thus, Theorem \ref{thm: non-asymptotic_WGAN} can be proved using the results in the proof of Theorem \ref{thm: non-asymptotic_LS}. We omit the details here.
}

\subsection*{Proof of Theorem \ref{thm: varying_weights}:}
{\color{black}
Following similar arguments as in the proof of Theorem \ref{thm: non-asymptotic_LS},
we know that
the Lipschitz constant $K_{\mathcal{D}}$ in (\ref{eqn: error_weighted_bound}) satisfies $ K_{\mathcal{D}} \leq   54B2^{d+q}(d+q)^{1/2}\lceil n^{2a}\rceil.$
Thus, for any $\lambda_{l},\lambda_{w}>0$ satisfying $\lambda_{l}+\lambda_{w}=1$ and $\lambda_{w}=O(n^{-\frac{1}{d+q+2}}),$
\begin{equation*}
    K_{\mathcal{D}}\lambda_{w}   \leq   C_154B2^{d+q}(d+q)^{1/2},
\end{equation*}
where $C_1$ is a positive constant independent of $(n,J)$. Therefore, for $d,q,m$ satisfying $2\beta(d+q+1)\geq 3(m+d)+\beta$, $J\gtrsim n^{\{3(m+d)+6\beta\}/\{4\beta(d+q+2)\}}$, by (\ref{eqn: error_weighted_bound}), 
\begin{align*}
    &\mathbb{E}_{\mathcal{S}}\left\{\mathbb{E}\Vert \mathbb{E}_\eta \hat{g}(X,\eta)-\mathbb{E}_\eta g^*(X,\eta)\Vert^2\right\}\\
     \leq &\frac{1}{\lambda_{l}}\mathbb{E}_{\mathcal{S}}\{\lambda_{l}\mathbb{E}\Vert \mathbb{E}_\eta \hat{g}(X,\eta)-\mathbb{E}_\eta g^*(X,\eta)\Vert^2 +\lambda_{w}d_{\mathcal{F}^1_B}(P_{X,\hat{g}}, P_{X,Y})\}\\
     \leq& C_3 n^{-\frac{3}{2(d+q+2)}}(\log n)^{\frac{2\beta}{m+d}\vee2},
\end{align*}
and
\begin{align*}
    \mathbb{E}_{\mathcal{S}}\left\{d_{\mathcal{F}^1_B}(P_{X,\hat{g}}, P_{X,Y})\right\}
    &\leq \frac{1}{\lambda_{w}}\mathbb{E}_{\mathcal{S}}\{\lambda_{l}\mathbb{E}\Vert \mathbb{E}_\eta \hat{g}(X,\eta)-\mathbb{E}_\eta g^*(X,\eta)\Vert^2 +\lambda_{w}d_{\mathcal{F}^1_B}(P_{X,\hat{g}}, P_{X,Y})\}\\
    &\leq C_3n^{-\frac{1}{2(d+q+2)}}(\log n)^{\frac{2\beta}{m+d}\vee2} \rightarrow 0,
\end{align*}
as  $n \rightarrow \infty$,
where $C_3$ is a positive constant independent of $(n,J)$.
}

\section{Supporting lemmas}\label{appdx:S3}
In this section, we give some supporting lemmas that are used to  establish the upper bound for the terms in the error decomposition and  are needed in the proof of  Theorems \ref{thm: non-asymptotic_LS} - \ref{thm: varying_weights}.


{\color{black}
We bound $\mathcal{E}_1$ in Lemma \ref{lem: stochastic_error_G_LS1}.
\begin{lemma}\label{lem: stochastic_error_G_LS1}
    Suppose Conditions \ref{C1} and \ref{C4} hold. Then,
         \begin{align*}
            &\mathbb{E}_{\mathcal{S}}\left\{\mathbb{E}\Vert Y- \mathbb{E}_\eta \hat{g}(X,\eta)\Vert^2-\frac{1}{n}\sum_{i=1}^{n}\left\| Y_i-\mathbb{E}_{\eta}\hat{g}(X_i,\eta)\right\|^2\right\}\\
            \leq & CB_0 q \sqrt{\frac{H_{\mathcal{G}}S_{\mathcal{G}}\log S_{\mathcal{G}}\{\log B_0+\log n\}}{n}},
        \end{align*}
    where $C$ is a positive constant,  and  $H_{\mathcal{G}}, S_{\mathcal{G}}$ are the width and depth of the network class $\mathcal{G}$ respectively.
\end{lemma}
}

\begin{proof}
The proof will be done in two steps.

{\it Step 1.} Symmetrization.
Let $\{X'_i,Y'_j\}_{i=1}^n$ be $n$ i.i.d. copies of $(X,Y)$, independent of $\{X_i,Y_i\}_{i=1}^n$. And let $h_{g}(X,Y)\coloneqq \|Y-\mathbb{E}_{\eta}g(X,\eta)\|^2$ for any $g\in\mathcal{G}$. Then, by standard symmetrization technique,
\begin{align*}
    &\mathbb{E}_{\mathcal{S}}\left\{\mathbb{E}\Vert Y- \mathbb{E}_\eta \hat{g}(X,\eta)\Vert^2-\frac{1}{n}\sum_{i=1}^{n}\left\| Y_i-\mathbb{E}_{\eta}\hat{g}(X_i,\eta)\right\|^2\right\}\\
    \leq &\mathbb{E}_{\mathcal{S}}\left[\sup_{g\in\mathcal{G}}\left\{\mathbb{E}\Vert Y- \mathbb{E}_\eta g(X,\eta)\Vert^2-\frac{1}{n}\sum_{i=1}^{n}\left\| Y_i-\mathbb{E}_{\eta}g(X_i,\eta)\right\|^2\right\}\right]\\
    =&\mathbb{E}_{\mathcal{S}}\left[\sup_{g\in\mathcal{G}}\left\{\mathbb{E}h_g(X,Y)-\frac{1}{n}\sum_{i=1}^n h_g(X_i,Y_i)\right\}\right]\\
    \leq& \mathbb{E}_{\{X_i,Y_i,X'_i,Y'_i\}_{i=1}^n}\left[\sup_{g\in\mathcal{G}}\frac{1}{n}\sum_{i=1}^n \{h_g(X'_i,Y'_i)-h_g(X_i,Y_i)\} \right]\\
    = & \mathbb{E}_{\{X_i,Y_i,X'_i,Y'_i,\epsilon_i\}_{i=1}^n}\left[\sup_{g\in\mathcal{G}}\frac{1}{n}\sum_{i=1}^n \epsilon_i\{h_g(X'_i,Y'_i)-h_g(X_i,Y_i)\} \right]\\
    \leq & \mathbb{E}_{\{X_i,Y_i,X'_i,Y'_i,\epsilon_i\}_{i=1}^n}\left[\sup_{g\in\mathcal{G}}\left\{\frac{1}{n}\sum_{i=1}^n \epsilon_i\{h_g(X'_i,Y'_i)-h_g(X_i,Y_i)\}\right\} \right]\\
    \leq &  \mathbb{E}_{\{X'_i,Y'_i,\epsilon_i\}_{i=1}^n}\left\{\sup_{g\in\mathcal{G}} \frac{1}{n}\sum_{i=1}^n \epsilon_i h_g(X'_i,Y'_i) \right\}+\mathbb{E}_{\{X_i,Y_i,\epsilon_i\}_{i=1}^n}\left\{\sup_{g\in\mathcal{G}} \frac{1}{n}\sum_{i=1}^n \epsilon_i h_g(X_i,Y_i)  \right\}\\
    =&2\mathbb{E}_{\{X_i,Y_i,\epsilon_i\}_{i=1}^n}\left\{\sup_{g\in\mathcal{G}} \frac{1}{n}\sum_{i=1}^n \epsilon_i h_g(X_i,Y_i) \right\},
\end{align*}
where $\{\epsilon_j\}_{j=1}^J$ 
    are independent uniform $\{\pm1\}$-valued Rademacher random variables. We use $\mathcal{R}(\mathcal{H})$ to denote the Rademacher complexity for $\mathcal{H}$, that is 
    \begin{equation*}
        \mathcal{R}(\mathcal{H}) \coloneqq \mathbb{E}_{\{X_i,Y_i,\epsilon_i\}_{i=1}^n}\left[\sup_{g\in\mathcal{G}}\frac{1}{n}\sum_{i=1}^n \epsilon_i h_g(X_i,Y_i) \right].
    \end{equation*}

{\it Step 2.} To bound the Rademacher complexity $\mathcal{R}(\mathcal{H})$. 
We define a function class $\mathcal{G}_{|\{X_i\}_{i=1}^n}=\{\mathbb{E}_{\eta}g(X_1,\eta),\ldots,\mathbb{E}_{\eta}g(X_n,\eta):g\in\mathcal{G}\}$. Then,
under Condition \ref{C1},
\begin{align*}
    \mathcal{R}(\mathcal{H})\leq &
    4B_0 \mathbb{E}_{\{X_i,Y_i,\epsilon_i\}_{i=1}^n}\left[\sup_{g\in\mathcal{G}}\frac{1}{n}\sum_{i=1}^n \epsilon_i \{Y_i-\mathbb{E}_{\eta}g(X_i,\eta)\}\right]\\
    \leq & 4B_0 \mathbb{E}_{\{Y_i,\epsilon_i\}_{i=1}^n} \left[\frac{1}{n}\sum_{i=1}^n \epsilon_i Y_i\right]+ 4B_0 \mathbb{E}_{\{X_i,\epsilon_i\}_{i=1}^n} \left[\sup_{g\in\mathcal{G}}\frac{1}{n}\sum_{i=1}^n \epsilon_i\mathbb{E}_{\eta} g(X_i,\eta)\right]\\
    \leq &  32\mathbb{E}_{\{(X_i,Y_i)\}_{i=1}^n} \left[\inf_{0< \delta < B_0/2}\left\{\delta+ \frac{3}{\sqrt{n}}\int_{\delta}^{B_0/2}
        \sqrt{\mathcal{N}(\epsilon, \mathcal{G}_{|\{X_i\}_{i=1}^n}, \|\cdot\|_{\infty})}d\epsilon\right\}\right],
\end{align*}
    where the last inequality is by Lemma 12 in Huang {\it et al.} (2022). By Condition \ref{C4}, there exists a vector $\eta_x\in\Omega_{\eta}$ such that for any $g$, $\tilde{g}\in\mathcal{G}$,
    \begin{align*}
        \|\mathbb{E}_{\eta}g(x,\eta)-\mathbb{E}_{\eta}\tilde{g}(x,\eta)\|_1\leq \|g(x,\eta_x)-\tilde{g}(x,\eta)\|_1.
    \end{align*}
    Thus,
    \begin{align*}
        \mathcal{N}(\epsilon,\mathcal{G}_{|\{X_i\}_{i=1}^n},\|\cdot\|_{\infty})\leq \mathcal{N}(\epsilon,\mathcal{G}_{|\{X_{i},\eta_{xi}\}_{i=1}^n},\|\cdot\|_{\infty}).
    \end{align*}

   Next, we consider to use $\text{Pdim}(\mathcal{G})$, the pseudo-dimension of $\mathcal{G}$, to bound $\mathcal{N}(\epsilon, \mathcal{G}_{|\{(X_i,\eta_i)\}_{i=1}^n}, \|\cdot\|_{\infty})$.
    When $n \geq \text{Pdim}(\mathcal{G})$, its upper bound can be directly obtained by Theorem 12.2 in \cite{anthony1999neural}. 
    When $n < \text{Pdim}(\mathcal{G})$, $\mathcal{G}_{|\{(X_i,\eta_i)\}_{i=1}^n}\subseteq \{x\in \mathbb{R}^n:\|x\|_{\infty}\leq B_2\}$ can be covered by at most $\lceil2B_0/\epsilon\rceil^n$ balls with radius $\epsilon$ in the distance $\|\cdot\|_{\infty}$. Thus, for any $n$ and $\epsilon >0$,
    \begin{equation*}
        \log \mathcal{N}(\epsilon, \mathcal{G}_{|\{(X_i,\eta_i)\}_{i=1}^n}, \|\cdot\|_{\infty}) \leq \text{Pdim}(\mathcal{G})\log\left[\frac{2eB_0n}{\epsilon}\right].
    \end{equation*}
    Since $\mathcal{G}$ is a ReLU network class, according to \cite{bartlett2019nearly}, 
    its pseudo-dimension can be bounded by
    \begin{equation}\label{pdim_G}
        H_{\mathcal{G}}S_{\mathcal{G}}\log(S_{\mathcal{G}}/H_{\mathcal{G}}) \lesssim \text{Pdim}(\mathcal{G}) \lesssim H_{\mathcal{G}}S_{\mathcal{G}}\log S_{\mathcal{G}},
    \end{equation}
    where $H_{\mathcal{G}}, S_{\mathcal{G}}$ are the width and depth of the network class $\mathcal{G}$, respectively.
    As a result, $\mathcal{E}_1$ can be bounded by
    \begin{align*}
        &\mathbb{E}_{\mathcal{S}}\left\{\mathbb{E}\Vert Y- \mathbb{E}_\eta \hat{g}(X,\eta)\Vert^2-\frac{1}{n}\sum_{i=1}^{n}\left\| Y_i-\mathbb{E}_{\eta}\hat{g}(X_i,\eta)\right\|^2\right\}\\
        \leq & 32 \inf_{0< \delta < B_2/2} \left(\delta+ \frac{3\sqrt{H_{\mathcal{G}}S_{\mathcal{G}}\log S_{\mathcal{G}}(\mathcal{G})}}{\sqrt{n}}\int_{\delta}^{B_0/2}
         \sqrt{\log[2eB_2n/\epsilon]}d\epsilon\right)\nonumber\\
        \leq & 32 \inf_{0< \delta < B_0/2}\left(\delta+ \frac{3B_2}{2}\sqrt{\frac{H_{\mathcal{G}}S_{\mathcal{G}}\log S_{\mathcal{G}}(\mathcal{G})\log[2eB_0n/\delta]}{n}}\right)\nonumber\\
        \leq & C B_0\sqrt{\frac{H_{\mathcal{G}}S_{\mathcal{G}}\log S_{\mathcal{G}}(\mathcal{G})[\log n + \log B_0]}{n}},
    \end{align*}
\end{proof}

{\color{black} Next, we intend to bound $\mathcal{E}_2$,  the  stochastic error for the generator. Under Condition \ref{C1}, for any $g\in\mathcal{G}$ and $i\in\{1,\ldots,n\}$,
\begin{align}
        \nonumber &\mathbb{E}_{\mathcal{S}}\left\{\sup_{g\in\mathcal{G}}\frac{1}{n}\sum_{i=1}^n\left|\left\| Y_i-\frac{1}{J}\sum_{j=1}^{J} g(X_i,\eta_{ij})\right\|^2- \Vert Y_i-\mathbb{E}_\eta g(X_i,\eta)\Vert^2\right|\right\}\\
        \leq & 4B_0\mathbb{E}_{\mathcal{S}}\left[\sup_{g \in \mathcal{G}}\left\{\frac{1}{n}\sum_{i=1}^n \left\|\mathbb{E}_\eta g(X_i,\eta)-\frac{1}{J}\sum_{j=1}^{J} g(X_i,\eta_{ij})\right\|_1\right\}\right].\label{eqn: Err_2_7_inequ}
    \end{align}
Then, we can establish the upper bound of (\ref{eqn: Err_2_7_inequ}) in the following lemma.

\begin{lemma}\label{lem: stochastic_error_G_LS2}
    Assume Condition \ref{C1} hold. Then,
    \begin{align*}
            &\mathbb{E}_{\{X_i\}_{i=1}^n \cup \{\eta_{ij}\}_{j=1}^J}\left[\sup_{g \in \mathcal{G}}\left\{\frac{1}{n}\sum_{i=1}^n \left\|\mathbb{E}_\eta g(X_i,\eta)-\frac{1}{J}\sum_{j=1}^{J} g(X_i,\eta_{ij})\right\|_1\right\}\right]\\
            & \quad \leq CqB_0\sqrt{\frac{H_{\mathcal{G}}S_{\mathcal{G}}\log S_{\mathcal{G}}\{\log J+ \log(B_0)\}}{J}},
        \end{align*}
    where $C>0$ is a universal constant.
\end{lemma}}

\begin{proof}{\color{black}
    Recall that $\eta\sim P_{\eta}$ and $\{\eta_{ij},i=1,\ldots,n,j=1,\ldots,J\}$ is $n\times J$ i.i.d. copies of $\eta$, independent of $\{X_i\}_{i=1}^n$. The  proof will be done in three steps.

    {\it Step 1:} To decompose the network class $\mathcal{G}$.  The network class $\mathcal{G}$ can be decomposed into the product $\mathcal{G}_1\otimes \mathcal{G}_2\otimes\ldots\otimes\mathcal{G}_q$, where $\mathcal{G}_k$ satisfies
    that for $g=(g_1,\ldots, g_q)^{\top} \in \mathcal{G}$, $g_k \in \mathcal{G}_k \subseteq \mathbb{R}$ for $k=1,\ldots,q$.
    Note that
    \begin{align}\label{eqn: stochastic_error_Gk}
        &\mathbb{E}_{\{X_i\}_{i=1}^n \cup \{\eta_{ij}\}_{j=1}^J}\left[\sup_{g \in \mathcal{G}}\left\{\frac{1}{n}\sum_{i=1}^n \left\|\mathbb{E}_\eta g(X_i,\eta)-\frac{1}{J}\sum_{j=1}^{J} g(X_i,\eta_{ij})\right\|_1\right\}\right]\nonumber\\
        \leq &\mathbb{E}_{\{X_i\}_{i=1}^n \cup \{\eta_{ij}\}_{j=1}^J}\left[\frac{1}{n}\sum_{i=1}^n\left\{\sup_{g \in \mathcal{G}}\left\|\mathbb{E}_\eta g(X_i,\eta)-\frac{1}{J}\sum_{j=1}^{J} g(X_i,\eta_{ij})\right\|_1\right\}\right]\nonumber\\
         = & \mathbb{E}_{\{X\} \cup \{\eta_{j}\}_{j=1}^J}\left\{\sup_{g \in \mathcal{G}}\left\|\mathbb{E}_\eta g(X,\eta)-\frac{1}{J}\sum_{j=1}^{J} g(X,\eta_j)\right\|_1\right\}\nonumber\\
        \leq & \sum_{k=1}^q \mathbb{E}_{\{X\} \cup \{\eta_j\}_{j=1}^J}\left\{\sup_{g_k \in \mathcal{G}_k}\left|\mathbb{E}_\eta g_k(X,\eta)-\frac{1}{J}\sum_{j=1}^{J} g_k(X,\eta_j)\right|\right\}\nonumber\\
         = &\sum_{k=1}^q \mathbb{E}_{X}\left[\mathbb{E}_{\{\eta_j\}_{j=1}^J}\left\{\sup_{g_k \in \mathcal{G}_k}\left|\mathbb{E}_\eta g_k(X,\eta)-\frac{1}{J}\sum_{j=1}^{J} g_k(X,\eta_j)\right|\right\}\right].
    \end{align}

    {\it Step 2:} To bound $\mathbb{E}_{\{\eta_j\}_{j=1}^J}\{\sup_{g_k \in \mathcal{G}_k}|\mathbb{E}_\eta g_k(X,\eta)-\frac{1}{J}\sum_{j=1}^{J} g_k(X,\eta_j)|\}$.
    Let $\{\eta'_{j}\}_{j=1}^J$ be $J$ i.i.d. copies of $\eta$, independent of $\{\eta_j\}_{j=1}^J$ and $X$.
    It then follows from the standard symmetrization technique that given $X$, for $k=1,\ldots,q$,
    \begin{align*}
        &\mathbb{E}_{\{\eta_j\}_{j=1}^J}\left\{\sup_{g_k \in \mathcal{G}_k}\left|\mathbb{E}_\eta g_k(X,\eta)-\frac{1}{J}\sum_{j=1}^{J} g_k(X,\eta_j)\right|\right\}\\
        \leq & \mathbb{E}_{\{\eta_j\}_{j=1}^J}\left[\sup_{g_k\in \mathcal{G}_k}\mathbb{E}_{\{\eta'_j\}_{j=1}^J}\left|\frac{1}{J}\sum_{j=1}^J\{g_k(X,\eta_j)-g_k(X,\eta'_j)\}\right|\right]\\
         =&\mathbb{E}_{\{\eta_j,\eta'_j,\epsilon_j\}_{j=1}^J}\left[\sup_{g_k\in \mathcal{G}_k}\left|\frac{1}{J}\sum_{j=1}^J\epsilon_j\{g_k(X,\eta_j)-g_k(X,\eta'_j)\}\right|\right]\\
         \leq &2\mathbb{E}_{\{\eta_j,\epsilon_j\}_{j=1}^J}\left\{\sup_{g_k \in \mathcal{G}_k}\left|\frac{1}{J}\sum_{j=1}^J\epsilon_jg_k(X,\eta_j)\right|\right\}\\
         =&2\mathcal{R}(\mathcal{G}_k|X),
    \end{align*}
     where $\mathcal{R}(\mathcal{G}_k|X)$ is the conditional Rademacher complexity of $\mathcal{G}_k$ given $X$. 

    {\it Step 3:} To bound the conditional Rademacher complexity $\mathcal{R}(\mathcal{G}_k|X)$.
    For any $g_k, \tilde{g}_k \in \mathcal{G}_k$,
    \begin{align*}
        \left|\frac{1}{J}\sum_{j=1}^J\epsilon_jg_k(X,\eta_j)\right|-\left|\frac{1}{J}\sum_{j=1}^J\epsilon_j\tilde{g}_k(X,\eta_j)\right|
        &\leq \left|\frac{1}{J}\sum_{j=1}^J\epsilon_j\{g_k(X,\eta_j)-\tilde{g}_k(X,\eta_j)\}\right|\\
        &\leq \max_{1\leq j \leq J} |g_k(X,\eta_j)-\tilde{g}_k(X,\eta_j)|.
    \end{align*}
    We use $e_{X,J}(g_k,\tilde{g}_k)$ to denote the distance between $g_k$ and $\tilde{g}_{k}$ with respect to $\{\eta_j\}_{j=1}^J$ and $X$  defined as
    \begin{align*}
        e_{X,J}(g_k,\tilde{g}_k) = \|g_k-\tilde{g}_k\|_{L^{\infty}(\{(X,\eta_j)\}_{j=1}^J)}.
    \end{align*}
     For any $\delta>0$, we define $\mathcal{G}_{k,\delta}$ to be a covering set of $\mathcal{G}_k$ with radius $\delta$ with respect to $e_{X,J}(\cdot)$ and $\mathcal{N}(\delta, \mathcal{G}_{k}, e_{X,J})$ to be the $\delta$-covering number of $\mathcal{G}_{k}$ with respect to the distance $e_{X,J}(\cdot)$.
    The, by  the triangle inequality and Lemma B.4 in \cite{zjlh2022}, we have
    \begin{align*}
        \mathcal{R}(\mathcal{G}_k|X)&\leq \delta + \mathbb{E}_{\{\eta_j,\epsilon_j\}_{j=1}^J}\left\{\sup_{g_k \in \mathcal{G}_{k,\delta}}\left|\frac{1}{J}\sum_{j=1}^J\epsilon_jg_k(X,\eta_j)\right|\right\}\\
        &\leq   \delta  +  \frac{C_1}{J} \mathbb{E}_{\{\eta_j\}_{j=1}^J}\left[\left\{\log \mathcal{N}(\delta, \mathcal{G}_{k}, e_{X,J})\right\}^{1/2} \left\{\max_{g_k \in \mathcal{G}_{k,\delta}}\sum_{j=1}^J g_k^2(X,\eta_j)\right\}^{1/2}\right]\\
         &\leq   \delta  + \frac{C_1B_0}{\sqrt{J}}\mathbb{E}_{\{\eta_j\}_{j=1}^J}\left[\left\{\log \mathcal{N}(\delta, \mathcal{G}_{k}, e_{X,J})\right\}^{1/2}\right],
    \end{align*}
    where $C_1>0$ is a constant and the third inequality holds since the probability measure of $(X, g(X,\eta))$ is supported on $\Omega \subseteq [-B_0,B_0]^{d+q}\subseteq \mathbb{R}^{d+q}$ for any $g \in \mathcal{G}$.
     It follows from Theorem 12.2 in \cite{anthony1999neural} 
     that for $J\geq\text{Pdim}(\mathcal{G}_k)$,
    \begin{equation*}
       \log \mathcal{N}(\delta, \mathcal{G}_{k}, e_{X,J})=\log \mathcal{N}(\delta, \mathcal{G}_{k|\{(X,\eta_j)\}_{j=1}^J},\|\cdot\|_{\infty})
       \leq \text{Pdim}(\mathcal{G}_k)\log\left\{\frac{2eJB_0}{\delta\text{Pdim}(\mathcal{G}_k)}\right\},
    \end{equation*}
    where $\mathcal{G}_{k|\{(X,\eta_j)\}_{j=1}^J}=\{(g_k(\eta_1,X),\ldots,g_k(X,\eta_j)):g_k \in \mathcal{G}_k\}$ and $\mathcal{N}(\delta, \mathcal{G}_{k|\{(X,\eta_j)\}_{j=1}^J}, \|\cdot\|_{\infty})$ is the $\delta$-covering number of $\mathcal{G}_{k|\{(X,\eta_j)\}_{j=1}^J} \subseteq \mathbb{R}^J$ with respect to the distance $\|\cdot\|_{\infty}$.
    When $J < \text{Pdim}(\mathcal{G}_k)$, $\mathcal{G}_{k|\{(X,\eta_j)\}_{j=1}^J}\subseteq \{x\in \mathbb{R}^J:\|x\|_{\infty}\leq B_0\}$ can be covered by at most $\lceil2B_0/\delta\rceil^J$ balls with radius $\delta$ with respect to the distance $\|\cdot\|_{\infty}$. This means that for any $J$, \begin{equation*}
        \log \mathcal{N}(\delta, \mathcal{G}_{k}, e_{X,J})  \leq \text{Pdim}(\mathcal{G}_k)\log\left(\frac{2eJB_0}{\delta}\right).
    \end{equation*}
    Therefore,
    \begin{equation*}
       \mathcal{R}(\mathcal{G}_k|X) \leq \delta  + \frac{C_1B_0}{\sqrt{J}}\log\left(\frac{2eJB_0}{\delta}\right)\text{Pdim}(\mathcal{G}_k).
    \end{equation*}
    Moreover, \cite{bartlett2019nearly} gives the following result for the pseudo-dimension of a
    ReLU network class:
    \begin{equation*}
        \text{Pdim}(\mathcal{G}_k) \lesssim H_{\mathcal{G}}S_{\mathcal{G}}\log S_{\mathcal{G}},
    \end{equation*}
    where $H_{\mathcal{G}}, S_{\mathcal{G}}$ are the width and depth of the network class $\mathcal{G}$, respectively. By letting $\delta=1/J$, we have
    \begin{align*}
        \mathbb{E}_{\{\eta_j\}_{j=1}^J}\left\{\sup_{g_k \in \mathcal{G}_k}\left|\mathbb{E}_\eta g_k(X,\eta)-\frac{1}{J}\sum_{j=1}^{J} g_k(X,\eta_j)\right|\right\}
        \lesssim B_0\sqrt{\frac{H_{\mathcal{G}}S_{\mathcal{G}}\log S_{\mathcal{G}}\{\log J+ \log(B_0)\}}{J}}.
    \end{align*}
   Thus, we have proved Lemma \ref{lem: stochastic_error_G_LS2}  in view of (\ref{eqn: stochastic_error_Gk}).
}

\end{proof}

{\color{black}
{\color{black}We next can bound  $\mathcal{E}_3$ by combining} Lemmas \ref{lem: bound e3} and \ref{lem: approx_error_G}.
\begin{lemma}\label{lem: bound e3}
    Suppose Condition \ref{C1} holds.
    Then, 
    \begin{equation}\label{eqn: Err_3_8_inequ}
    \mathcal{E}_3 \leq (K_{\mathcal{D}}\lambda_{w}+2B_0q^2\lambda_{l}) \inf_{g \in \mathcal{G}} \sup_{f \in \mathcal{F}^1_{B }}\left\{\mathbb{E}f(X,g(X,\eta))-\mathbb{E}f(X,Y)\right\},
\end{equation}
where $K_{\mathcal{D}}$ is the Lipschitz constant defined in {\bf ND} \ref{ND1}.
\end{lemma}}

\begin{proof}{\color{black}

     Let $Y=(Y^{(1)},\ldots,Y^{(q)})^{\top} \in \mathbb{R}^q$.
   Under Condition \ref{C1}, for any $g =(g^{(1)}, \ldots, g^{(q)})^{\top} \in \mathcal{G}$,
\begin{align*}
    \mathbb{E}\Vert \mathbb{E}_\eta g(X,\eta) -\mathbb{E}_\eta g^*(X,\eta)\Vert^2
    &\leq \mathbb{E}\Vert \mathbb{E}_\eta g(X,\eta) -\mathbb{E}(Y|X)\Vert_1^2\\
    &\leq 2B_0q \mathbb{E}\Vert \mathbb{E}_\eta g(X,\eta) -\mathbb{E}(Y|X)\Vert_1\\
    &\leq 2B_0q^2 \max_{k=1,\ldots, q}\mathbb{E}\left|\mathbb{E}_\eta g^{(k)}(X,\eta) -\mathbb{E}(Y^{(k)}|X)\right|.
\end{align*}
Let $F_{g^{(k)}(X,\eta)|X}(\cdot)$ and $F_{Y^{(k)}|X}(\cdot)$ denote the conditional cumulative distribution of $g^{(k)}(X,\eta)$ and $Y^{(k)}$, respectively.
For any $k\in \{1,\ldots,q\}$ and $g =(g^{(1)}, \ldots, g^{(q)})^{\top} \in \mathcal{G}$, by the definition of Wasserstein distance,
\begin{align*}
        &\mathbb{E}\left|\mathbb{E}_\eta g^{(k)}(X,\eta) -\mathbb{E}(Y^{(k)}|X)\right|\\
        \leq &\mathbb{E}\mathcal{W}_1(P_{g^{(k)}(X,\eta)|X},P_{Y^{(k)}|X})
        \leq \mathcal{W}_1(P_{X,g(X,\eta)},P_{X,Y})
        = \sup_{f \in \mathcal{F}^1_{B}}\{\mathbb{E}f(X,g(X,\eta))-\mathbb{E}f(X,Y)\},
\end{align*}
where the last equality holds for any $g\in\mathcal{G}$, and $\mathcal{F}^1_{B} =\{f: \mathbb{R}^{d+q} \mapsto \mathbb{R}, |f(z_1)-f(z_2)|\leq \|z_1-z_2\|, z_1, z_2\in \mathbb{R}^{d+q}, \|f\|_{\infty}\leq B\}$.
Hence,
for any $g \in \mathcal{G}$,
\begin{align}\label{eqn:Pi_4_L2}
    \mathbb{E}\Vert \mathbb{E}_\eta g(X,\eta) -\mathbb{E}_\eta g^*(X,\eta)\Vert^2
    \leq 2B_0q^2 \sup_{f \in \mathcal{F}^1_{B}}\{\mathbb{E}f(X,g(X,\eta))-\mathbb{E}f(X,Y)\}.
\end{align}
Under Conditions \ref{C1}, {\bf ND} \ref{ND1} and {\bf NG} \ref{NG1},
\begin{align*}
        \sup_{f \in \mathcal{D}} \left\{\mathbb{E}f(X,g(X,\eta))-\mathbb{E}f(X,Y)\right\} \leq K_{\mathcal{D}} \sup_{f \in \mathcal{F}^1_B}\left\{\mathbb{E}f(X,g(X,\eta))-\mathbb{E}f(X,Y)\right\}.
\end{align*}
Consequently, together with (\ref{eqn:Pi_4_L2}), we have
\begin{align*}
     \mathcal{E}_3=&\inf _{g \in \mathcal{G}}\left[\lambda_l \mathbb{E}\left\|\mathbb{E}_\eta g(X, \eta)-\mathbb{E}_\eta g^*(X, \eta)\right\|^2+\lambda_w \sup _{f \in \mathcal{D}}\{\mathbb{E} f(X, g(X, \eta))-\mathbb{E} f(X, Y)\}\right]\\
     \leq &(K_{\mathcal{D}}\lambda_{w}+2B_0q^2\lambda_{l}) \inf_{g \in \mathcal{G}} \sup_{f \in \mathcal{F}^1_{B }}\left\{\mathbb{E}f(X,g(X,\eta))-\mathbb{E}f(X,Y)\right\}.
\end{align*}
}
\end{proof}

{\color{black}
\begin{lemma}\label{lem: approx_error_G}
    For $g^*=(g^*_1,\ldots,g^*_q)^{\top}$, assume that  $g^*_k \in \mathcal{H}^{\beta}([-B_0,B_0]^{m+d},B_1), k=1,\ldots, q$, and the network class $\mathcal{G}$ satisfies {\bf NG} \ref{NG1}. Then, 
    \begin{align*}
        &\inf_{g \in \mathcal{G}} \sup_{f \in \mathcal{F}^1_{B}} \left\{\mathbb{E}f(X,g(X,\eta))-\mathbb{E}f(X,Y)\right\}\\
        \leq &19B_1(\lfloor\beta\rfloor+1)^2(m+d)^{\lfloor\beta\rfloor+1/2+(\beta \vee 1)/2}(\bar{W}\bar{H})^{-2\beta/(m+d)}.
    \end{align*}
\end{lemma}
}

\begin{proof}{\color{black}
By the definition of $g^*$ in (\ref{noise-out1}) and the Lipschitz continuity of $f \in \mathcal{F}^1_{B}$, we have
    \begin{align*}
        \inf_{g \in \mathcal{G}} \sup_{f \in \mathcal{F}^1_{B}} \left\{\mathbb{E}f(X,g(X,\eta))-\mathbb{E}f(X,Y)\right\}
        &= \inf_{g \in \mathcal{G}} \sup_{f \in \mathcal{F}^1_{B}} \left\{\mathbb{E}f(X,g(X,\eta))-\mathbb{E}f(X,g^*(X,\eta))\right\}\\
        &\leq \inf_{g \in \mathcal{G}}  \left\{\mathbb{E}\|g(X,\eta)-g^*(X,\eta)\|\right\}\\
        &\leq \sqrt{m+d}\inf_{g \in \mathcal{G}}\|g-g^*\|_{L^{\infty}([-B_0,B_0]^{m+d})}.
    \end{align*}
    For any $k \in \{1,\ldots,q\}$ and $g^*_k\in \mathcal{H}^{\beta}([-B_0,B_0]^{m+d},B_1)$, 
    we  define a new function
    \begin{align*}
        \tilde{g}^0_k(x) := g^*_k\left(2B_0x-B_0\right)\in \mathcal{H}^{\beta}([0,1]^{m+d},B_1).
    \end{align*}
     Corollary 3.1 in \cite{jiao2021deep} tells that there exists a ReLU network function $\tilde{\phi}_k$, satisfying {\bf NG} \ref{NG1},
    such that for any $k$,
    \begin{align*}
        \|\tilde{\phi}_k-\tilde{g}^0_k\|_{L^{\infty}([0,1]^{m+d})} \leq 19B_1 (\lfloor\beta\rfloor+1)^2(m+d)^{\lfloor\beta\rfloor+(\beta \vee 1)/2}(WH)^{-2\beta/(m+d)}.
    \end{align*}
    Then, we define
    \begin{align*}
        \phi(x)  := \left(\tilde{\phi}_1\left(\frac{1}{2B_0}x-\frac{1}{2}\right),\ldots,\tilde{\phi}_q\left(\frac{1}{2B_0}x-\frac{1}{2}\right)\right)^{\top}.
    \end{align*}
    Hence, by Remark 14 in \cite{nakada2020adaptive}, $\phi \in \mathcal{G}$ with width $W_{\mathcal{G}}$ and $H_{\mathcal{G}}$ satisfying {\bf NG} \ref{NG1}.
  Then,
    \begin{align*}
        \|\phi-g^*\|_{L^{\infty}([-B_0,B_0]^{m+d})}
        &\leq \max_{1\leq k \leq q} \left\|\tilde{\phi}_k\left(\frac{1}{2B_0}x-\frac{1}{2}\right)-\tilde{g}^0_k\left(\frac{1}{2B_0}x-\frac{1}{2}\right)\right\|_{L^{\infty}([-B_0,B_0]^{m+d})} \\
        & = \max_{1\leq k \leq q} \|\tilde{\phi}_k-\tilde{g}^0_k\|_{L^{\infty}([0,1]^{m+d})} \\
        & \leq 19B_1 (\lfloor\beta\rfloor+1)^2(m+d)^{\lfloor\beta\rfloor+(\beta \vee 1)/2}(WH)^{-2\beta/(m+d)}.
    \end{align*}
    Therefore,
    \begin{align*}
        &\inf_{g \in \mathcal{G}} \sup_{f \in \mathcal{F}^1_{B}} \left\{\mathbb{E}f(X,g(X,\eta))-\mathbb{E}f(X,Y)\right\}
        \leq \sqrt{m+d} \|\phi-g^*\|_{L^{\infty}([-B_0,B_0]^{m+d})} \\
        \leq &19B_1(\lfloor\beta\rfloor+1)^2(m+d)^{\lfloor\beta\rfloor+1/2+(\beta \vee 1)/2}(WH)^{-2\beta/(m+d)}.
    \end{align*}}
\end{proof}
{\color{black}
Next, {\color{black}we can bound $\mathcal{E}_4$ by Lemmas \ref{lem: approx_error_D}.}

\begin{lemma}\label{lem: approx_error_D}
    Suppose Condition \ref{C1} holds.
    Then, there exists a network function $\tilde{f}$ satisfying the width and depth in {\bf ND} \ref{ND1} such that on the domain $[-B_0,B_0]^{d+q}$,
    \begin{equation*}
        \sup_{z \in [-B_0,B_0]^{d+q}}|\phi(z)-f(z)|\leq  2^{d+q}(d+q)BW^{-1}.
    \end{equation*}
\end{lemma}}
\begin{proof}
The proof is carried out in two steps. We begin with approximating $f$ by a sum-product
combination of univariate piecewise-linear functions.  And then  we use a neural network to approximate it.

    \emph{Step 1:} To approximate $f$ by a piecewise-linear function. Let $N$ be a positive integer.
    First, we find a grid of $(N+1)^{d+q}$ functions  $\phi_{\mathbf{m}}(\mathbf{x})$ such that
    \begin{align*}
        \sum_{\mathbf{m}\in \{0,1,\cdots,N\}^{d+q}} \phi_{\mathbf{m}}(\mathbf{x}) \equiv 1, \quad \mathbf{x} \in [-B_0,B_0]^{d+q}.
    \end{align*}
    To address this issue, for any $\mathbf{x}=(x_1,\cdots,x_{d+q})^{\top}$, we define 
    \begin{align}\label{eq:phi_m}
        \phi_{\mathbf{m}}(\mathbf{x}) = \prod_{k=1}^{d+q} \psi\left(3N\left(x_k-\frac{m_k}{N}\right)\right),
    \end{align}
    where $\mathbf{m}=(m_1,\cdots,m_{d+q}) \in \{0,1,\cdots,N\}^{d+q}$ and
    \begin{align}\label{eq:psi_def}
        \psi(x) =
        \begin{cases}
         0,       &   |x|>2,\\
         2-|x|,   & 1\geq|x|\leq 2,\\
         1,       &  |x|<1.
        \end{cases}
    \end{align}
    Note that for arbitrary $\mathbf{m}$, $\text{supp}\,\phi_{\mathbf{m}} \subset\{\mathbf{x}:|x_k-m_k/N|\leq1/N\}$.
    Then, we construct the function $f_1$ to approximate $f$  defined as
    \begin{align*}
        f_1(\mathbf{x}) = \sum_{\mathbf{m}\in \{0,1,\cdots,N\}^{d+q}} a_{\mathbf{m}} \phi_{\mathbf{m}}(\mathbf{x}),
    \end{align*}
    where $a_{\mathbf{m}}= f(\mathbf{m}/N)$.
    Then, by the properties of $\phi_{\mathbf{m}}$ and the Lipschitz property of $f$, the approximation error can be bounded by
    \begin{align*}
        \|f-f_1\|_{L^{\infty}([-B_0,B_0]^{d+q})} &= \|\sum_{\mathbf{m}\in \{0,1,\cdots,N\}^{d+q}} \phi_{\mathbf{m}}(\mathbf{x})(f(\mathbf{x})-a_{\mathbf{m}})\|_{L^{\infty}([-B_0,B_0]^{d+q})}\\
        & \leq \sum_{\mathbf{m}:|x_k-m_k/N|\leq1/N, \,\forall k} |f(\mathbf{x})-a_{\mathbf{m}}|\\
        & \leq 2^{d+q} \max_{\mathbf{m}:|x_k-m_k/N|\leq1/N, \,\forall k} |f(\mathbf{x})-a_{\mathbf{m}}|\\
        & \leq 2^{d+q} \max_{\mathbf{m}:|x_k-m_k/N|\leq1/N, \,\forall k} \|\mathbf{x}-\mathbf{m}/N\|_2\\
        & \leq \frac{2^{d+q} \sqrt{d+q}}{N}.
    \end{align*}

    \emph{Step 2:} To approximate $f_1$ by a neural network.
    First, we construct a ReLU activated $\psi$  satisfying (\ref{eq:psi_def}):
    \begin{align}\label{eq: psi_network}
        \psi(x) = \sigma(\sigma(x+2)-\sigma(x+1)+\sigma(2-x)-\sigma(1-x)-1).
    \end{align}
    According to \cite{hon2022simultaneous}, for an almost everywhere differentiable function $g:[0,1]^{d+q}\rightarrow \mathbb{R}$, we can define its sobolev norm as
 $$\|g\|_{\mathcal{W}^{1,\infty}([0,1]^{d+q})}=\max_{\mathbf{n}:|\mathbf{n}|_1= 1}\sup_{\mathbf{x} \in [0,1]^{d+q}}|D^{\mathbf{n}}g(\mathbf{x})|,$$
 where $\mathbf{n}=(n_1,\cdots,n_{d+q})\in \{0,1\}^{d+q}$ and $D^{\mathbf{n}}g$ is the respective weak derivative.
    Then, we define
    \begin{align}\label{eq:tilde_f_def}
       \nonumber \tilde{f}_{\mathbf{m}}(\mathbf{x})=& \varphi\left(\psi(3Nx_1-3m_1),\cdots,\psi(3Nx_{d+q}-3m_{d+q})\right),\\
       \tilde{f}(\mathbf{x})=&\sum_{\mathbf{m}\in \{0,1,\cdots,N\}^{d+q}} a_{\mathbf{m}} \tilde{f}_{\mathbf{m}}(\mathbf{x}),
    \end{align}
    where $\varphi$ is a ReLU neural network with width $9(N_1+1)+d+q-1$ and depth $14(d+q)(d+q-1)L$, and $\|\varphi\|_{\mathcal{W}^{1,\infty}([0,1]^{d+q})} \leq 18$.
    By Lemma \ref{lem: lip-D}, we can implement $\tilde{f}$ by a neural network satisfying {\bf ND} \ref{ND1}.
    Then, the approximation error of $f_1$ can be bounded by
   \begin{align*}
        \mathop{\sup}_{\mathbf{x} \in [-B_0,B_0]^{d+q}}|\tilde{f}(\mathbf{x})-f_1(\mathbf{x})|
        & = \mathop{\sup}_{\mathbf{x} \in [-B_0,B_0]^{d+q}}|\sum_{\mathbf{m}}a_{\mathbf{m}}\{ \tilde{f}_{\mathbf{m}}(\mathbf{x})- \phi_{\mathbf{m}}(\mathbf{x})\}|\\
        &\leq B(N+1)^{d+q}\mathop{\sup}_{\mathbf{x} \in [-B_0,B_0]^{d+q}}\max_{\mathbf{m}} | \tilde{f}_{\mathbf{m}}(\mathbf{x})- \phi_{\mathbf{m}}(\mathbf{x})|\\
        &\leq B(N+1)^{d+q}\|\varphi-t_1t_2\cdots t_{d+q}\|_{\mathcal{W}^{1,\infty}([0,1]^{d+q})}\\
        &\leq B(N+1)^{d+q} 10(d+q-1)(N+1)^{-7(d+q)}\\
        &= 10B(d+q-1)(N+1)^{-6(d+q)},
    \end{align*}
    where the third inequality is by Lemma 3.5 in \cite{hon2022simultaneous}. Therefore,
    \begin{align*}
        \mathop{\sup}_{\mathbf{x} \in [-B_0,B_0]^{d+q}}|\tilde{f}(\mathbf{x})-f(\mathbf{x})| &\leq \mathop{\sup}_{\mathbf{x} \in [-B_0,B_0]^{d+q}}|\tilde{f}(\mathbf{x})-f_1(\mathbf{x})|+ \|f-f_1\|_{L^{\infty}([-B_0,B_0]^{d+q})}\\
        &\leq 2^{d+q}(d+q)^{1/2}N^{-1}+10B(d+q-1)(N+1)^{-6(d+q)}\\
        &\leq 2^{d+q}(d+q)BN^{-1}.
    \end{align*}
\end{proof}

\begin{lemma}\label{lem: lip-D}
    Suppose Condition \ref{C1} holds. Then, for the network function $\tilde{f}$ definied in (\ref{eq:tilde_f_def}), its Lipschitz constant satisfies $Lip(\tilde{f})\leq 54B2^{d+q}(d+q)^{1/2}N^2$.
\end{lemma}

\begin{proof}
{\color{black}
To bound the Lipschitz constant of $\tilde{f}$, we need to find a positive constant $C$ such that $$\max_{\mathbf{n}:|\mathbf{n}|_1= 1}\sup_{\mathbf{x} \in [0,1]^{d+q}}|D^{\mathbf{n}}\tilde{f}(\mathbf{x})|\leq C,$$
as it is a sufficient condition for $\text{Lip}(\tilde{f})\leq \sqrt{d+q}C$. By the definitions of $\tilde{f}$ and $\tilde{f}_{\mathbf{m}}$, we can decompose the first-order partial derivative of $\tilde{f}$ over $x_{k_0}$ as
\begin{align*}
        \left|\frac{\partial}{\partial x_{k_0}} \tilde{f}(\mathbf{x})\right|
        =& \left|\sum_{\mathbf{m}\in \{0,1,\cdots,N\}^{d+q}} a_{\mathbf{m}} \frac{\partial}{\partial x_{k_0}}\tilde{f}_{\mathbf{m}}(\mathbf{x})\right|\\
        \leq &\left|\sum_{\mathbf{m}: \forall k, \atop x_{k}\in \{t:|t-m_{k}/N|\leq 1/N\}} a_{\mathbf{m}} \frac{\partial}{\partial x_{k_0}}\tilde{f}_{\mathbf{m}}(\mathbf{x})\right|+\left|\sum_{\mathbf{m}: \exists k_1, \atop x_{k_1}\notin \{t:|t-m_{k_1}/N|\leq 1/N\}} a_{\mathbf{m}} \frac{\partial}{\partial x_{k_0}}\tilde{f}_{\mathbf{m}}(\mathbf{x})\right|\\
        \leq &\left|\sum_{\mathbf{m}: \forall k, \atop x_{k}\in \{t:|t-m_{k}/N|\leq 1/N\}} a_{\mathbf{m}} \frac{\partial}{\partial x_{k_0}}\tilde{f}_{\mathbf{m}}(\mathbf{x})\right|+\left|\sum_{\mathbf{m}:\exists k_1\neq k_0, \atop x_{k_1}\notin \{t:|t-m_{k_1}/N|\leq 1/N\}} a_{\mathbf{m}} \frac{\partial}{\partial x_{k_0}}\tilde{f}_{\mathbf{m}}(\mathbf{x})\right|\\
        &\quad + \left|\sum_{\mathbf{m}:x_{k_0}\notin\{t:|t-m_{k_0}/N|\leq 1/N\}, \atop \forall k_1\neq k_0, x_{k_1}\in \{t:|t-m_{k_1}/N|\leq 1/N\}} a_{\mathbf{m}} \frac{\partial}{\partial x_{k_0}}\tilde{f}_{\mathbf{m}}(\mathbf{x})\right|\\
        \coloneqq & E_1+E_2+E_3.
    \end{align*}
    We shall next bound $E_1$, $E_2$ and $E_3$, respectively. For $E_1$, by the definition of $\tilde{f}_{\mathbf{m}}$ and the properties of $\varphi$,
    \begin{align*}
        E_1&\leq 2^{d+q}B\max_{\mathbf{m}: \forall k, \atop x_{k}\in \{t:|t-m_{k}/N|\leq 1/N\}} \left|\frac{\partial}{\partial x_{k_0}}\tilde{f}_{\mathbf{m}}(\mathbf{x})\right|
        \leq   2^{d+q}B\cdot3N\|\varphi\|_{\mathcal{W}^{1,\infty}([0,1]^{d+q})}
        \leq   54B2^{d+q}N,
    \end{align*}
    where the second inequality holds according to the chain derivative rule.

    Next, we consider $E_2$.
    By the definition of $\phi_{\mathbf{m}}$, for any $\mathbf{m} \in \{\mathbf{m}:\exists k_1\neq k_0, x_{k_1}\notin \{t:|t-m_{k_1}/N|\leq 1/N\}\}$, we have $\partial \phi_{\mathbf{m}}/\partial x_{k_0}=0$.
    Thus, similar to $E_1$, we have
    \begin{align*}
        E_2&\leq B \sum_{\mathbf{m}:\exists k_1\neq k_0, x_{k_1}\notin \{t:|t-m_{k_1}/N|\leq 1/N\}} \left|\frac{\partial}{\partial x_{k_0}}\tilde{f}_{\mathbf{m}}(\mathbf{x})\right|\\
        &\leq B (N+1)^{d+q}\max_{\mathbf{m}:\exists k_1\neq k_0, x_{k_1}\notin \{t:|t-m_{k_1}/N|\leq 1/N\}} \left|\frac{\partial}{\partial x_{k_0}}\tilde{f}_{\mathbf{m}}(\mathbf{x})\right|\\
        &= B (N+1)^{d+q}\cdot \max_{\mathbf{m}:\exists k_1\neq k_0, x_{k_1}\notin \{t:|t-m_{k_1}/N|\leq 1/N\}}\left|\frac{\partial}{\partial x_{k_0}}\tilde{f}_{\mathbf{m}}(\mathbf{x})-\frac{\partial }{\partial x_{k_0}}\phi_{\mathbf{m}}(\mathbf{x})\right|\\
        &\leq B (N+1)^{d+q}\cdot 3N \|\varphi-t_1t_2\cdots t_{d+q}\|_{\mathcal{W}^{1,\infty}([0,1]^{d+q})}\\
        & \leq 30B(d+q-1)(N+1)^{-5(d+q)}.
    \end{align*}
    Third, $E_3$ satisfies
    \begin{align*}
        E_3 &\leq B(N-1)2^{d+q-1}\max_{\mathbf{m}:x_{k_0}\notin\{t:|t-m_{k_0}/N|\leq 1/N\}, \forall k_1\neq k_0, x_{k_1}\in \{t:|t-m_{k_1}/N|\leq 1/N\}} \left| \frac{\partial}{\partial x_{k_0}}\tilde{f}_{\mathbf{m}}(\mathbf{x})\right|\\
        &\leq B(N-1)2^{d+q-1}\cdot3N\|\varphi\|_{\mathcal{W}^{1,\infty}([0,1]^{d+q})} \\
        &\leq 542^{d+q-1}BN^2.
    \end{align*}
    As a result, we can bound the sobolev norm of $\tilde{f}$ by
    \begin{align*}
        \max_{\mathbf{n}:|\mathbf{n}|_1= 1}\sup_{\mathbf{x} \in [-B_0,B_0]^{d+q}}|D^{\mathbf{n}}\tilde{f}(\mathbf{x})| \leq 54B2^{d+q}N^2.
    \end{align*}
    Therefore, on the domain $[-B_0,B_0]^{d+q}$, $\text{Lip}(\tilde{f})\leq 54B2^{d+q}(d+q)^{1/2}N^2$.
}
\end{proof}
{\color{black}
In the following, we shall establish the bound of $\mathcal{E}_5$ in Lemma \ref{lem: stochastic_error_D}.
\begin{lemma}\label{lem: stochastic_error_D}
    Suppose $\sup_{f \in \mathcal{D}}\|f\|_{\infty}\leq B_2$ for a constant $B_2$ and the pseudo-dimension of $\mathcal{D}$ satisfies $\text{Pdim}(\mathcal{D})< \infty$, then
    \begin{align*}
        &\mathbb{E}_{\{(X_i,Y_i)\}_{i=1}^n}\left[\sup_{f \in \mathcal{D}} \left\{\mathbb{E}f(X,Y)- \frac{1}{n}\sum_{i=1}^n f(X_i,Y_i)\right\}\right]  \leq CB_2 \sqrt{\frac{H_{\mathcal{D}}S_{\mathcal{D}}\log(S_{\mathcal{D}})[\log n + \log B_2]}{n}},
    \end{align*}
    where $C>0$ is a universal constant, $H_{\mathcal{D}}, S_{\mathcal{D}}$ are the width and depth of the network class $\mathcal{D}$, respectively.
\end{lemma}}

\begin{proof}{\color{black}
    Define $\mathcal{D}_{|\{(X_i,Y_i)\}_{i=1}^n}=\{(f(X_1,Y_1),\ldots,f(X_n,Y_n)):f \in \mathcal{D}\}$, where $\{(X_i,Y_i)\}_{i=1}^n$ are $n$ i.i.d. samples from $P_{X,Y}$. Let $\mathcal{N}(\epsilon, \mathcal{D}_{|\{(X_i,Y_i)\}_{i=1}^n}, \|\cdot\|_{\infty})$ be the $\epsilon$-covering number of $\mathcal{D}_{|\{(X_i,Y_i)\}_{i=1}^n} \subseteq \mathbb{R}^n$ with respect to the distance $\|\cdot\|_{\infty}$.

    It follows from Lemma 12 in \cite{huang2022error} 
     that  
    \begin{align*}
        &\mathbb{E}_{\{(X_i,Y_i)\}_{i=1}^n}\left[\sup_{f \in \mathcal{D}} \left\{\mathbb{E}f(X,Y)- \frac{1}{n}\sum_{i=1}^n f(X_i,Y_i)\right\}\right] \\
        &\quad \leq 8\mathbb{E}_{\{(X_i,Y_i)\}_{i=1}^n} \left[\inf_{0< \delta < B_2/2}\left\{\delta+ \frac{3}{\sqrt{n}}\int_{\delta}^{B_2/2}
        \sqrt{\mathcal{N}(\epsilon, \mathcal{D}_{|\{(X_i,Y_i)\}_{i=1}^n}, \|\cdot\|_{\infty})}d\epsilon\right\}\right].
    \end{align*}
 It   follows from  similar arguments as in the proof of Lemma \ref{lem: stochastic_error_G_LS1} that
 $\mathcal{E}_5$ can be bounded by
    \begin{align}
        &\mathbb{E}_{\{(X_i,Y_i)\}_{i=1}^n}\left[\sup_{f \in \mathcal{D}} \left\{\mathbb{E}f(X,Y)- \frac{1}{n}\sum_{i=1}^n f(X_i,Y_i)\right\}\right] \nonumber\\
        \leq & C B_2\sqrt{\frac{H_{\mathcal{D}}S_{\mathcal{D}}\log S_{\mathcal{D}}(\mathcal{D})[\log n + \log B_2]}{n}},\label{stochastic_error_pdimD}
    \end{align}
    where $C$ is a universal positive constant.
    }
\end{proof}

We will bound $\mathcal{E}_6$ in Lemma \ref{lem: stochastic_D_G}.
\begin{lemma}\label{lem: stochastic_D_G}
     Suppose $\sup_{f \in \mathcal{D}}\|f\|_{\infty}\leq B_2$ and the pseudo-dimension of $\mathcal{G}$ and $\mathcal{D}$ satisfies $\text{Pdim}(\mathcal{G})< \infty, \text{Pdim}(\mathcal{D})< \infty$. Let $(W_{\mathcal{G}}, H_{\mathcal{G}})$ be the width and depth of the network in $\mathcal{G}$ and $(W_{\mathcal{D}}, H_{\mathcal{D}})$ be the width and depth of the network in $\mathcal{D}$.
   Let $S_{\mathcal{G},\mathcal{D}}$ denote the size of the network in $\mathcal{NN}(m+d,1,2W_{\mathcal{G}}+2W_{\mathcal{D}}+2d,H_{\mathcal{G}}+H_{\mathcal{D}}+1)$.
    Then,
    \begin{align*}
        &\mathbb{E}_{\{(X_i,Y_i)\}_{i=1}^n}\left[\sup_{f \in \mathcal{D}, g \in \mathcal{G}} \left\{\mathbb{E}f(X,g(X,\eta))-\frac{1}{n}\sum_{i=1}^nf(X_i,g(X_i,\eta_i))\right\}\right]\\
        & \quad \leq C B_2 \sqrt{\frac{(H_{\mathcal{G}}+H_{\mathcal{D}}+1)S_{\mathcal{G},\mathcal{D}}\log(S_{\mathcal{G},\mathcal{D}})[\log n + \log B_2]}{n}},
    \end{align*}
    where $C>0$ is a universal constant.
\end{lemma}

\begin{proof}
{\color{black}
Let $h(X,\eta)=f(X,g(X,\eta))$ and $\mathcal{F}=\mathcal{NN}(m+d,1,2W_{\mathcal{G}}+2W_{\mathcal{D}}+2d,H_{\mathcal{G}}+H_{\mathcal{D}}+1)+1$. Then, $\sup_{h \in \mathcal{F}}\|h\|_{\infty}\leq B_2$. According to Remarks 13-14 in \cite{nakada2020adaptive}, we can show by a similar argument to the proof of Lemma \ref{lem: stochastic_error_D} that
\begin{align*}
        &\mathbb{E}_{\{(X_i,Y_i)\}_{i=1}^n}\left[\sup_{f \in \mathcal{D}, g \in \mathcal{G}} \left\{\mathbb{E}f(X,g(X,\eta))-\frac{1}{n}\sum_{i=1}^nf(X_i,g(X_i,\eta_{i0}))\right\}\right]\\
        \leq & \mathbb{E}_{\{(X_i,Y_i)\}_{i=1}^n}\left[\sup_{h\in \mathcal{F}} \left\{\mathbb{E}h(X,\eta)-\frac{1}{n}\sum_{i=1}^nh(X_i,\eta_{i0})\right\}\right]\\
        \leq & C B_2 \sqrt{\frac{(H_{\mathcal{G}}+H_{\mathcal{D}}+1)S_{\mathcal{G},\mathcal{D}}\log(S_{\mathcal{G},\mathcal{D}})[\log n + \log B_2]}{n}},
    \end{align*}
where $C>0$ is a universal constant.
}
\end{proof}

\section{Additional simulation results}\label{appdx:S4}

In this section, we provide additional simulation results. 
{\color{black} Other than Models 1 and 2 considered in the main text, three additional models are used to show the performance of the proposed method.}

\noindent{\bf Model 1.}
{\it A nonlinear model with an additive error term:}
	\begin{align*}
	Y=X_{1}^{2}+\exp \left(X_{2}+X_{3} / 3\right)+\sin \left(X_{4}+X_{5}\right)+\varepsilon, \quad \varepsilon \sim N(0,1).
	\end{align*}
	
\noindent{\bf Model 2.}
{\it  A model with an additive error term whose variance depends on the predictors:}
\begin{align*}
Y=X_{1}^{2}+\exp\left(X_{2}+X_{3} / 3\right)+X_{4}-X_{5}+\left(0.5+X_{2}^{2} / 2+X_{5}^{2} / 2\right) \times \varepsilon,\ \varepsilon \sim N(0,1)
\end{align*}

\noindent{\bf Model 6.}
{\it A nonlinear model with a heavy tail additive error term:}
	\begin{align*}
	Y=X_{1}^{2}+\exp \left(X_{2}+X_{3} / 3\right)+\sin \left(X_{4}+X_{5}\right)+\varepsilon, \quad \varepsilon \sim t(3).
	\end{align*}

\noindent{\bf Model 7.}
{\it A model with a multiplicative non-Gassisan error term:
\begin{align*}
Y=(5+X_{1}^2/3 + X_{2}^2 + X_{3}^2 + X_{4} + X_{5})\times \exp(0.5 \epsilon),
\end{align*}
where $\epsilon\sim 0.5N(-2,1)+0.5N(2,1).$}

\noindent{\bf Model 8.}
{\it A mixture of two normal distributions:
\begin{align*}
&Y=I_{\{U<0.5\}}N(-X_1,0.25^2)+I_{\{U>0.5\}}N(X_1,0.25^2),
\end{align*}
where $U\sim$Uniform$(0,1)$ and is independent of $X$.}

\noindent We use the same evaluation criteria  in Section \ref{Sim}.

\subsection{Simulation results for different methods for Models 1-2, 6-8}
We compare the performance of the proposed method to the nonparametric least squares regression (NLS) and conditional WGAN (cWGAN).
The numerical results are summarized in {\color{black}Tables \ref{Tab-App-1} and \ref{Tab-App-2}}. It can be seen that the proposed method has smaller MSEs for most cases, indicating that the proposed method can improve the distribution matching to some extent.

\begin{table}
\centering
\caption{\label{Tab-App-1}The $L_1$ and $L_2$ errors, MSE of the estimated  conditional mean and the estimated standard deviation by different methods.}
\begin{threeparttable}
\begin{tabular}{ccc|cccc}
\hline
$X$&Model & Method & $L_1$ & $L_2$ & Mean & SD  \\
\hline
\multirow{15}{*}{$N(0,I_d)$}&\multirow{3}{*}{M1} & NLS & 0.83(0.02) & 1.08(0.05) & - & -\\
&& cWGAN & 1.00(0.02) & 1.97(0.25) & 0.98(0.25) & 0.09(0.03) \\
&& WGR  & {\bf 0.82}(0.02) & {\bf 1.07}(0.05) & {\bf 0.06}(0.01) & {\bf 0.04}(0.01)\\
\cline{2-7}
&\multirow{3}{*}{M2} & NLS & 1.24(0.04) & {\bf 2.45}(0.38) & - & - \\
&& cWGAN & 1.62(0.05) & 4.11(0.44) & 0.85(0.24) & 0.37(0.10) \\
&& WGR  & {\bf 1.23}(0.04) & 3.41(0.38) & {\bf 0.19}(0.10) & {\bf 0.22}(0.04) \\
\cline{2-7}
&\multirow{3}{*}{M6} & NLS & 1.17(0.04) & 3.05(0.70) &- & - \\
&& cWGAN & 1.31(0.03) & 3.95(0.53) & 2.92(0.46) & 0.39(0.15) \\
&& WGR  & {\bf 1.16}(0.05) & {\bf 3.00}(0.71)\ & {\bf 2.49}(0.14) & {\bf 0.28}(0.14) \\
\cline{2-7}
&\multirow{3}{*}{M7} & NLS & {\bf 2.23}(0.09) & {\bf 9.43}(1.09) & - & - \\
&& cWGAN & 2.31(0.09) & 10.13(1.21) & 0.91(0.12) & 1.00(0.08) \\
&& WGR  & 2.24(0.09) & 9.47(1.08) & {\bf 0.20}(0.03) & {\bf 0.85}(0.09) \\
\cline{2-7}
&\multirow{3}{*}{M8} & NLS & {\bf 0.83}(0.02) & 1.09(0.07) &- & - \\
&& cWGAN & {\bf 0.83}(0.03) & 1.09(0.07) & {\bf 0.01}(0.00) & {\bf 0.53}(0.14) \\
&& WGR  & {\bf 0.83}(0.02) & {\bf 1.09}(0.07) & {\bf 0.01}(0.00) & 0.62(0.09) \\
\hline
\multirow{15}{*}{$N(0,I_{100})$}&\multirow{3}{*}{M1} & NLS & 1.17(0.03) & 2.52(0.15) & - & - \leavevmode\\
&& cWGAN & 1.17(0.04) & 2.65(0.41) & 1.67(0.39) & 0.81(0.02)\\
&& WGR  &  {\bf 1.15}(0.04) & {\bf 2.40}(0.24)&{\bf 1.64}(0.25) & {\bf 0.16}(0.04) \\
\cline{2-7}
&\multirow{3}{*}{M2} & NLS &  1.63(0.05) & 5.37(0.47)&  - & -\\
&& cWGAN & 1.64(0.06) & 5.34(0.46) & 2.20(0.22) & 1.11(0.08)\\
&& WGR  &  {\bf 1.61}(0.06) &{\bf 5.27}(0.45) & {\bf 2.06}(0.24) & {\bf 0.23}(0.03)\\
\cline{2-7}
&\multirow{3}{*}{M6} & NLS & {\bf 1.58}(0.03) & {\bf 5.09}(0.62)&  - & -\\
&& cWGAN & 1.71(0.08) & 6.01(0.80) & 5.16(0.74) & 1.19(0.45)\\
&& WGR  &  1.62(0.06) & 5.39(0.66) & {\bf 4.20}(0.53) & {\bf 0.47}(0.19)\\
\cline{2-7}
&\multirow{3}{*}{M7} & NLS &  2.51(0.09) & 11.52(1.17)& - & -\\
&& cWGAN & 2.51(0.09) & 12.02(1.49) & 2.88(0.29) & 0.96(0.25) \\
&& WGR  &  {\bf 2.45}(0.11) & {\bf 11.37}(1.30) & {\bf 2.22}(0.17) & {\bf 0.73}(0.06) \\
\cline{2-7}
&\multirow{3}{*}{M8} & NLS & 0.85(0.02) & 1.13(0.07) &  - & -\\
&& cWGAN &0.89(0.03) & 1.29(0.10) & {\bf 0.00}(0.00) & {\bf 0.06}(0.00) \\
&& WGR  & {\bf 0.83}(0.02) & {\bf 1.09}(0.07) & 0.19(0.05) & {\bf 0.06}(0.30)\\
\hline
\end{tabular}
\footnotesize
NOTE: The corresponding simulation standard errors are given in parentheses. WGR is the propoed method.
\end{threeparttable}
\end{table}

\begin{table}
\centering
\caption{\label{Tab-App-2}MSE of the estimated conditional quantiles by different methods.}
\begin{threeparttable}
\resizebox{\textwidth}{!}{\begin{minipage}{\textwidth}
\begin{tabular}{ccc| ccccc}
\hline
$X$ &Model & Method & $\tau=0.05$ & $\tau=0.25$ & $\tau=0.50$ & $\tau=0.75$ & $\tau=0.95$  \\
\hline
\multirow{10}{*}{$ N(0,I_d)$}&\multirow{2}{*}{M1} & cWGAN &  1.22(0.23) & 1.04(0.25) & 0.99(0.26) & 1.03(0.24) & 1.34(0.21)  \\
&& WGR& {\bf 0.29}(0.07) & {\bf 0.10}(0.01) & {\bf 0.09}(0.02) & {\bf 0.10}(0.03) & {\bf 0.23}(0.06) \\
\cline{2-8}
&\multirow{2}{*}{M2} & cWGAN &  1.86(0.21) & 0.94(0.26) & 0.85(0.26) & 1.00(0.21) & 2.59(0.52)  \\
&& WGR& {\bf 0.77}(0.09) & {\bf 0.31}(0.06) & {\bf 0.19}(0.04) & {\bf 0.27}(0.05) & {\bf 0.81}(0.15) \\
\cline{2-8}
&\multirow{2}{*}{M6} & cWGAN & 3.57(0.73)& 3.29(0.58) &2.95(0.50) &3.01(0.46)& 6.14(1.44) \\
&& WGR& {\bf 3.39}(0.48) & {\bf 2.70}(0.24) & {\bf 2.47}(0.12) & {\bf 2.50}(0.21) & {\bf 2.76}(0.50) \\
\cline{2-8}
&\multirow{2}{*}{M7}& cWGAN & {\bf 0.17}(0.03) & {\bf 0.28}(0.05) & 0.60(0.09) & 1.99(0.21) &  8.40(0.79)  \\
&& WGR& 2.01(0.47) & 0.41(0.10) &{\bf 0.39}(0.06) & {\bf0.74}(0.11) & {\bf 4.04}(0.40)  \\
\cline{2-8}
&\multirow{2}{*}{M8}& cWGAN & {\bf 0.03}(0.02) & 0.04(0.03) & {\bf 0.36}(0.04) & 0.04(0.03) & 0.03(0.02)  \\
&&WGR& {\bf 0.03}(0.01) &{\bf 0.02}(0.01) & 0.40(0.04) & {\bf0.02}(0.01) & {\bf0.03}(0.02)  \\
\hline
\multirow{10}{*}{$ N(0,I_{100})$} & \multirow{2}{*}{M1}&  cWGAN & 3.18(0.71) & 1.83(0.22) &1.75(0.16)  & 1.89(0.21) & 3.61(0.39)  \\
&& WGR& {\bf1.84}(0.23) & {\bf1.69}(0.20) & {\bf1.66}(0.16) & {\bf1.88}(0.13) & {\bf2.41}(0.14) \\
\cline{2-8}
& \multirow{2}{*}{M2} &   cWGAN & 4.99(0.51) & 2.63(0.24) & 2.21(0.22) & 2.79(0.36) & 5.41(0.65) \\
&& WGR& {\bf 3.42}(1.07)& {\bf 2.26}(0.28)& {\bf 2.19}(0.22)& {\bf 2.57}(0.27) & {\bf3.49}(0.47)   \\
\cline{2-8}
& \multirow{2}{*}{M6}&   cWGAN & 7.79(1.79) & 5.95(1.09) & 5.04(0.80) & 6.15(0.88) & 12.86(2.92)  \\
&& WGR& {\bf5.47}(0.76) & {\bf4.89}(0.67) & {\bf4.72}(0.79) &{\bf 4.91}(0.78) &{\bf 4.65}(0.71) \\
\cline{2-8}
& \multirow{2}{*}{M7}&  cWGAN & {\bf 1.43}(0.24) & 1.92(0.11) & 2.65(0.17) & 4.21(0.49) & 10.38(1.89)    \\
&& WGR & 1.88(0.186) & {\bf 1.82}(0.13) & {\bf 2.17}(0.17) & {\bf 3.11}(0.23) & {\bf 7.15}(0.59)\\
\cline{2-8}
& \multirow{2}{*}{M8}&  cWGAN & 1.69(0.10) & 1.06(0.08) & {\bf0.34}(0.04) & 0.99(0.08) & 1.59(0.10)   \\
&& WGR &  {\bf0.64}(0.13)  & {\bf0.64}(0.13) & 0.56(0.10) & {\bf0.66}(0.11) & {\bf 0.62}(0.10) \\
\hline
\end{tabular}
\footnotesize
NOTE: The corresponding simulation standard errors are given in parentheses. WGR is the propoed method.
\end{minipage}}
\end{threeparttable}
\end{table}

\subsection{Simulation results with varying  noise dimension}
We check the performance of the proposed method with varying dimension of the noise $\eta$.
Other than  $m=3$ considered in Section \ref{Sim}, we also try $m=10$ and $m=25$. The results are given in Tables \ref{Tab-App-m} and \ref{Tab-App-m-quantile}.
For larger $m$, there tends to be a smaller MSE of the estimated conditional quantiles, but it takes  longer time to train.
There is trade-off between performance and training time.
We also observe that the performance of the proposed method is
somewhat robust to the noise dimension $m$, as long as
 $m$ takes reasonable values.

\begin{table}
\centering
\caption{\label{Tab-App-m}The $L_1$ and $L_2$ error, MSE of the estimated conditional mean and the estimated standard deviation for different $m$.}
\begin{threeparttable}
\begin{tabular}{ccc|cccc}
\hline
$X$&Model & $m$ & $L_1$ & $L_2$ & Mean & SD  \\
\hline
\multirow{15}{*}{$ N(0,I_d)$}&\multirow{3}{*}{M1}
& 3 &  0.83(0.02) & 1.07(0.05) &0.06(0.01) & 0.04(0.01) \\
&& 10 & 0.83(0.02)  & 1.08(0.05) & 0.06(0.01) & 0.04(0.2)   \\
&& 25 & 0.83(0.03) & 1.08(0.08) & 0.06(0.01) & 0.07(0.02)  \\
\cline{2-7}
&\multirow{3}{*}{M2}
& 3 &  1.24(0.04) & 3.42(0.38) & 0.18(0.13) & 0.25(0.08)\\
&& 10 & 1.23(0.05) & 3.38(0.39) & 0.14(0.03) & 0.14(0.03)    \\
&& 25 & 1.21(0.04) & 3.21(0.28) & 0.16(0.04) & 0.20(0.05) \\
\cline{2-7}
&\multirow{3}{*}{M6}
& 3 &   1.16(0.05) & 3.00(0.72) & 2.49(0.14) & 0.28(0.14) \\
&& 10 &  1.16(0.05) & 3.02(0.72) & 2.36(0.16) & 0.20(0.06) \\
&& 25 &  1.18(0.07) & 3.10(0.72) & 2.37(0.17) & 0.18(0.06)  \\
\cline{2-7}
&\multirow{3}{*}{M7}
& 3 &    2.24(0.09) & 9.47(1.08) & 0.20(0.03) & 0.85(0.09) \\
&& 10 &  2.24(0.09) & 9.43(1.09) & 0.20(0.02) & 0.80(0.15) \\
&& 25 &  2.21(0.10) & 9.21(1.22) & 0.22(0.04) & 0.77(0.07)\\
\cline{2-7}
&\multirow{3}{*}{M8}
& 3 &   0.83(0.02) & 1.09(0.07) & 0.01(0.01) & 0.62(0.09) \\
&& 10 &  0.83(0.02) &1.09(0.07) & 0.00(0.00) & 0.63(0.10) \\
&& 25 &  0.83(0.02) & 1.08(0.07) & 0.00(0.00) & 0.56(0.13) \\
\hline
\multirow{15}{*}{$ N(0,I_{100})$}& \multirow{3}{*}{M1}& 3 &1.15(0.04) & 2.40(0.24) & 1.64(0.25) & 0.16(0.04) \\
&& 10   & 1.25(0.03) & 2.83(0.21) & 1.86(0.22) & 0.17(0.05)\\
&& 25 & 1.18(0.04) & 2.47(0.17) & 1.44(0.13) & 0.16(0.05) \\
\cline{2-7}
& \multirow{3}{*}{M2}& 3  & 1.62(0.05) & 5.36(0.46) & 2.11(0.22) & 0.36(0.12)    \\
&& 10&1.63(0.08) & 5.40(0.57) & 2.20(0.23) & 0.38(0.12) \\
&& 25 & 1.62(0.06) & 5.31(0.43) & 2.17(0.19) & 0.45(0.19)\\
\cline{2-7}
& \multirow{3}{*}{M6}& 3  & 1.62(0.06) & 5.39(0.66) & 4.20(0.53) & 0.47(0.19)  \\
&& 10  & 1.63(0.05) & 5.34(0.56) & 4.47(0.62) & 0.65(0.18) \\
&& 25 & 1.58(0.03) & 4.92(0.33) & 4.20(0.39) & 0.50(0.13) \\
\cline{2-7}
& \multirow{3}{*}{M7}& 3 & 2.45(0.11) & 11.37(1.30) & 2.22(0.17) & 0.73(0.06)  \\
&& 10  & 2.44(0.10) & 11.14(1.08) & 2.15(0.24) & 0.80(0.15) \\
&& 25 & 2.46(0.09) & 11.26(1.20) & 2.10(0.25) & 0.73(0.06) \\
\cline{2-7}
& \multirow{3}{*}{M8}& 3& 0.83(0.02) & 1.09(0.07) & 0.19(0.05) & 0.30(0.16)    \\
&& 10 & 0.90(0.02) & 1.34(0.06) & 0.24(0.08) & 0.34(0.12) \\
&& 25 & 0.89(0.04) & 1.27(0.11) & 0.15(0.08) & 0.40(0.13)\\
\hline
\end{tabular}
\footnotesize
NOTE: The corresponding  standard errors are given in parentheses.
\end{threeparttable}
\end{table}

\begin{table}
\centering
\caption{\label{Tab-App-m-quantile}Mean squared prediction error (MSE) of the estimated conditional quantile for different values of  $m$.}
\begin{threeparttable}
\begin{tabular}{ccc| ccccc }
\hline
$X$ & Model & $m$ & $\tau=0.05$ & $\tau=0.25$ & $\tau=0.50$ & $\tau=0.75$ & $\tau=0.95$\\
\hline
\multirow{15}{*}{$N(0,I_d)$}&\multirow{3}{*}{M1} & 3 & 0.29(0.07) & 0.10(0.01) & 0.09(0.02) & 0.10(0.03) & 0.23(0.06)  \\
&& 10 & 0.18(0.04) & 0.09(0.02) & 0.08(0.01) & 0.09(0.01) & 0.16(0.05) \\
& &25 & 0.26(0.05) & 0.10(0.02) & 0.07(0.02) & 0.11(0.10) & 0.29(0.10)  \\
\cline{2-8}
&\multirow{3}{*}{M2}  & 3 & 0.77(0.09) & 0.31(0.06) & 0.19(0.04) & 0.27(0.05) & 0.81(0.15)  \\
&& 10 & 0.49(0.11) & 0.21(0.05) & 0.15(0.03) & 0.20(0.04) & 0.54(0.12)  \\
&& 25 & 0.62(0.18) & 0.24(0.07) & 0.16(0.04) & 0.23(0.04) & 0.76(0.16) \\
\cline{2-8}
&\multirow{3}{*}{M6} & 3 &  3.39(0.48) & 2.70(0.24) & 2.47(0.12) & 2.50(0.21) & 2.76(0.50) \\
&& 10 & 2.74(0.26) & 2.39(0.19) & 2.39(0.20) & 2.59(0.21) & 3.07(0.45) \\
&& 25 &  2.97(0.30) & 2.46(0.20) & 2.37(0.20) & 2.55(0.21) & 3.02(0.35) \\
\cline{2-8}
&\multirow{3}{*}{M7} & 3 &  2.01(0.47) & 0.40(0.10) & 0.39(0.06) & 0.74(0.11) & 4.04(0.40)  \\
&& 10 & 1.51(0.66) & 0.56(0.05) & 0.31(0.05) & 0.63(0.19) & 4.73(0.68)  \\
&& 25 & 1.07(0.14) & 0.50(0.11) & 0.32(0.09) & 0.57(0.10) & 3.88(0.51)  \\
\cline{2-8}
&\multirow{3}{*}{M8} & 3 &  0.03(0.01) & 0.02(0.01) & 0.40(0.04) & 0.02(0.01) & 0.03(0.02) \\
&& 10 &  0.02(0.01) & 0.02(0.01) & 0.41(0.05) & 0.02(0.01) & 0.02(0.01) \\
&& 25 & 0.03(0.02) & 0.02(0.02) & 0.41(0.04) & 0.02(0.02) & 0.03(0.01) \\
\hline
\multirow{15}{*}{$N(0,I_{100})$} &\multirow{3}{*}{M1} & 3  &  1.84(0.23) & 1.69(0.20) & 1.75(0.16) & 1.88(0.13) & 2.41(0.14) \\
&& 10 & 2.05(0.30) & 1.91(0.25) & 1.92(0.22) & 2.02(0.23) & 2.52(0.36)\\
&& 25 & 1.69(0.17) & 1.42(0.15) & 1.47(0.13) & 1.65(0.12) & 2.02(0.22)\\
\cline{2-8}
 &\multirow{3}{*}{M2} & 3 & 3.42(1.07) & 2.26(0.28) & 2.19(0.22) & 2.57(0.27) & 3.49(0.47)\\
&& 10& 3.07(0.67) & 2.29(0.21) & 2.33(0.26) & 2.64(0.37) & 3.55(0.44) \\
&& 25 & 3.58(1.15) & 2.25(0.17) & 2.34(0.17) & 2.79(0.12) & 3.54(0.26) \\
\cline{2-8}
 &\multirow{3}{*}{M6} & 3 & 4.47(0.76) & 4.89(0.67) & 4.72(0.79) & 4.91(0.78) & 5.65(0.71) \\
&& 10  & 5.91(0.55) & 4.68(0.55) & 4.55(0.61) & 4.94(0.68) & 6.17(0.74) \\
&& 25 & 4.28(0.49) & 4.08(0.36) & 4.36(0.38) & 4.71(0.57) & 5.35(0.76)\\
\cline{2-8}
 &\multirow{3}{*}{M7} & 3 & 1.88(0.19) & 1.82(0.13) & 2.17(0.17) & 3.11(0.23) & 7.15(0.59) \\
&& 10& 1.91(0.29) & 1.87(0.17) & 2.16(0.21) & 3.06(0.30) & 7.22(0.91) \\
&& 25 & 1.89(0.23) & 1.79(0.18) & 2.10(0.25) & 3.00(0.37) & 6.77(0.69) \\
\cline{2-8}
 &\multirow{3}{*}{M8} & 3& 0.64(0.13) & 0.64(0.13) & 0.56(0.10) & 0.66(0.11) & 0.62(0.10) \\
&& 10& 0.67(0.12) & 0.62(0.10) & 0.63(0.08) & 0.64(0.12) & 0.59(0.08) \\
&& 25& 0.69(0.17) & 0.60(0.09) & 0.52(0.10) & 0.65(0.06) & 0.67(0.09)  \\
\hline
\end{tabular}
\footnotesize
NOTE: The corresponding simulation standard errors are given in parentheses.
\end{threeparttable}
\end{table}

\subsection{Simulation results for  different  values of $J$}
We conduct simulation studies to check the performance of the proposed method for different $J$, the sample size of the  generated noise vector $\eta$.
The generator and discriminator networks have two fully-connected hidden layers. The LeakyReLU function is used as the active function.
The noise vector $\eta$ is generated from $N(\bm{0},\bm{I}_{3})$.
We take $J=10,50,200,500$, and report the results in Tables \ref{Tab-App-J} and \ref{Tab-App-J-quantile}.
It can be seen that the proposed method works comparably  for  different  $J$.

\begin{table}
\centering
\caption{\label{Tab-App-J}The $L_1$ and $L_2$ error, MSE of the estimated conditional mean and the estimated standard deviation for different $J$.}
\begin{threeparttable}	
\begin{tabular}{ccc| cccc }
\hline
$X$ & Model & $J$ & $L_1$ & $L_2$ & Mean & SD \\
\hline
\multirow{20}{*}{$N(0,I_d)$} &\multirow{4}{*}{M1} 
& 10 & 0.83(0.02) & 1.09(0.04) & 0.07(0.01) & 0.10(0.03)   \\
&& 50 &   0.82(0.02) & 1.07(0.04) & 0.06(0.02) & 0.05(0.02)\\
&& 200 & 0.83(0.02) & 1.07(0.05) &0.06(0.01) & 0.04(0.01) \\
&& 500 & 0.82(0.02) &1.06(0.06) & 0.05(0.01) &0.05(0.02)  \\
\cline{2-7}
&\multirow{4}{*}{M2}
& 10 &  1.24(0.04) & 3.42(0.38) & 0.18(0.13) & 0.25(0.08) \\
&& 50 &   1.23(0.05) & 3.42(0.40) & 0.19(0.14) & 0.26(0.09) \\
&& 200 & 1.23(0.04) & 3.41(0.38) & 0.19(0.10) & 0.22(0.04)\\
&& 500 & 1.22(0.04) & 3.40(0.38) & 0.16(0.13) &0.22(0.03) \\
\cline{2-7}
&\multirow{4}{*}{M6}
 & 10 &  1.16(0.03) & 2.86(0.30) & 2.60(0.22) &0.35(0.13)\\
&& 50 &  1.16(0.04) & 3.02(0.66) & 2.51(0.15) & 0.24(0.14)\\
&& 200 & 1.16(0.05) & 3.00(0.72) & 2.49(0.14) & 0.28(0.14)\\
&& 500 & 1.16(0.05) & 3.00(0.71) & 2.49(0.14) & 0.28(0.14)\\
\cline{2-7}
&\multirow{4}{*}{M7}
& 10 &  2.23(0.09) & 9.44(1.08) & 0.21(0.03) & 0.95(0.12)\\
&& 50 &  2.24(0.09) & 9.42(1.10) & 0.21(0.04) & 0.77(0.15)  \\
&& 200 & 2.24(0.09) & 9.47(1.08) & 0.20(0.03) & 0.85(0.09) \\
&& 500 & 2.23(0.09) & 9.40(1.07) & 0.18(0.02) & 0.86(0.09) \\
\cline{2-7}
&\multirow{4}{*}{M8}
& 10 &  0.83(0.03) & 1.09(0.07) & 0.00(0.00) & 0.59(0.15)\\
&& 50 &  0.83(0.02) & 1.08(0.07) & 0.00(0.00) & 0.62(0.14)  \\
&& 200 & 0.83(0.02) & 1.09(0.07) & 0.01(0.01) & 0.62(0.09)\\
&& 500 & 0.83(0.02) & 1.09(0.07) & 0.00(0.00) & 0.60(0.11) \\
\hline
\multirow{20}{*}{$N(0,I_{100})$} &\multirow{4}{*}{M1} & 10& 1.17(0.05) & 2.40(0.21) & 1.71(0.20) & 0.20(0.13) \\
&& 50 & 1.15(0.03) & 2.50(022) &1.65(0.22) &0.14(0.04)  \\
&& 200& 1.15(0.04) & 2.40(0.24) & 1.64(0.25) & 0.16(0.04) \\
&&500 & 1.20(0.03) &2.45(0.20) & 1.74(0.18) & 0.13(0.03)\\
\cline{2-7}
&\multirow{4}{*}{M2}& 10 & 1.62(0.05) & 5.36(0.46) & 2.11(0.22) & 0.36(0.12) \\
&& 50 & 1.59(0.04) & 5.23(0.54) & 1.98(0.27) & 0.37(0.10) \\
&& 200  & 1.61(0.06) & 5.27(0.45) & 2.06(0.24) & 0.30(0.03)\\
&&500 & 1.62(0.05)& 5.39(0.48) & 2.08(0.24) &0.30(0.05) \\
\cline{2-7}
&\multirow{4}{*}{M6} & 10  & 1.58(0.05) &5.14(0.54) &4.32(0.51) &0.43(0.17)\\
&& 50  &1.57(0.05) & 5.06(0.52) &4.18(0.59) & 0.53(0.17)\\
&& 200 & 1.62(0.06) & 5.39(0.66) & 4.20(0.53) & 0.47(0.19)\\
&&500  & 1.57(0.05) & 5.05(0.56) & 4.14(0.50) &0.44(0.13) \\
\cline{2-7}
&\multirow{4}{*}{M7} & 10 & 2.44(0.11) & 11.13(1.23) & 2.03(0.15) & 0.73(0.09) \\
&& 50& 2.45(0.10) & 11.23(1.28) & 2.12(0.16) & 0.68(0.05)\\
&& 200& 2.45(0.11) & 11.37(1.30) & 2.22(0.17) & 0.73(0.06) \\
&&500 & 2.46(0.12) & 11.75(1.61) &2.24(0.21) & 0.78(0.10) \\
\cline{2-7}
&\multirow{4}{*}{M8}& 10 & 0.83(0.02) & 1.15(0.06) & 0.15(0.08) & 0.25(0.15)  \\
&& 50 & 0.83(0.02) & 1.09(0.06) & 0.19(0.07) & 0.28(0.12)  \\
&& 200  & 0.83(0.02) & 1.09(0.07) & 0.19(0.05) & 0.30(0.16)\\
&&500 & 0.83(0.02) & 1.10(0.07) & 0.16(0.06) & 0.32(0.16) \\
\hline
\end{tabular}
\footnotesize
NOTE: The corresponding simulation standard errors are given in parentheses.
\end{threeparttable}
\end{table}

\begin{table}
\centering
\caption{\label{Tab-App-J-quantile}Mean squared prediction error (MSE) of the estimated conditional quantile for different $J$.}
\begin{threeparttable}	
\begin{tabular}{ccc| ccccc }
\hline
$X$ &Model & $J$ & $\tau=0.05$ & $\tau=0.25$ & $\tau=0.50$ & $\tau=0.75$ & $\tau=0.95$  \\
\hline
\multirow{20}{*}{$N(0,I_d)$}&\multirow{4}{*}{M1} & 10 & 0.44(0.12) & 0.12(0.03) & 0.09(0.02) & 0.15(0.05) & 0.30(0.09)  \\
&& 50 & 0.33(0.11) & 0.09(0.02) & 0.07(0.02) & 0.11(0.03) & 0.19(0.05)   \\
&& 200 & 0.29(0.07) & 0.10(0.01) & 0.09(0.02) & 0.10(0.03) & 0.23(0.06)   \\
&& 500 & 0.25(0.07) & 0.09(0.02) & 0.06(0.01) & 0.09(0.02) & 0.20(0.05)   \\
\cline{2-8}
&\multirow{4}{*}{M2}  & 10 &  1.19(0.36) & 0.30(0.08) & 0.25(0.03) & 0.33(0.06) & 0.72(0.19)  \\
&& 50 &  1.01(0.29) & 0.31(0.09) & 0.24(0.04) & 0.39(0.09) & 0.87(0.29)    \\
&& 200 & 0.77(0.09) & 0.31(0.06) & 0.19(0.04) & 0.27(0.05) & 0.81(0.15)\\
&& 500 & 0.73(0.14) & 0.30(0.05) & 0.20(0.04) & 0.29(0.05) & 0.78(0.14)   \\
\cline{2-8}
&\multirow{4}{*}{M6}  & 10 &  3.30(0.64) & 2.75(0.39) & 2.61(0.25) & 2.61(0.22) & 2.86(0.26) \\
&& 50 & 3.26(0.38) & 2.70(0.29) & 2.51(0.19) & 2.51(0.19) & 2.75(0.44)   \\
&& 200 &  3.39(0.48) & 2.70(0.24) & 2.47(0.12) & 2.50(0.21) & 2.76(0.50)\\
&& 500 & 3.77(0.67) & 2.89(0.36) & 2.51(0.15) & 2.54(0.20) & 3.16(0.79)\\
\cline{2-8}
&\multirow{4}{*}{M7} & 10 &  3.10(0.56) & 0.63(0.12) & 0.42(0.08) & 0.97(0.11) & 3.83(0.44)  \\
&& 50 &  1.77(0.49) & 0.30(0.06) & 0.37(0.06) & 0.80(0.16) & 4.34(0.96) \\
&& 200 &  2.01(0.47) & 0.40(0.10) & 0.39(0.06) & 0.74(0.11) & 4.04(0.40)  \\
&& 500 &  2.80(0.63) & 0.42(0.14) & 0.42(0.06) & 0.96(0.13) & 3.99(0.68)\\
\cline{2-8}
&\multirow{4}{*}{M8} & 10 & 0.04(0.01) & 0.03(0.01) & 0.56(0.09) & 0.05(0.01) & 0.01(0.01)  \\
&& 50 &   0.02(0.01) & 0.02(0.01) & 0.39(0.05) & 0.02(0.01) & 0.03(0.02) \\
&& 200 &  0.03(0.01) & 0.02(0.01) & 0.40(0.04) & 0.02(0.01) & 0.03(0.02)\\
&& 500 &  0.02(0.01) & 0.02(0.01) & 0.42(0.05) & 0.04(0.01) & 0.02(0.02) \\
\hline
\multirow{20}{*}{$N(0,I_{100})$}&\multirow{4}{*}{M1}& 10& 2.29(0.46) & 1.81(0.25) &1.74(0.20) & 1.91(0.21) & 2.62(0.50)   \\
&& 50 & 1.77(0.29) & 1.74(0.46) & 1.72(0.20) & 1.86(0.17) & 2.25(0.24)  \\
&& 200&  1.84(0.23) & 1.69(0.20) & 1.75(0.16) & 1.88(0.13) & 2.41(0.14) \\
&& 500 & 1.81(0.25) & 1.73(0.21) & 1.81(0.19) & 1.95(0.15) & 2.35(0.22) \\
\cline{2-8}
&\multirow{4}{*}{M2}& 10 &2.80(0.40) & 2.22(0.25) & 2.20(0.22) & 2.45(0.29) & 3.24(0.48) \\
&& 50& 2.55(0.33) & 2.17(0.23) & 2.11(0.22) & 2.36(0.31) & 3.12(0.43) \\
&& 200 & 3.42(1.07) & 2.26(0.28) & 2.19(0.22) & 2.57(0.27) & 3.49(0.47)  \\
&& 500  & 2.52(0.45) & 2.14(0.23) & 2.18(0.23) & 2.46(0.32) & 3.40(0.49) \\
\cline{2-8}
&\multirow{4}{*}{M6}& 10 & 4.61(0.57) & 4.42(0.57) & 4.49(0.53) & 4.72(0.56) & 5.18(0.62) \\
&& 50 & 4.52(0.81) & 4.27(0.64) & 4.41(0.61) & 4.80(0.63) & 5.66(0.97) \\
&& 200  & 4.47(0.76) & 4.89(0.67) & 4.72(0.79) & 4.91(0.78) & 5.65(0.71)\\
&& 500 & 4.44(0.55) & 4.24(0.51) & 4.36(0.51) & 4.65(0.59) & 5.14(0.63) \\
\cline{2-8}
&\multirow{4}{*}{M7}& 10 & 1.78(0.16) & 1.77(0.16) & 2.01(0.15) & 2.79(0.19) & 6.74(0.77) \\
&& 50  & 1.72(0.19) & 1.75(0.19) & 2.09(0.21) & 3.00(0.25) & 6.77(0.52)  \\
&& 200& 1.88(0.19) & 1.82(0.13) & 2.17(0.17) & 3.11(0.23) & 7.15(0.59)\\
&& 500 & 2.07(0.18) & 1.86(0.16) & 2.16(0.17) & 3.16(0.24) & 7.23(0.70)  \\
\cline{2-8}
&\multirow{4}{*}{M8}& 10& 1.60(0.13) & 1.04(0.10) & 0.36(0.12) & 1.07(0.11) & 1.63(0.14)   \\
&& 50&0.56(0.05) & 0.66(0.10) & 0.42(0.04) & 0.64(0.09) & 0.60(0.09) \\
&& 200  & 0.64(0.13) & 0.64(0.13) & 0.56(0.10) & 0.66(0.11) & 0.62(0.10)\\
&& 500& 0.60(0.08) & 0.72(0.09) & 0.40(0.04) & 0.65(0.08) & 0.61(0.08)  \\
\hline
\end{tabular}
\footnotesize
NOTE: The corresponding  standard errors are given in parentheses.
\end{threeparttable}
\end{table}

\clearpage
\section{Additional real data examples}\label{appdx:S5}

\subsection{Color image dataset: Cifar10 and STL10 dataset.}
In this part, we apply the proposed method to two color image datasets: Cifar10 and STL10 dataset available at \url{https://www.cs.toronto.edu/~kriz/cifar.html} and  \url{https://cs.stanford.edu/~acoates/stl10/}.
The images in  Cifar-10 dataset and STL10 dataset are stored in $32 \times 32 \times 3 $ matrices and $96 \times 96 \times 3 $ matrices, respectively, indicating that the information contained in images from Cifar-10 dataset is less and easier to reconstruct than those in STL10 dataset.
Since the images are of different sizes, we resize all images into the size $128\times 128 \times 3$, a commonly-used size for color image generation.

The Cifar-10 dataset contains 60000 colored images in 10 classes, with 6000 images per class.
STL10 dataset contains 13000 labelled images belonging to 10 classes on average, and 100000 unlabelled images, which may be images of other types of animals or vehicles.
For comparison, we consider several  settings for training.

\noindent{\bf (S1)}: For Cifar-10 dataset, we randomly select 10000 images for training, 1000 images for validation, and 10000 images for testing.

\noindent{\bf(S2)}: For  STL10 dataset, we only pick the unlabelled images from the dataset, and randomly select 80000 images for training, 1000 images for validation, and 2000 images for testing.

\noindent{\bf(S3)}:  For STL10 dataset, we only pick the labelled images from the dataset, and randomly partition them into three parts:  10000 training images,  1000 validation images, and 2000 test images.

\noindent{\bf(S4)}: We use the same images as in setting (S3), but  use the training images in each class to obtain different generators for different classes. In other words, 10 generators are obtained as there are 10 classes of labelled images in STL10 dataset.

We use the  Fr\'{e}chet inception distance (FID) score  \citep{NIPS2017_8a1d6947} to measure the performance of our method and compare it with other methods. Table \ref{Tab-app-stl10} shows the FID scores for different settings. NLS has the highest FID scores in all settings, which means it has the worst image reconstruction quality. WGR has lower FID scores than cWGAN in setting (S1), (S2), and (S3), but similar FID scores in setting (S4). Setting (S4) is the easiest one, because it only requires learning the conditional distribution of one class, while other settings require learning a mixture conditional distribution of all classes. Our method achieves the lowest FID scores in all settings.

\begin{table}
\centering
\caption{FID score of the three methods on the test dataset of color images.}
\label{Tab-app-stl10}
\begin{tabular}{c |cccc }
\hline
Method & (S1) & (S2) & (S3) & (S4)\\
\hline
NLS & 66.093 & 116.875 &113.659 & 110.068  \\
cWGAN & 32.669 & 75.211 & 72.731 & 66.669 \\
WGR & 29.513 & 69.653 & 67.565 & 67.110 \\
\hline
\end{tabular}
\end{table}

\subsubsection{Image reconstruction: STL10 dataset}\label{STL10}
In this part, we apply WGR to the reconstruction task for color images.
We use the STL10 dataset, which is available at \url{https://cs.stanford.edu/~acoates/stl10/}.

We preprocess the STL10 dataset by resizing all the images to $128\times 128 \times 3$, which is a standard size for color image generation tasks.
We then simulate a missing data scenario by masking the central quarter part of each image.
Our goal is to reconstruct the masked part of the image.
Therefore, the covariates $X$ have a dimension of $128\times128\times3-65\times65\times 3=36477$ and the response $Y$ has a dimension of $65\times 65\times 3=12675$.

Figure \ref{Fig-stl10} displays the reconstructed images from three different methods. The images produced by WGR are more faithful to the original, as they have higher clarity than NLS and higher accuracy than cWGAN. We use the FID score to measure the quality of reconstructed images numerically. The FID scores of NLS, cWGAN, and WGR are 113.66, 72.73, and 67.56 respectively. In this case, our method achieves the lowest FID score, which implies that it reconstructs images with better quality to a certain degree.

\begin{figure}
\centering
\begin{minipage}[t]{0.85\linewidth}
\hspace{3.0cm} $X$ \hspace{0.8cm} Truth \hspace{0.6cm} WGR\hspace{1.0cm} NLS \hspace{0.5cm} cWGAN
\end{minipage}
\\[0.1\baselineskip]
\begin{minipage}[t]{0.85\linewidth}
\centering
	\includegraphics[width=0.08\textheight]{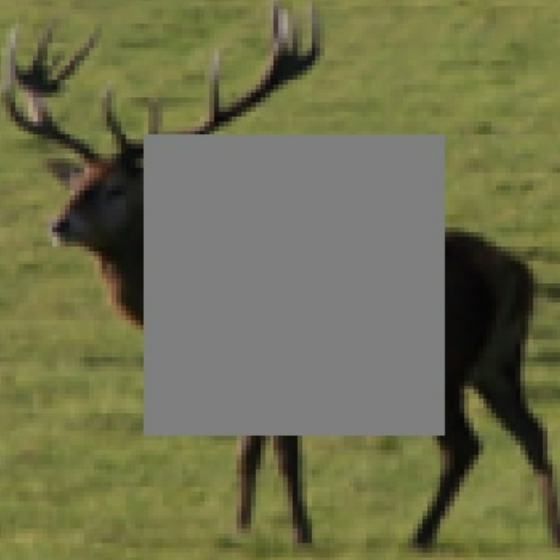}
	\includegraphics[width=0.08\textheight]{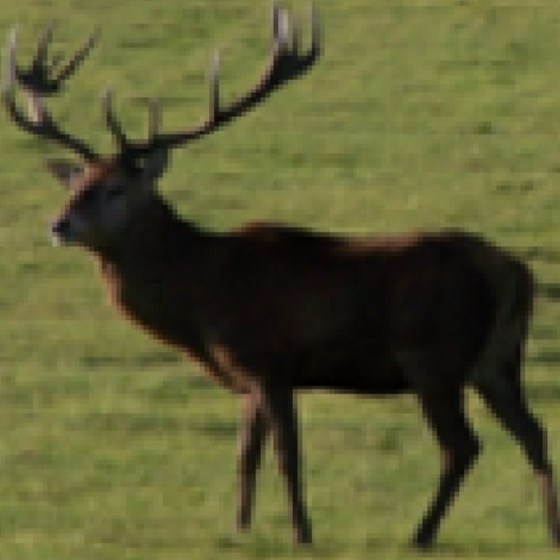}
	\includegraphics[width=0.08\textheight]{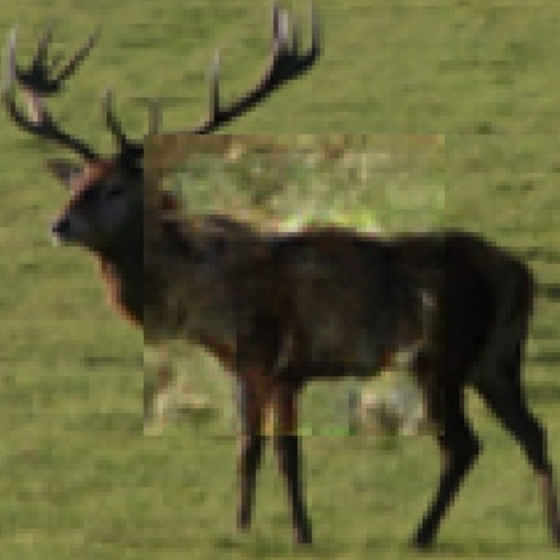}
	\includegraphics[width=0.08\textheight]{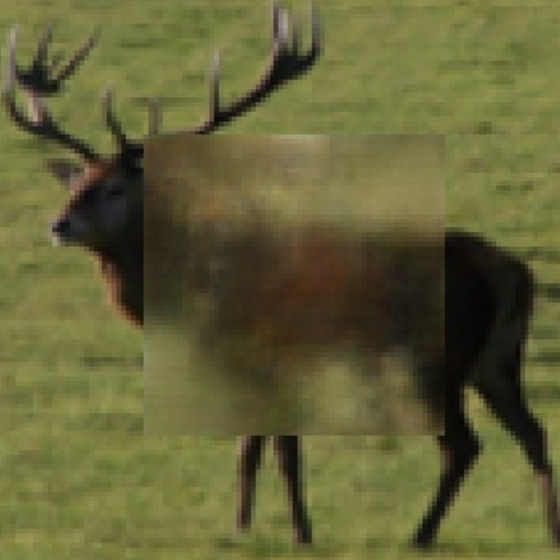}
	\includegraphics[width=0.08\textheight]{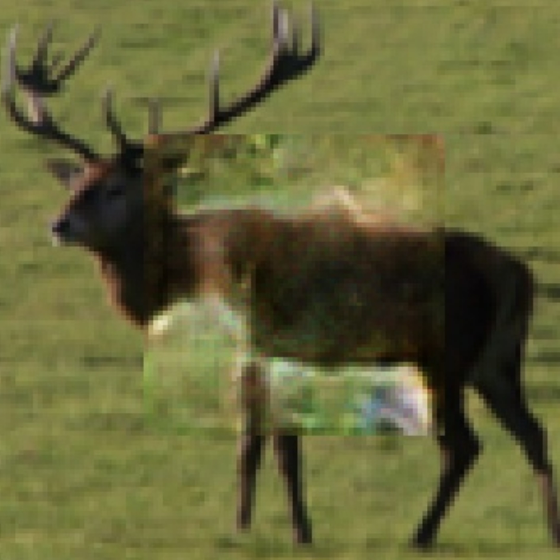}
\end{minipage}
\\[0.1\baselineskip]
\begin{minipage}[t]{0.85\linewidth}
\centering
	\includegraphics[width=0.08\textheight]{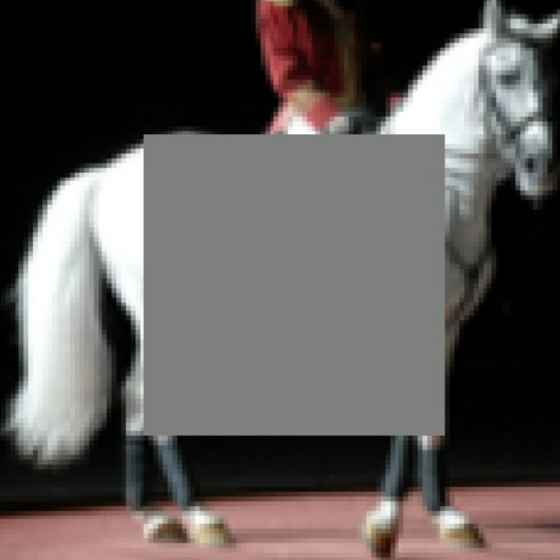}
	\includegraphics[width=0.08\textheight]{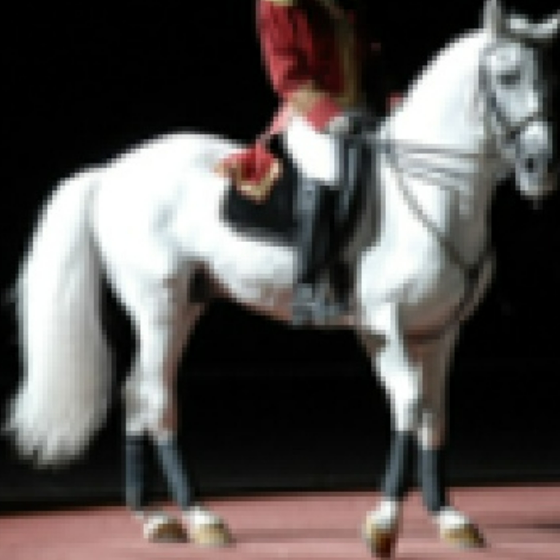}
	\includegraphics[width=0.08\textheight]{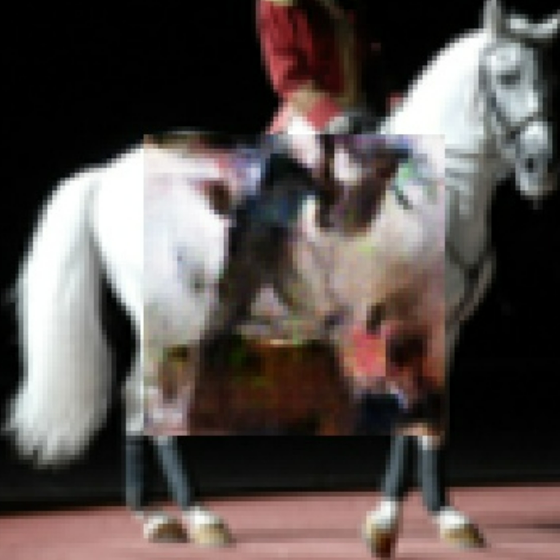}
	\includegraphics[width=0.08\textheight]{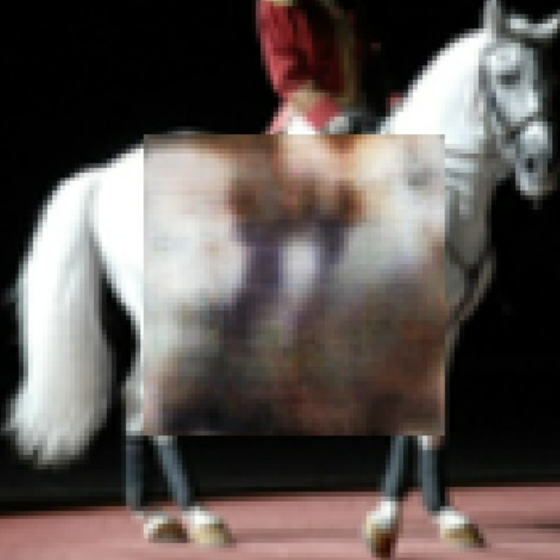}
	\includegraphics[width=0.08\textheight]{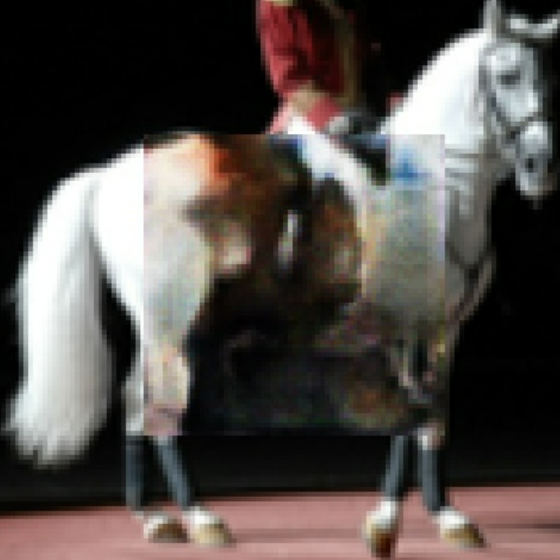}
\end{minipage}
\\[0.1\baselineskip]
\begin{minipage}[t]{0.85\linewidth}
\centering
	\includegraphics[width=0.08\textheight]{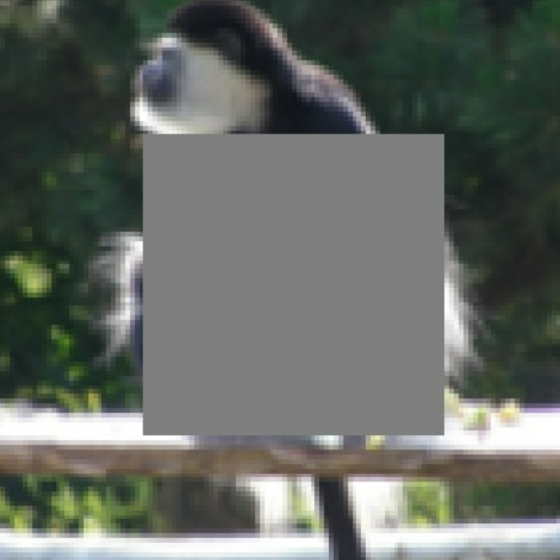}
	\includegraphics[width=0.08\textheight]{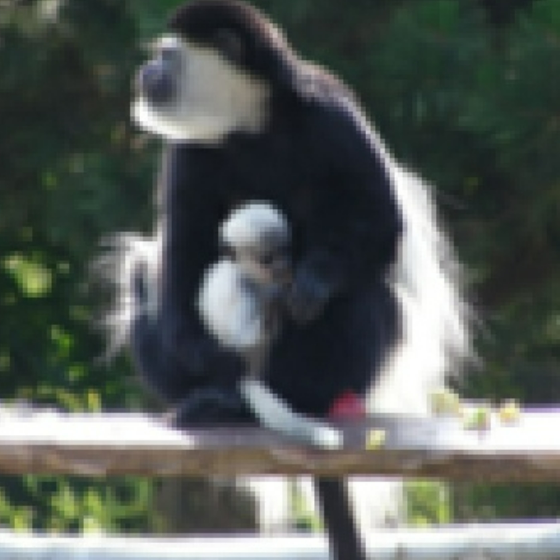}
	\includegraphics[width=0.08\textheight]{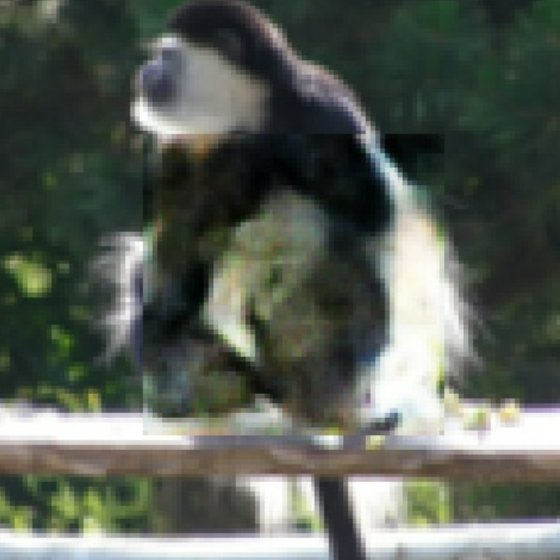}
	\includegraphics[width=0.08\textheight]{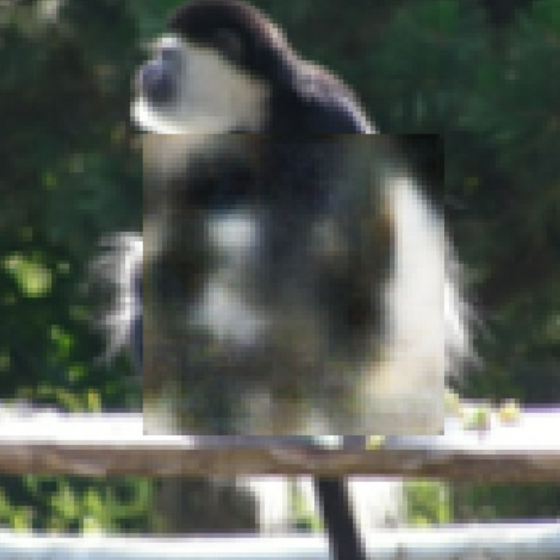}
	\includegraphics[width=0.08\textheight]{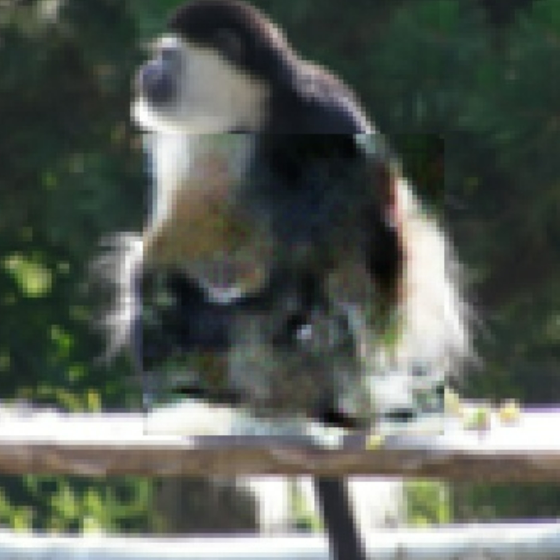}
\end{minipage}
\\[0.1\baselineskip]
\begin{minipage}[t]{0.85\linewidth}
\centering
	\includegraphics[width=0.08\textheight]{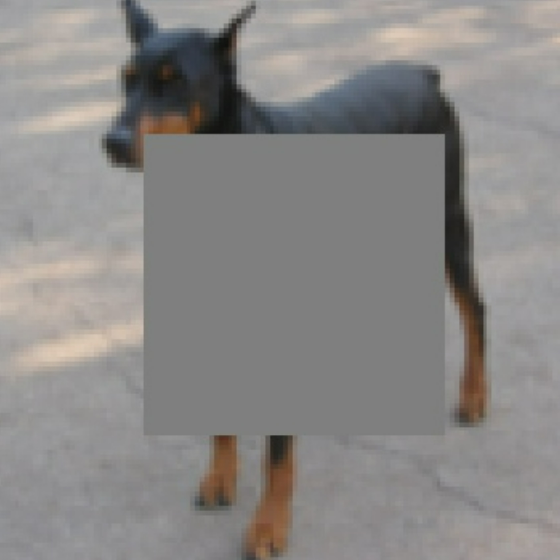}
	\includegraphics[width=0.08\textheight]{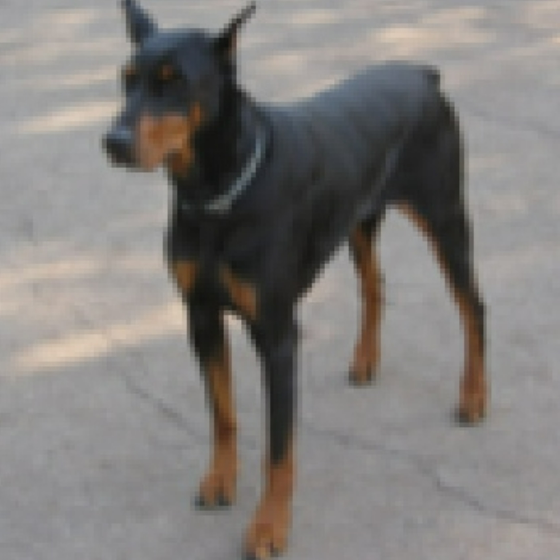}
	\includegraphics[width=0.08\textheight]{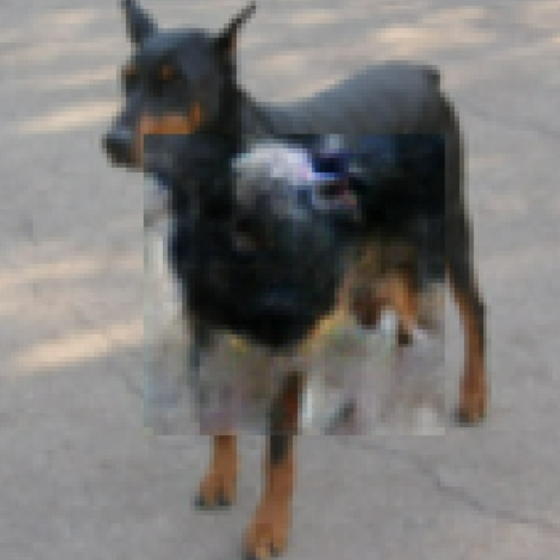}
	\includegraphics[width=0.08\textheight]{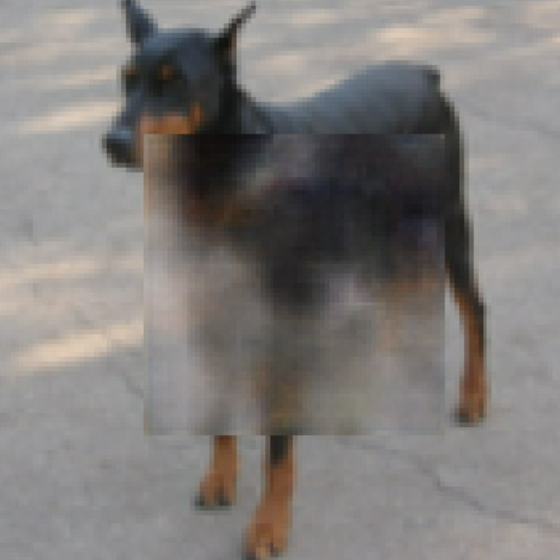}
	\includegraphics[width=0.08\textheight]{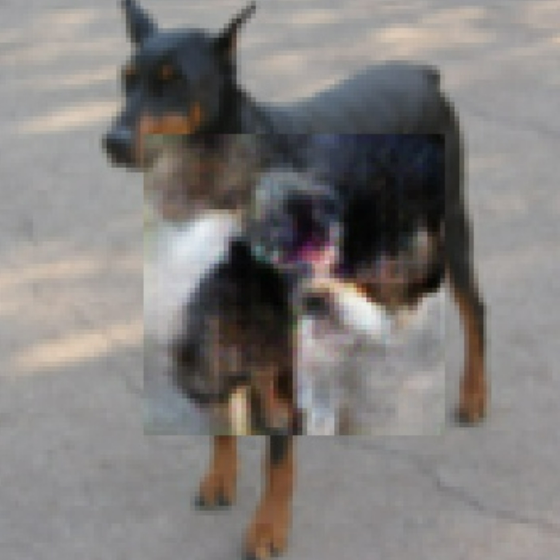}
\end{minipage}
\caption{\label{Fig-stl10}Reconstructed images in STL10 test dataset.}
\end{figure}


\end{appendices}

\end{document}